\pdfoutput=1

\documentclass[11pt]{article}
\ifdefined\pdfsuppressptexinfo\pdfsuppressptexinfo=-1\fi
\usepackage[preprint]{acl}
\usepackage{times}
\usepackage{latexsym}
\usepackage[T1]{fontenc}
\usepackage[utf8]{inputenc}
\usepackage{microtype}
\IfFileExists{inconsolata.sty}{\usepackage{inconsolata}}{}
\usepackage{graphicx}
\usepackage{amsmath,amssymb,amsthm}
\usepackage{booktabs}
\usepackage{adjustbox}  %
\usepackage{multirow}
\usepackage{xcolor}
\usepackage{colortbl}
\renewcommand{\topfraction}{0.9}\renewcommand{\bottomfraction}{0.8}\renewcommand{\textfraction}{0.07}
\renewcommand{\floatpagefraction}{0.75}\renewcommand{\dbltopfraction}{0.9}\renewcommand{\dblfloatpagefraction}{0.75}
\definecolor{oursbg}{HTML}{FCE9DC}   %
\definecolor{basebg}{HTML}{E4EFF7}   %
\definecolor{gaincol}{HTML}{1A7A3A}  %
\definecolor{losscol}{HTML}{B2182B}  %
\definecolor{cigray}{HTML}{767676}   %
\definecolor{refgray}{HTML}{555555}  %
\newcommand{\best}[1]{\textbf{#1}}
\newcommand{\second}[1]{\underline{#1}}
\newcommand{\up}[1]{\textcolor{gaincol}{\textbf{#1}}}
\newcommand{\dn}[1]{\textcolor{losscol}{\textbf{#1}}}
\newcommand{\ns}[1]{\textcolor{cigray}{#1}}
\newcommand{\tabnote}[1]{\par\vspace{2pt}{\footnotesize\raggedright #1\par}}
\newcommand{\dstack}[2]{\begin{tabular}[t]{@{}c@{}}#1\\[-1.5pt]{\scriptsize\color{cigray}#2}\end{tabular}}
\newcommand{\rf}[1]{\textcolor{refgray}{#1}}
\newcommand{\grp}[1]{\textsc{#1}}
\newcommand{\oursrow}{\rowcolor{oursbg}}
\newcommand{\baserow}{\rowcolor{basebg}}
\usepackage{float}
\usepackage{placeins}
\makeatletter\newcommand{\captionof}[1]{\def\@captype{#1}\caption}\makeatother  %
\usepackage{enumitem}
\usepackage{pgfplots}
\pgfplotsset{compat=1.18}
\usepgfplotslibrary{groupplots}

\newcommand{\method}{ReMaRk}         %
\newif\ifarmunion
\newif\ifarmrw
\newif\ifarmmean
\armuniontrue      %

\newcommand{\E}{\mathbb{E}}
\newcommand{\1}{\mathbf{1}}
\newcommand{\phat}{\hat p}
\newcommand{\Var}{\mathrm{Var}}
\newcommand{\Cov}{\mathrm{Cov}}
\newcommand{\tr}{\mathrm{tr}}
\newcommand{\gS}{\hat g_{\mathrm S}}
\newcommand{\gM}{\hat g_{\mathrm M}}
\newcommand{\Jrow}{J_{\mathrm{row}}}
\newcommand{\Jreq}{J_{\mathrm{req}}}
\newtheorem{proposition}{Proposition}
\newtheorem{lemma}{Lemma}

\newfloat{algorithm}{tbp}{loa}
\floatname{algorithm}{Algorithm}
\setlist{nosep,leftmargin=*}

\title{Many Right Answers, One Training Example:\\ Removing Responder Noise from GRPO for LLM Recommenders}

\author{Xin Li \quad Hao Jiang \quad Jinghao Guo \quad Weisi Lin \quad Chau Yuen \\
  Nanyang Technological University, Singapore \\
  \url{https://lixin.ai/ReMaRk}}

\begin{document}
\maketitle

\begin{abstract}
A recommendation request on Reddit draws several replies, each naming different titles. When an LLM recommender is
trained on such threads with reinforcement learning (RL), which replies should a generated list be rewarded against?
The state-of-the-art recipe, Rank-GRPO, picks one reply per training example, so the advantage mostly records which
reply was drawn (85\% of its variance) rather than how good the list is. We argue that RL should train on the request,
not on each answer, and propose \method{}: it scores every list once against the union of all replies' titles and
samples each request equally. This makes RL cheaper. On seven backbones and three datasets, \method{} reaches the
validation recall of 2,000 per-reply RL steps in 25 to 53\% of the steps and leads at that budget; trained until
neither improves, it is never significantly behind Rank-GRPO and in four of seven pairs reaches Rank-GRPO's best with
fewer GPU-hours. The same reward helps on ProtoQA, where about 100 people answer each question, which points to a
simple rule for RL with several valid human answers: reward them jointly.
\end{abstract}

\section{Introduction}
\begin{figure}[t!]
\centering
\includegraphics[width=\columnwidth]{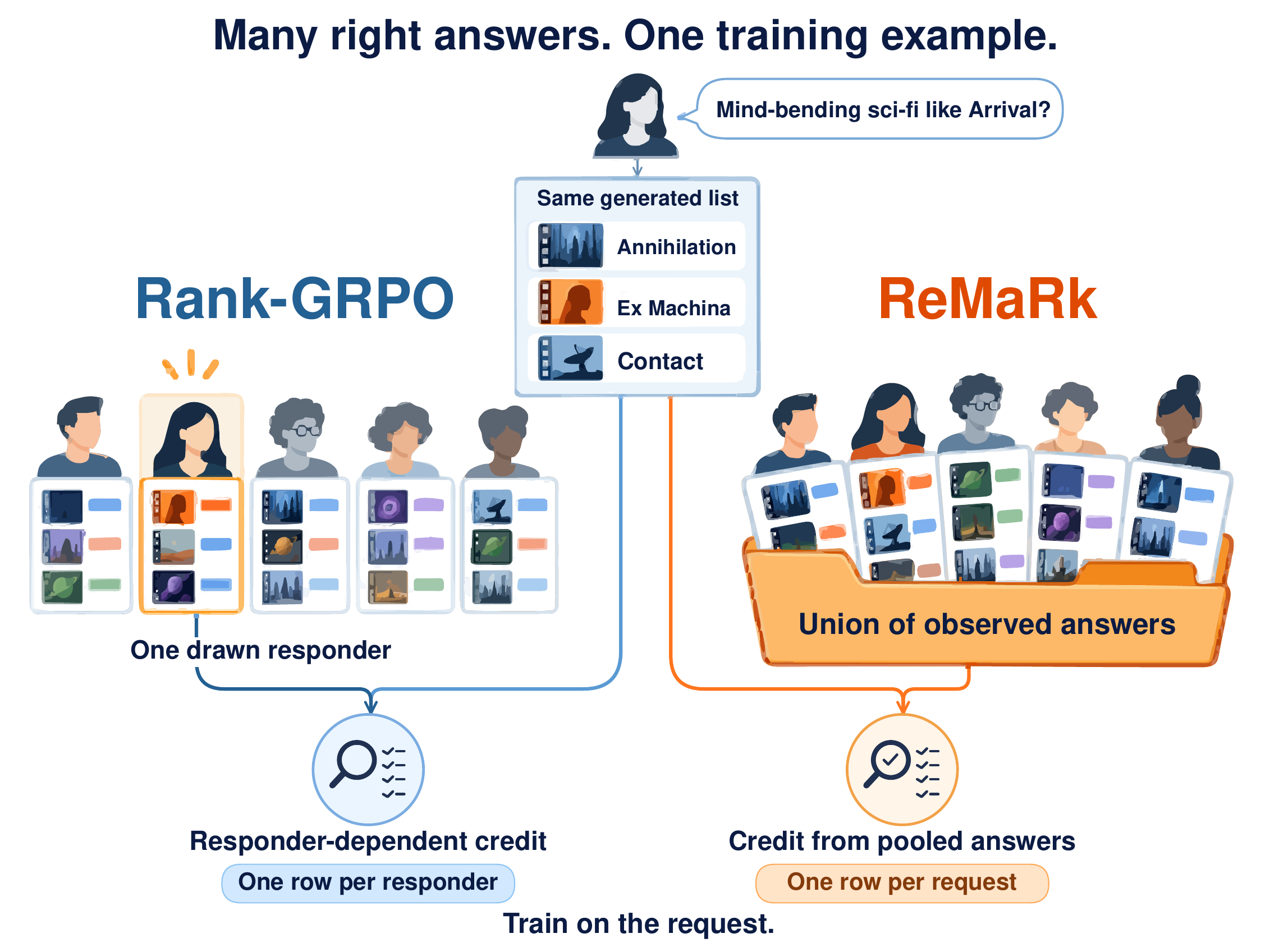}
\caption{One request, five responders and one generated list. Rank-GRPO credits the list against one drawn
responder, so its reward depends on which one was drawn; \method{} credits it once against the union of all
responders' items.}
\label{fig:teaser}
\end{figure}

\label{sec:intro}

People often ask online communities for recommendations. A Reddit user who wants ``mind-bending sci-fi like
\textit{Arrival}, not action-heavy'' gets replies from several people, each suggesting a few films
(Figure~\ref{fig:teaser}). Large language models (LLMs) can learn to answer such requests by writing a ranked list of
titles \citep{he2023large}. The strongest recipe, Rank-GRPO \citep{zhu2026rankgrpo}, first fine-tunes a small LLM on
example lists and then improves it with reinforcement learning (RL): the model writes several candidate lists for a
request, each list earns a reward for every title that people actually suggested, and training shifts the model toward
the lists that scored higher. Its data, Reddit-v2 \citep{he2023large,zhu2025crag}, have 5.3 people replying to a
request on average (\emph{responders}). Which of their answers should a list be rewarded against?

Rank-GRPO picks one: it turns each responder into a separate training example and rewards a list only for the titles
that this responder named. Responders rarely agree, and most suggested titles are named by a single person,
so a good list scores high or low depending on which responder happened to be drawn, much as
if an exam were graded against one of several valid answer keys picked at random. The disagreement is legitimate
variation, not annotation error, but the estimator turns it into noise: which responder was drawn accounts for 85\% of
the variance of the advantages, and in more than half of the comparisons between the eight lists every list gets the
same score, so the model learns nothing from them (Table~\ref{tab:data}).

We argue that RL should train on the request, not on each answer. The responder is a nuisance variable of how the
data were collected, and averaging the reward over it is a Rao--Blackwellization: at a fixed advantage scale the
expected update is unchanged and a variance term disappears (Proposition~\ref{prop:rb}). GRPO's per-group scaling
complicates this, because an averaged reward is fractional and dividing by a group's standard deviation can amplify
one responder's vote $n$-fold, whereas the scale of a 0/1 reward stays bounded (Proposition~\ref{prop:norm}). \method{}
(\textbf{Re}sponder-\textbf{Ma}rginalized \textbf{R}an\textbf{k}-level RL) therefore keeps a 0/1 reward: it rewards
each position of a list against the union of all the request's responders' titles and samples each request equally
instead of once per responder (Figure~\ref{fig:teaser}); nothing else in Rank-GRPO changes. Lower variance predicts the
same end point reached sooner, so we measure the RL compute needed to reach a target, the cost that dominates RL
post-training \citep{devvrit2026the}.

We find:
\begin{enumerate}[leftmargin=*,itemsep=2pt,topsep=2pt]
\item \textbf{The same quality with less RL compute.} In all ten backbone and dataset settings \method{} reaches the
2,000-step validation R@20 of its baseline in 25 to 53\% of the steps and leads by 0.8 to 1.6 points at 2,000 steps
(Table~\ref{tab:main}). Trained until neither improves, it is never significantly behind on test and in four of seven
pairs reaches Rank-GRPO's best with 18 to 63\% of the GPU-hours. The final level is a weaker claim: at 9B and on Reddit-Books,
Rank-GRPO's validation peak ends 0.08 to 0.36 points higher (not significant).
\item \textbf{Pooling the responders matters, not the aggregate.} Once list lengths are controlled, the averaged reward
is level with the union at 2B ($-0.08$ $[-0.45,+0.29]$, three seeds), and per-group scaling helps both
(Table~\ref{tab:ablation}).
\item \textbf{The gain follows the variance removed.} The exact Rao--Blackwellization of Rank-GRPO's update gains
$+0.93$ $[+0.83,+1.04]$ (three seeds), and the gain grows with the number of responders and with the share of
responder noise removed (Figure~\ref{fig:noise}).
\item \textbf{Beyond recommendation, the answers should be rewarded jointly.} The union also beats per-responder rows
under ordinary sequence-level GRPO (Appendix~\ref{app:mech}) and, on ProtoQA \citep{boratko-etal-2020-protoqa}, where
about 100 people answer each question, gains $+3.67$ $[+2.64,+4.69]$ recall@10 at 250 steps (three seeds). Unlike
Mult-DPO \citep{zhu2026multdpo} and A$^2$Search \citep{zhang2026a2search}, which pool answers offline or never split
them by annotator, we keep the online GRPO update and change only its training unit. We release Reddit-Books and
Reddit-Music as IDs and labels.
\end{enumerate}

\section{Related Work}
\label{sec:related}

\paragraph{LLM conversational recommenders.}
LLMs recommend zero-shot from dialogue \citep{he2023large,zhu2025crag,xie2024nbcrs}; trained recommenders combine
supervised fine-tuning (SFT) with RL that assigns credit per rank position \citep{zhu2026rankgrpo}, simulated users
\citep{feng-etal-2025-expectation,xiang2026smtpo} or other rewards \citep{yue2026rar,lin2025recr,xie2025recllmr1,zheng2025deeprec,wen2025usbrec}.
\method{} changes the reward and training unit of such GRPO trainers and could be combined with them; with one
recommender per request (ReDial, INSPIRED; \citealp{li2018redial,hayati-etal-2020-inspired}) it reduces to Rank-GRPO.
\citet{xu-etal-2025-beyond} counter single labels in CRS data by generating extra positives with an LLM; we use the
labels that several human responders already give, inside the RL reward. Mult-DPO \citep{zhu2026multdpo} also uses every item any
responder named, but trains offline with a set-wise loss and ranks the catalog; an offline DPO counterpart with the
same pooled positives trails by 5.5 to 6.9 points (\S\ref{sec:exp}).

\begin{figure*}[t]
\centering
\includegraphics[width=\textwidth]{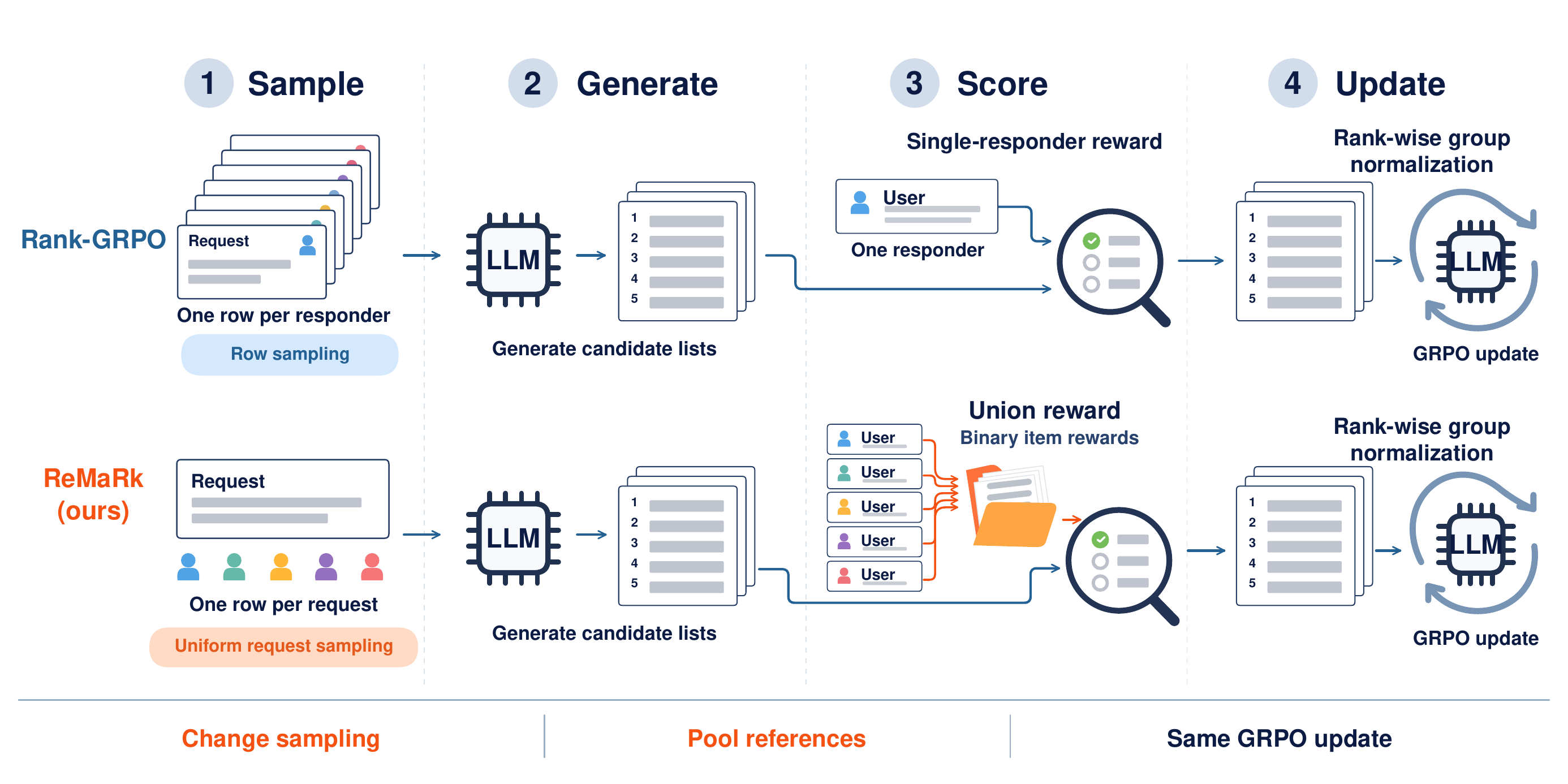}
\caption{\textbf{Rank-GRPO scores each sampled list against one drawn responder; \method{} scores it once
against all of them.} Top: one row per responder, drawn by row sampling, rewarded by that responder alone. Bottom: one row per request,
drawn uniformly, with binary item rewards against the union of the responders' items. Both keep Rank-GRPO's rank-wise
group normalization (per-group scaling) and GRPO update.}
\label{fig:method}
\end{figure*}

\paragraph{Multiple references and label variation.}
Disagreement between annotators carries information \citep{plank-2022-problem,uma2021learning,davani2022dealing} beyond what
aggregation into one label keeps \citep{dawid1979maximum,raykar2010learning}, and soft labels from
several annotators improve and stabilize supervised training \citep{peterson2019human,pmlr-v139-menon21a}, the
supervised analogue of Proposition~\ref{prop:rb}. Scoring an output against all references at once is standard in
multi-reference training \citep{zheng2018multi}, evaluation \citep{surana2026musicrs} and self-critical training with CIDEr
\citep{rennie2017self,vedantam2015cider}, and crowd or annotator distributions have served as rewards
\citep{li2024aligning,wu2026shala,siththaranjan2024distributional}. Our contribution is narrower: the common practice
of turning one prompt into a row per reference costs GRPO most of its signal, and pooling the references recovers it. Rao--Blackwellization of policy gradients is classical
\citep{casella1996rao,greensmith2004variance,kool2020estimating,ciosek2018expected}; VeriFree \citep{zhou2026verifree} and \citet{tang2025beyond} apply it to the
sampled answer in LLM RL, and we apply it to the responder. A$^2$Search \citep{zhang2026a2search} rewards QA answers
with an F1 against a mined answer set that is not split by annotator. Methods for noisy rewards estimate the noise
\citep{elmansouri2025noise,wang2020perturbed}; here its source is observed. Other GRPO variants average over the
policy's own samples \citep{wang2026grpoma,ichihara2026consensus,zuo2025ttrl}, not over human references.

\paragraph{Advantage normalization.}
GRPO divides group-centred rewards by the group standard deviation \citep{shao2024deepseekmath}; RLOO does not scale
\citep{kool2019buy,ahmadian-etal-2024-back}, and per-question scaling biases the update \citep{liu2025understanding}.
Batch-level scaling \citep{hu2025reinforcepp,liu2026tricks,rastogi2025magistral}, the inflation of fractional rewards
by small group deviations \citep{salmani-zarchi-etal-2026-mdp,wang2025prefgrpo,ding2026advantage}, filtering of tied
groups \citep{yu2025dapo} and the benefit of group scaling for binary rewards \citep{chu2025gpg} are known; we add a
bound that separates binary from fractional rewards. Reward variance governs how fast policy gradients learn
\citep{razin2025makes}, which is why groups without signal matter. ShapE-GRPO \citep{ai2026shapegrpo} splits a
set-level reward among the candidates in one output; it assigns credit across candidates, \method{} removes variance
across references, and the two could compose.

\section{Task and Setup}
\label{sec:setup}

\paragraph{Task.}
Each example is a Reddit request and the people who answered it. A \emph{request} $x$ is the text of the conversation
up to a recommendation turn: the seeker's post and any earlier replies on the same path. The recommender generates a list $L=(L_1,\dots,L_K)$ of $K=20$ lines of the form
``Title (Year)''. Line $k$ counts if it parses, is in the catalog, does not repeat an earlier line, and does not name a
title already mentioned in the request; $e_k(L;x)\in\{0,1\}$ indicates this. A request has $n_x\ge1$
\emph{responders}, users who replied with recommendations, and responder $r$ named the item set $S_r$. Following
\citet{zhu2026rankgrpo}, position $k$ earns the per-responder reward
\begin{equation}
q_k(L;S_r)=e_k(L;x)\,\1\{L_k\in S_r\},
\label{eq:single}
\end{equation}
where $L_k\in S_r$ means that the normalized titles match and the release years differ by at most two. Figure~\ref{fig:method} shows how Rank-GRPO turns these data into training rows and how
\method{} changes them.

\begin{table}[t]
\centering\footnotesize
\setlength{\tabcolsep}{2.8pt}
\begin{tabular}{@{}l rrr@{}}
\toprule
 & Reddit-v2 & Books & Music \\
\midrule
training requests & 72,233 & 32,539 & 19,159 \\
training rows & 383,013 & 149,679 & 142,396 \\
validation requests & 1,752 & 1,254 & 1,966 \\
test requests & 2,050 & 1,150 & 1,951 \\
responders: mean (median) & 5.30 (2) & 4.60 (3) & 7.43 (5) \\
single-responder (\%) & 41.0 & 18.1 & 18.3 \\
items per responder & 2.42 & 1.90 & 1.82 \\
items named once (\%) & 88.0 & 95.2 & 95.9 \\
catalog items & 7,404 & 74,474 & 38,881 \\
\midrule
\multicolumn{4}{@{}l}{\grp{Reward noise at the 2B SFT policy}} \\
s.d.\ across responders & 0.175 & 0.154 & 0.100 \\
s.d.\ across lists & 0.022 & 0.007 & 0.0005 \\
no signal: own responder (\%) & 55.0 & 62.9 & 80.1 \\
no signal: all responders (\%) & 12.4 & 22.2 & 38.4 \\
responder share of adv.\ (\%) & 85 & 66 & 79 \\
\bottomrule
\end{tabular}
\caption{Training-split statistics and the cost of per-responder rows (definitions in Appendix~\ref{app:repro}).}
\label{tab:data}
\end{table}

\paragraph{Data.}\looseness=-1
Table~\ref{tab:data} summarizes the datasets; each row is a (request, responder) pair, and all are split in time.
We use Reddit-v2 \citep{he2023large,zhu2025crag} as released with Rank-GRPO \citep{zhu2026rankgrpo}, where a Reddit
submission yields one request per reply path (48,618 training submissions). We build Reddit-Books, in the spirit of forum-based book suggestion data \citep{koolen2012socialbook}, from
r/suggestmeabook (training 2014-05 to 2017-07, validation 2017-08, test 2017-09): responders are distinct
top-level commenters other than the submitter, and an LLM extracts the books they recommend, which we ground to the
UCSD Goodreads book graph \citep{wan2018item,wan-etal-2019-fine} as ``Title (Year)'' with the first-publication year
(precision 0.997, recall 0.969, year agreement 0.903; Appendix~\ref{app:books}). Reddit-Music applies the same
pipeline to r/musicsuggestions with MusicBrainz artists as items (Appendix~\ref{app:music}). Both add catalog
grounding, responder structure and a temporal split to earlier collections of such threads.

\section{Method}
\label{sec:method}
We first restate Rank-GRPO's update and then change two of its parts, the reward and the sampling
(Figure~\ref{fig:method}).

\paragraph{Rank-GRPO.}
Rank-GRPO samples $G$ lists per request and, at each rank position, compares every list's reward with the group
average; a list that beats the average at a position becomes more likely there. Formally, for request $x$ the policy
samples $G$ lists. Position $k$ of list $g$ receives a reward $R^g_k$ and the advantage
$A^g_k=(R^g_k-\bar R_k)/\sigma_k$, with $\bar R_k$ the group mean and $\sigma_k$ the group's standard deviation at
position $k$ (plus $10^{-4}$), applied to the tokens of line $k$ in a clipped surrogate with a KL penalty
\citep{schulman2017ppo}; on-policy the update is $\hat g=\frac1G\sum_{g,k}A^g_k\,s^g_k$ with $s^g_k$ the score
function of line $k$ \citep{williams1992simple}. Rank-GRPO rewards $R^g_k=q_k(L^g;S_r)$, a hit against the item set
$S_r$ of the responder $r$ of the sampled row.

\subsection{Integrating out the responder}
\label{sec:rmr}
Averaging the reward over all of a request's responders removes its dependence on which one was drawn. Let $c_i$ be
the number of the request's responders who named item $i$ and $\phat_i=c_i/n_x$. With $r$ uniform over
the responders,
\begin{equation}
\E_{r}\big[q_k(L;S_r)\big]=e_k(L;x)\,\phat_{L_k}(x),
\label{eq:rmr}
\end{equation}
so the $n_x$ rows of a request become one training example whose reward no longer depends on which responder was drawn
(Figure~\ref{fig:method}). With $\sigma$ fixed, $\hat g$ is linear in the item reward (Lemma~\ref{lem:linear},
Appendix~\ref{app:proofs}): write
$\gS$ for the estimator with $Y(i)=\1\{i\in S_r\}$ and $\gM$ for the one with $Y(i)=\phat_i$.

\begin{proposition}[Rao--Blackwellization over responders]
\label{prop:rb}
Fix $x$ and the lists $L^{1:G}$, let $r\sim\mathrm{Unif}\{1,\dots,n_x\}$ independently of the lists, and let $\sigma$
not depend on this group's rewards. Then $\E_r[\gS\mid x,L^{1:G}]=\gM$, and over the draw of requests, lists and
responders $\tr\Var(\gS)-\tr\Var(\gM)\ge0$, with equality when every generated item was named by all or by none of
the responders, in particular when $n_x=1$.
\end{proposition}

Marginalizing also gives more groups a signal: a group's rewards are all zero only if no list hits any responder's
item (Table~\ref{tab:data}). The exact Rao--Blackwellization of Rank-GRPO's own group-scaled update averages its
per-responder advantages, at one advantage computation per responder, and gains $+0.93$ points (\S\ref{sec:exp}); a
single averaged reward per list is cheaper but meets per-group scaling, which Proposition~\ref{prop:rb} does not cover.

\subsection{Fractional rewards under per-group scaling}
\label{sec:norm}
GRPO divides each group's advantages by the group's standard deviation. For 0/1 rewards this is harmless; for the
averaged reward $\phat$, whose values lie on a grid of spacing $1/n_x$, it is not:
\begin{proposition}[binary versus fractional rewards]
\label{prop:norm}
With $G=8$, per-group scaling is within a factor 1.51 of a constant at each position when rewards are 0/1, but for
rewards on a grid of spacing $1/n_x$ it can amplify one responder's vote $n_x$-fold.
\end{proposition}
A group in which one list hits an item named by one of 40 responders has a tiny standard deviation, so that single
vote receives the same advantage as a hit on an item all 40 named. Without length control the mean with per-group
scaling diverged on Llama-3.2-3B (KL rose from 0.6 to 47 after about 850 steps); with the length control of
\S\ref{sec:training} it trains on Qwen3.5-2B and is level with the union (\S\ref{sec:exp}, Appendix~\ref{app:p0}), so
Proposition~\ref{prop:norm} describes a worst case, not a typical outcome. With a batch-level scale (the standard
deviation of position $k$ over the batch) the mean keeps improving and Proposition~\ref{prop:rb} holds up to the
$O(1/B)$ effect of one group ($B=384$ lists per step), but per-group scaling helps both the union and the mean at 2B.
We use the 0/1 union because its per-group scale is bounded.

\subsection{\method{}: the union, sampled per request}
\label{sec:agg}
\method{} rewards a title if any of the request's responders named it, the \emph{bounded sum} $\min(c_i,1)$: the
union of the responders' items. It keeps Rank-GRPO's 0/1 rewards and per-group scaling, removes the responder draw from
the reward, and equals the sum of the responders' hits on 89\% of union items over all training requests.
Requests are sampled uniformly rather than once per responder. With a shared scale, the sum with uniform request
sampling has a constant multiple of Rank-GRPO's expected update, and the mean with uniform sampling targets the
per-request objective (Proposition~\ref{prop:agg}, Appendix~\ref{app:proofs}); the union approximates the former.
The union is not the marginal of Proposition~\ref{prop:rb}: it changes the target, scoring an item named by one
responder like one named by all, and under per-group scaling neither the unbiasedness nor the variance bound carries
over. Averaging explains and removes the responder noise; the union is the practical design that keeps the pooled
support and the 0/1 values that per-group scaling handles, and its gains are established by experiment (\S\ref{sec:exp}).
Everything else is unchanged (Algorithm~1, Appendix~\ref{app:alg}). We call each combination of reward, advantage
scale and sampling an \emph{arm}.

\subsection{Training details}
\label{sec:training}
Every arm starts from the same SFT checkpoint and uses Rank-GRPO's hyperparameters ($G=8$, 48 requests per step,
2,000 steps; Appendix~\ref{app:hparams}). A length control (small length-shaping rewards and a 320-token cap, 384 on Reddit-Books), applied identically to
every arm, keeps Qwen3.5 lists near 20 items; because it is scaled by each run's mean advantage scale, its effective
weight is 6\% lower for \method{} than for Rank-GRPO in every measured pair (Appendix~\ref{app:p0}; \S\ref{sec:exp} matches it). We train with TRL \citep{vonwerra2020trl} and vLLM rollouts \citep{kwon2023vllm}.

\begin{table*}[t]
\centering\small
\renewcommand{\arraystretch}{1.1}
\setlength{\tabcolsep}{3.4pt}
\renewcommand{\dstack}[2]{\begin{tabular}[t]{@{}c@{}}#1\\[-1.5pt]{\scriptsize\color{refgray}#2}\end{tabular}}
\begin{tabular}{@{}l ccccc cc cc c@{}}
\toprule
 & \multicolumn{7}{c}{Reddit-v2} & \multicolumn{2}{c}{Reddit-Books} & Music \\
\cmidrule(lr){2-8}\cmidrule(lr){9-10}\cmidrule(l){11-11}
 & \multicolumn{5}{c}{Qwen3.5} & \multicolumn{2}{c}{Llama} & \multicolumn{2}{c}{Qwen3.5} & Qwen3.5 \\
\cmidrule(lr){2-6}\cmidrule(lr){7-8}\cmidrule(lr){9-10}\cmidrule(l){11-11}
 & 2B & 4B & 9B & 27B & 35B-A3B & 3B & 8B & 2B & 4B & 2B \\
\midrule
SFT & 15.40 & 16.53 & 17.23 & 17.90 & 17.34 & 15.38 & 16.84 & 20.67 & 21.73 & 9.31 \\
\baserow Rank-GRPO & \second{20.97} & \second{23.56} & \second{24.05} & \second{25.17} & \second{24.74} & 21.93 & \second{24.22} & 29.34 & 30.40 & 13.40 \\
Per-resp., batch & 20.43 & 23.16 & 23.95 & 23.96 & 24.58 & \second{22.28} & 23.89 & \second{29.35} & \second{30.51} & \second{13.71} \\
\oursrow \method{} (ours) & \best{22.53} & \best{25.10} & \best{25.51} & \best{26.62} & \best{26.20} & \best{23.56} & \best{25.22} & \best{30.16} & \best{31.49} & \best{14.48} \\
\addlinespace[1pt]
$\Delta$ per row & \dstack{\up{+1.57}}{[1.10,2.06]} & \dstack{\up{+1.54}}{[1.10,2.00]} & \dstack{\up{+1.46}}{[1.02,1.93]} & \dstack{\up{+1.45}}{[0.98,1.94]} & \dstack{\up{+1.47}}{[1.03,1.90]} & \dstack{\up{+1.28}}{[0.86,1.71]} & \dstack{\up{+1.00}}{[0.55,1.44]} & \dstack{\up{+0.81}}{[0.19,1.46]} & \dstack{\up{+0.98}}{[0.35,1.64]} & \dstack{\up{+0.78}}{[0.38,1.17]} \\
$\Delta$ per request & \dstack{\up{+2.29}}{[1.43,3.17]} & \dstack{\up{+2.58}}{[1.82,3.41]} & \dstack{\up{+2.40}}{[1.64,3.22]} & \dstack{\up{+1.46}}{[0.68,2.25]} & \dstack{\up{+1.95}}{[1.14,2.78]} & \dstack{\up{+1.73}}{[0.92,2.52]} & \dstack{\up{+1.35}}{[0.55,2.16]} & \dstack{\ns{+0.70}}{[$-$0.15,1.52]} & \dstack{\up{+1.36}}{[0.51,2.27]} & \dstack{\up{+1.08}}{[0.49,1.69]} \\
steps to match & \dstack{0.52}{[0.42,0.64]} & \dstack{0.51}{[0.35,0.69]} & \dstack{0.42}{[0.36,0.60]} & \dstack{0.25}{[0.22,0.46]} & \dstack{0.35}{[0.32,0.47]} & \dstack{0.53}{[0.47,0.62]} & \dstack{0.31}{[0.24,0.41]} & \dstack{0.48}{[0.36,0.66]} & \dstack{0.47}{[0.43,0.75]} & \dstack{0.33}{[0.29,0.37]} \\
\addlinespace[2pt]
$\Delta$, batch scaling & \dstack{\up{+0.87}}{[0.47,1.28]} & \dstack{\up{+1.46}}{[0.97,1.92]} & \dstack{\up{+1.47}}{[1.03,1.93]} & \dstack{\up{+1.15}}{[0.68,1.67]} & \dstack{\up{+1.42}}{[0.98,1.91]} & \dstack{\up{+1.15}}{[0.70,1.62]} & \dstack{\up{+1.45}}{[0.98,1.91]} & \dstack{\up{+1.25}}{[0.65,1.91]} & \dstack{\up{+0.90}}{[0.22,1.62]} & \dstack{\up{+0.80}}{[0.39,1.19]} \\
\midrule
\multicolumn{11}{@{}l}{\grp{Reference (no training)}} \\
Popularity & \multicolumn{7}{c}{\rf{2.75}} & \multicolumn{2}{c}{\rf{5.77}} & \rf{7.97} \\
kNN prior & \multicolumn{7}{c}{\rf{15.94}} & \multicolumn{2}{c}{\rf{21.60}} & \rf{9.17} \\
GPT-5.6 zero-shot & \multicolumn{7}{c}{\rf{22.63}} & \multicolumn{2}{c}{\rf{23.66}} & \rf{7.40} \\
\bottomrule
\end{tabular}

\caption{Per-row test R@20 (\%), 2,000 steps; $\Delta$: \method{} minus the better per-responder arm; seed 3407 (App.~\ref{app:repro}).}
\label{tab:main}
\end{table*}

\begin{figure*}[t]
\begin{minipage}[c]{0.48\textwidth}
\centering
\includegraphics[width=\linewidth]{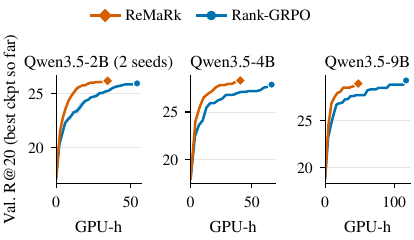}
\caption{Best validation R@20 so far against training GPU-hours for Reddit-v2 pairs trained until neither improves; each line ends at the selected checkpoint.}
\label{fig:gpuh}
\end{minipage}\hfill
\begin{minipage}[c]{0.48\textwidth}
\vspace{24pt}
\centering\footnotesize
\setlength{\tabcolsep}{4pt}
\begin{adjustbox}{max width=\linewidth}
\begin{tabular}{@{}l ccc@{}}
\toprule
 & 2B & 4B & 9B \\
\midrule
1k vs.\ 2k steps & \begin{tabular}[t]{@{}c@{}}\ns{+0.35}\\[-1.5pt]{\scriptsize\color{cigray}[$-$0.07,0.77]}\\[-1.5pt]{\scriptsize\color{cigray}7 / 13 GPU-h}\end{tabular} & \begin{tabular}[t]{@{}c@{}}\up{+0.87}\\[-1.5pt]{\scriptsize\color{cigray}[0.42,1.33]}\\[-1.5pt]{\scriptsize\color{cigray}14 / 26 GPU-h}\end{tabular} & \begin{tabular}[t]{@{}c@{}}\up{+0.60}\\[-1.5pt]{\scriptsize\color{cigray}[0.16,1.10]}\\[-1.5pt]{\scriptsize\color{cigray}17 / 33 GPU-h}\end{tabular} \\
2k vs.\ 4k steps & \begin{tabular}[t]{@{}c@{}}\up{+0.63}\\[-1.5pt]{\scriptsize\color{cigray}[0.18,1.04]}\\[-1.5pt]{\scriptsize\color{cigray}13 / 25 GPU-h}\end{tabular} & \begin{tabular}[t]{@{}c@{}}\up{+0.80}\\[-1.5pt]{\scriptsize\color{cigray}[0.35,1.25]}\\[-1.5pt]{\scriptsize\color{cigray}28 / 51 GPU-h}\end{tabular} & -- \\
best vs.\ best & \begin{tabular}[t]{@{}c@{}}\ns{+0.14}\\[-1.5pt]{\scriptsize\color{cigray}[$-$0.30,0.59]}\\[-1.5pt]{\scriptsize\color{cigray}30 / 43 GPU-h}\end{tabular} & \begin{tabular}[t]{@{}c@{}}\up{+0.83}\\[-1.5pt]{\scriptsize\color{cigray}[0.39,1.29]}\\[-1.5pt]{\scriptsize\color{cigray}41 / 66 GPU-h}\end{tabular} & \begin{tabular}[t]{@{}c@{}}\ns{+0.08}\\[-1.5pt]{\scriptsize\color{cigray}[$-$0.34,0.50]}\\[-1.5pt]{\scriptsize\color{cigray}48 / 117 GPU-h}\end{tabular} \\
\bottomrule
\end{tabular}

\end{adjustbox}
\captionof{table}{Compute-matched test comparisons on Reddit-v2: \method{} minus Rank-GRPO (R@20), GPU-hours \method{} / Rank-GRPO; --: not run.}
\label{tab:compute}
\end{minipage}
\end{figure*}

\section{Experiments}
\label{sec:exp}

\paragraph{Evaluation.}
We follow Rank-GRPO's protocol. The model writes one list per test request with greedy decoding, and each (request,
responder) row scores that list against the responder's items: Recall@$K$ (R@$K$) is the share of those items in the
top $K$, and NDCG@$K$ (N@$K$) also credits placing them early.
These \emph{per-row} metrics count a request once per responder (the average test row belongs to a request with 17.2
responders, the average test request has 5.35), so we also report \emph{per-request} metrics, which average within each
request first \citep{manning2008introduction}. Metrics are in \%, and differences in points with 95\% CIs from a
paired bootstrap over test requests \citep{koehn-2004-statistical} (2,000 resamples, seed 42). Checkpoints are
selected on validation every 250 steps and tested once (Appendix~\ref{app:repro}).

\paragraph{Models and arms.}
We train Qwen3.5 2B, 4B, 9B and 35B-A3B (mixture of experts) base models and the post-trained Qwen3.5-27B \citep{qwen35}, each after our SFT; Llama-3.2-3B-Instruct \citep{grattafiori2024llama}
from the SFT checkpoint released with Rank-GRPO; and Llama-3.1-8B-Instruct with Rank-GRPO's SFT recipe.
An arm is a triple (aggregation, scale, sampling), listed in Table~\ref{tab:ablation}; for a request-level reward,
sampling by row draws requests in proportion to their responders. Rank-GRPO is (per-responder, group, row), retrained
with the same budget and length control, and \method{} is (bounded sum, group, request). The primary comparison is
\method{} minus the better of the two per-responder arms with row sampling (chosen on validation), in per-row test
R@20. We chose the bounded sum on 2B validation among three aggregations (all within their CIs).
Seeds: 3407, then 42 and 1234.
Reference systems are popularity, a kNN prior (Appendix~\ref{app:repro}), zero-shot GPT-5.6 and
Qwen3.5-27B, the released Rank-GRPO checkpoint and GPT-4o \citep{zhu2026rankgrpo}.

\begin{table*}[t]
\begin{minipage}[t]{0.5\textwidth}
\vspace{0pt}\centering\footnotesize
\setlength{\tabcolsep}{1.7pt}
\renewcommand{\arraystretch}{1.08}
\begin{tabular}{@{}r ll r l r@{}}
\toprule
 & reward & sampling & R@20 & $\Delta$ to row 1 & req. \\
\midrule
0 & \multicolumn{2}{@{}l}{SFT (no RL)} & 15.40 & & 15.32 \\
\multicolumn{6}{@{}l}{\grp{group scaling}} \\
\baserow 1 & per-responder & row & 20.97 &  & 20.94 \\
8 & per-responder & request & 21.13 & \ns{+0.16} {\scriptsize\color{cigray}[$-$0.33,0.66]} & 22.13 \\
9 & mean & request & \best{22.53} & \up{+1.56} {\scriptsize\color{cigray}[1.10,2.05]} & \second{23.41} \\
\oursrow 7 & bounded sum & request & \best{22.53} & \up{+1.57} {\scriptsize\color{cigray}[1.10,2.06]} & 23.24 \\
\multicolumn{6}{@{}l}{\grp{batch scaling}} \\
2 & per-responder & row & 20.43 & \dn{$-$0.53} {\scriptsize\color{cigray}[$-$0.91,$-$0.14]} & 19.92 \\
3 & per-responder & request & 20.93 & \ns{$-$0.03} {\scriptsize\color{cigray}[$-$0.52,0.45]} & 22.19 \\
4 & mean & row & 21.79 & \up{+0.82} {\scriptsize\color{cigray}[0.37,1.27]} & 22.25 \\
5 & mean & request & \second{21.95} & \up{+0.98} {\scriptsize\color{cigray}[0.51,1.49]} & \best{23.60} \\
6 & bounded sum & request & 21.84 & \up{+0.87} {\scriptsize\color{cigray}[0.47,1.28]} & 22.38 \\
\bottomrule
\end{tabular}

\caption{Ablations on Qwen3.5-2B (Reddit-v2 test, R@20 in \%). Row 1 is Rank-GRPO, row 7 \method{}.}
\label{tab:ablation}
\end{minipage}\hfill
\begin{minipage}[t]{0.46\textwidth}
\vspace{0pt}\centering\small
\setlength{\tabcolsep}{2.5pt}
\begin{tabular}{@{}l cc@{}}
\toprule
\method{} minus & 2B & 4B \\
\midrule
Rank-GRPO, $G=16$ lists & \dstack{\up{+1.81}}{[1.35,2.29]} & \dstack{\up{+2.06}}{[1.62,2.53]} \\
Rank-GRPO, dynamic sampling & \dstack{\up{+1.34}}{[0.92,1.74]} & \dstack{\up{+1.31}}{[0.91,1.68]} \\
Rank-GRPO, Rao--Blackwellized & \dstack{\up{+0.63}}{[0.23,1.00]} & \dstack{\ns{+0.31}}{[$-$0.07,0.70]} \\
SFT on pooled positives & \dstack{\up{+6.59}}{[5.90,7.26]} & \dstack{\up{+7.85}}{[7.10,8.60]} \\
DPO on pooled positives & \dstack{\up{+5.45}}{[4.87,6.02]} & \dstack{\up{+6.93}}{[6.25,7.59]} \\
GPT-5.6 zero-shot & \dstack{\ns{$-$0.10}}{[$-$1.22,1.01]} & \dstack{\up{+2.47}}{[1.39,3.49]} \\
\addlinespace[2pt]
Rank-GRPO, both arms $G=16$ & \dstack{\up{+1.43}}{[0.98,1.89]} & \dstack{\up{+1.65}}{[1.21,2.11]} \\
\bottomrule
\end{tabular}
\caption{Stronger baselines: \method{} minus each system (Reddit-v2, per-row test R@20).}
\label{tab:checks}
\end{minipage}
\end{table*}

\begin{figure*}[t]
\centering
\includegraphics[width=\textwidth]{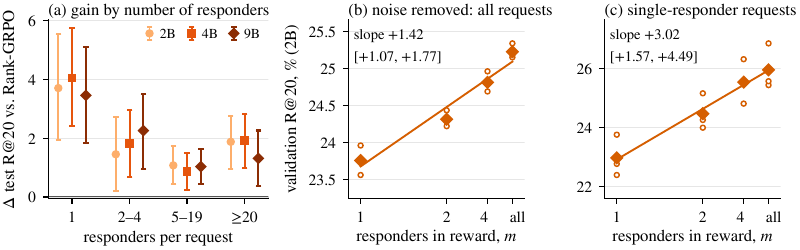}
\caption{The gain follows the responder noise removed. (a) Test gain of \method{} over Rank-GRPO by the request's
number of responders (seed 3407, 95\% CIs). (b, c) Dose-response at 2B on validation, batch scaling: R@20 when the reward averages
$m=1,2,4$ or all responders, spaced by the share of noise removed (three seeds; line: least squares), for all
and for single-responder requests.}
\label{fig:noise}
\end{figure*}

\paragraph{\method{} reaches the same quality with about half the RL compute.}
Figure~\ref{fig:gpuh} plots validation recall against GPU-hours for pairs trained until neither improves: \method{}
climbs faster at every scale. At the 2,000-step budget it reaches the baseline's 2,000-step score in a quarter to half of the steps in all
ten settings (Table~\ref{tab:main}, steps to match) and is ahead at nearly every checkpoint validated for both arms
(Appendix~\ref{app:robust}). Matched for GPU-hours, \method{} with half of Rank-GRPO's steps is level at 2B and ahead at
4B and 9B on test, and at 2B and 4B it beats Rank-GRPO trained twice as long (Table~\ref{tab:compute}); the GPU-hours include
\method{}'s 4 to 11\% slower step, and validation adds 2 to 4\% to either arm (Appendix~\ref{app:repro}).

The final level is a weaker claim: in all seven plateau pairs \method{} is never significantly behind on test, and four reach Rank-GRPO's best sooner; on
Reddit-Books and at 9B Rank-GRPO's validation peak stays slightly higher, at 9B after 2.4 to 2.5 times the GPU-hours
(Appendix~\ref{app:compute}). \method{} can peak early and decline (Figure~\ref{fig:gpuh-raw}), so the saving relies
on validation-based selection; locked at the first checkpoint matching Rank-GRPO's 2,000-step score, it is level or
ahead on test at 2B, 4B and 9B. At 2B, neither coverage nor the step size explains the speed-up:
request sampling alone never reaches Rank-GRPO's 2,000-step score, a larger learning rate speeds Rank-GRPO early
but leaves it lower, and with a cosine schedule \method{} reaches Rank-GRPO's best in 24 to 30\% of its steps.

\begin{figure}[t]
\centering
\includegraphics[width=\columnwidth]{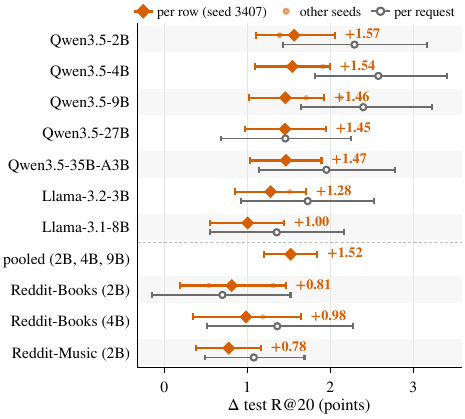}
\caption{Test gain of \method{} over the better per-responder baseline in every setting, with 95\% CIs. Filled: per
row (seed 3407); small dots: other seeds (Llama-3.2-3B seed 42: over Rank-GRPO); open: per request.}
\label{fig:main}
\end{figure}

\paragraph{At a fixed budget \method{} leads at every scale and on every dataset.}
With the same SFT start, length control and 2,000 RL steps, \method{} is ahead of Rank-GRPO with all seven backbones
and on all three datasets, by about 1.5 recall@20 points with Qwen3.5 on Reddit-v2 (Table~\ref{tab:main}; per-setting gains in Figure~\ref{fig:main}). The lead holds for further seeds, under resampling of whole submissions and at other cutoffs
(Appendix~\ref{app:robust}). With batch scaling
instead, \method{} still leads by $+0.80$ to $+1.47$ points (Table~\ref{tab:main}). On a later test set built after all checkpoints were selected (Reddit-Books threads from late 2017), the 4B pair gains $+0.67$ $[+0.34,+1.02]$ (Appendix~\ref{app:late}). On Llama-3.2-3B, 2,000 steps of \method{} come within 0.09 points of the released Rank-GRPO checkpoint trained for 15.8k steps (23.56 against 23.65).

\paragraph{More samples, filtering or offline training do not close the gap.}
Doubling Rank-GRPO's lists per request does not help it, nor do twice the requests per step at equal GPU-hours,
and DAPO-style filtering of groups without signal \citep{yu2025dapo} gains little at three times the time per step (Table~\ref{tab:checks},
Appendix~\ref{app:baselines}); \method{} turns those groups into signal instead. Offline SFT and DPO on all responders'
titles stay far below \method{} (Appendix~\ref{app:dpo}). Nor does the
length control explain the gain: with its effective weight matched to Rank-GRPO's (\S\ref{sec:training}), \method{}
stays $+1.18$ $[+0.73,+1.61]$ ahead ($+1.45$ $[+1.01,+1.88]$ without it; Appendix~\ref{app:p0}).

\paragraph{Both the reward and the sampling contribute.}
Table~\ref{tab:ablation} moves from Rank-GRPO to \method{} one component at a time. Sampling requests uniformly
leaves per-row recall nearly unchanged but raises per-request recall; the union reward then adds most of the per-row
gain ($+1.40$ $[+0.93,+1.86]$). The reward step is positive under either sampling and either aggregation, on test and
on validation (Tables~\ref{tab:ablation-diffs}, \ref{tab:ablation-val} and~\ref{tab:ablation-full}).
Even inside Rank-GRPO, the exact Rao--Blackwellization of its own update gains
$+0.93$ $[+0.83,+1.04]$ over three seeds, and \method{} stays above it, significantly at 2B (Table~\ref{tab:checks}). Per-group scaling helps both the union
($+0.69$, row 7 against 6) and the mean ($+0.58$, row 9 against 5) at 2B, and is
neutral for the union at larger sizes (Appendix~\ref{app:robust}). The averaged reward with per-group scaling (row 9) is
level with \method{} over three seeds ($-0.08$ $[-0.45,+0.29]$): the gain
comes from pooling the responders, not from the aggregate (Appendix~\ref{app:p0}).

\paragraph{The gain grows with the responder noise removed.}
The gain is positive for every number of responders and largest on single-responder requests
(Figure~\ref{fig:noise}a). There the union equals Rank-GRPO's reward and the gain comes from request sampling: row
sampling visits a request once per responder and sees only a quarter of the single-responder requests in 2,000 steps.
On requests with several responders the gain comes from the reward, whose step grows with the number of
responders (Table~\ref{tab:buckets}). Under batch scaling, controlling the noise directly agrees: averaging
the reward over more responders raises validation recall with the share of noise removed, by 1.4 points from none
to all (Figure~\ref{fig:noise}b), more on single-responder requests, whose own reward does not change
(Figure~\ref{fig:noise}c). With one fixed scale for all arms the slope is $+1.38$ $[+0.97,+1.78]$ (two seeds); per-group scaling erases it (Appendix~\ref{app:mech}).

\begin{table}[t]
\centering\small
\setlength{\tabcolsep}{4pt}
\begin{tabular}{@{}l cc@{}}
\toprule
 & 9B, 2k steps & 4B, plateau \\
\midrule
\multicolumn{3}{@{}l}{\grp{pairwise preference for \method{}}} \\
GPT-5.6 & \dstack{\up{0.56}}{[0.52,0.61]} & \dstack{\ns{0.54}}{[0.49,0.59]} \\
Gemini 3.1 Pro & \dstack{\up{0.58}}{[0.54,0.63]} & \dstack{\ns{0.52}}{[0.48,0.57]} \\
\multicolumn{3}{@{}l}{\grp{precision@20, \method{} $-$ Rank-GRPO}} \\
GPT-5.6 & \dstack{\up{+1.27}}{[0.16,2.53]} & \dstack{\ns{+0.58}}{[$-$0.40,1.61]} \\
Gemini 3.1 Pro & \dstack{\ns{+1.03}}{[$-$0.28,2.36]} & \dstack{\ns{+0.82}}{[$-$0.61,2.11]} \\
\bottomrule
\end{tabular}
\caption{Blind LLM judges on Reddit-v2 test: share of pairwise wins (ties 0.5) and precision@20 difference (points).}
\label{tab:judges}
\end{table}

\paragraph{Judges of two model families agree, and the recall gain does not cost precision.}\label{sec:analysis}
Recall counts only titles that responders named, so we also ask LLM judges, which never see the responders. Blind
pairwise judges from two families prefer 9B \method{}'s lists to Rank-GRPO's and rate the 4B models trained to their
plateaus level (Table~\ref{tab:judges}). Asked to mark each listed title as fitting the request or not, both judges
find \method{}'s lists at least as precise (Appendix~\ref{app:judge}). Pooling responders does not trade single answers for consensus: 92 to 95\% of test union
items are named by one responder, and \method{} recalls more of them in every setting (by 2.1 to 2.4 points at 2B to 9B,
0.8 on Reddit-Books and 1.3 on Reddit-Music, all CIs above zero).
The union ignores how many
responders named an item, and its gain is larger on head than on tail items (Appendix~\ref{app:judge}). A bonus
for items named by two or more responders does not help (Appendix~\ref{app:mech}).

\paragraph{The principle carries over beyond Rank-GRPO.}
ProtoQA \citep{boratko-etal-2020-protoqa} records how many of about 100 people gave each answer to a
commonsense prompt; each respondent can be a row, as each responder is in Rank-GRPO. Over three seeds
at 2B, the union gains $+3.67$ $[+2.64,+4.69]$ per-respondent recall@10 over per-respondent rows, exact
Rao--Blackwellization about half as much ($+1.94$), and the batch-scaled mean is not significantly
better (Table~\ref{tab:protoqa}). With longer training the per-respondent arm catches up on test, but the union reaches that arm's
best validation recall in 42\% of the steps.

\section{Conclusion}
\label{sec:conclusion}

When a request has several valid human answers, RL should train on the request, not on each answer. Per-responder rows
make the responder draw the main source of variance in Rank-GRPO's update; \method{} pools the responders into one 0/1
reward and samples each request equally; at 2B the gain comes from the pooling, not from the aggregate. On seven
backbones and three datasets it reaches the recall of 2,000 per-responder steps in 25 to 53\% of the steps; trained
until neither improves, it is never significantly behind on test and in four of seven pairs reaches Rank-GRPO's best
with 18 to 63\% of the GPU-hours. Its 2,000-step lead holds on a later test set and with matched length shaping. The
change touches only how training rows are built and scored, so any GRPO data with one row per annotator can use it:
reward the answers jointly.

\clearpage
\section*{Limitations}

\paragraph{Scope.} Every aggregation equals the per-responder reward on requests with one responder, so on datasets
with one recommender per request (ReDial, INSPIRED) \method{} reduces to Rank-GRPO. Its value grows with the share of
requests that have several responders.

\paragraph{Data.} All three datasets come from Reddit. Reddit-Books and Reddit-Music rely on LLM extraction and
catalog grounding (accuracy in Appendices~\ref{app:books} and~\ref{app:music}).

\paragraph{Metric and reward share a target.} The marginalized reward's expectation is the recall target, so gains are
measured against the same responders the reward integrates over. Responders name few items, and recall against them
undercounts acceptable unnamed items for every method; judges from two LLM families that do not see the responders
address this only in part (Table~\ref{tab:judges}).

\paragraph{Budget and selection.} The matched-budget comparisons stop at 2,000 steps, where our arms still improve;
the released Rank-GRPO checkpoint trained for 15.8k steps. On Reddit-Books and at 9B \method{} does not reach
Rank-GRPO's validation peak, so there the speed claim holds up to Rank-GRPO's 2,000-step level. Long runs of both
methods show instabilities, so the speed-up depends on choosing the checkpoint on validation. Several backbones have a
single seed.

\paragraph{Hyperparameters and theory.} Hyperparameters were transferred from Rank-GRPO and not tuned for \method{};
other learning rates at 2B do not change the ordering (Appendix~\ref{app:robust}). Propositions~\ref{prop:rb}
and~\ref{prop:agg} do not apply under the per-group scaling that \method{} uses, and all propositions describe
one update under a fixed policy, not training dynamics.

\section*{Ethics Statement}

\paragraph{Data.} Reddit-v2 is a public research dataset built from public Reddit posts. Reddit-Books is built from
public r/suggestmeabook posts and comments retrieved through the Arctic Shift archive; we exclude deleted and removed
content and bot accounts. We will release only submission and comment IDs and the derived item labels, not post text or
user names, and we do not model individual users. Reddit's terms do not allow redistributing post text, so users rehydrate
the text from the IDs through the archive; deleted content stays unavailable. The UCSD Goodreads book graph is used for catalog grounding under its
academic-use terms and is not redistributed. Reddit-Music is built the same way from r/musicsuggestions and grounded to
MusicBrainz (CC0). ProtoQA (CC BY 4.0) holds questions and aggregated answer counts, with
no data on individual respondents.

\paragraph{LLM use.} A proprietary LLM (GPT-5.6) extracted book and artist mentions and wrote the SFT target lists for
Reddit-Books, Reddit-Music and ProtoQA. AI assistants (Claude, Anthropic; Codex, OpenAI) were used for coding and for
drafting and editing the text; the authors checked all content and take responsibility for it. LLMs that
see item titles may have memorized Reddit threads. A few-shot probe of the Qwen3.5 base models names at least one of
the responders' items from the request alone for 5.8\%, 12.6\% and 18.6\% of Reddit-v2 test requests (2B, 4B, 9B), and
\method{}'s gain over Rank-GRPO holds both on those requests and on the rest; on Reddit-Books and Reddit-Music the
base models recall 2 to 13\% of test requests and the gain holds on the rest (Appendix~\ref{app:memo}).

\paragraph{Risks.} Recommenders trained on crowd answers can amplify popular items; we report recall by item popularity
(Table~\ref{tab:pop}, Appendix~\ref{app:judge}). The union reward raises recall more on popular than on tail items.

\paragraph{Compute.} Training and evaluation used at least 4,423 B200 GPU-hours (4,268 for training over 156 runs,
156 for 1,927 evaluations), counted from logged step times; staging and idle time are not included.

\bibliography{references}

\begin{thebibliography}{69}
\providecommand{\natexlab}[1]{#1}

\bibitem[{Ahmadian et~al.(2024)Ahmadian, Cremer, Gall{\'e}, Fadaee, Kreutzer,
  Pietquin, {\"U}st{\"u}n, and Hooker}]{ahmadian-etal-2024-back}
Arash Ahmadian, Chris Cremer, Matthias Gall{\'e}, Marzieh Fadaee, Julia
  Kreutzer, Olivier Pietquin, Ahmet {\"U}st{\"u}n, and Sara Hooker. 2024.
\newblock \href {https://doi.org/10.18653/v1/2024.acl-long.662} {Back to
  basics: Revisiting {REINFORCE}-style optimization for learning from human
  feedback in {LLM}s}.
\newblock In \emph{Proceedings of the 62nd Annual Meeting of the Association
  for Computational Linguistics (Volume 1: Long Papers)}, pages 12248--12267.
  Association for Computational Linguistics.

\bibitem[{Ai et~al.(2026)Ai, Pan, Simchi-Levi, and Wang}]{ai2026shapegrpo}
Rui Ai, Yu~Pan, David Simchi-Levi, and Chonghuan Wang. 2026.
\newblock {ShapE-GRPO}: Shapley-enhanced reward allocation for multi-candidate
  {LLM} training.
\newblock \emph{arXiv preprint arXiv:2603.29871}.

\bibitem[{Boratko et~al.(2020)Boratko, Li, O'Gorman, Das, Le, and
  McCallum}]{boratko-etal-2020-protoqa}
Michael Boratko, Xiang~Lorraine Li, Tim O'Gorman, Rajarshi Das, Dan Le, and
  Andrew McCallum. 2020.
\newblock {P}roto{QA}: A question answering dataset for prototypical
  common-sense reasoning.
\newblock In \emph{Proceedings of the 2020 Conference on Empirical Methods in
  Natural Language Processing (EMNLP)}, pages 1122--1136.

\bibitem[{Casella and Robert(1996)}]{casella1996rao}
George Casella and Christian~P. Robert. 1996.
\newblock \href {https://doi.org/10.1093/biomet/83.1.81} {Rao-blackwellisation
  of sampling schemes}.
\newblock \emph{Biometrika}, 83(1):81--94.

\bibitem[{Chu et~al.(2026)Chu, Huang, Zhang, Wei, and Wang}]{chu2025gpg}
Xiangxiang Chu, Hailang Huang, Xiao Zhang, Fei Wei, and Yong Wang. 2026.
\newblock {GPG}: A simple and strong reinforcement learning baseline for model
  reasoning.
\newblock In \emph{International Conference on Learning Representations
  (ICLR)}, pages 59637--59659.

\bibitem[{Ciosek and Whiteson(2018)}]{ciosek2018expected}
Kamil Ciosek and Shimon Whiteson. 2018.
\newblock \href {https://doi.org/10.1609/aaai.v32i1.11607} {Expected policy
  gradients}.
\newblock In \emph{Proceedings of the AAAI Conference on Artificial
  Intelligence}, volume~32.

\bibitem[{Dawid and Skene(1979)}]{dawid1979maximum}
A.~P. Dawid and A.~M. Skene. 1979.
\newblock Maximum likelihood estimation of observer error-rates using the {EM}
  algorithm.
\newblock \emph{Journal of the Royal Statistical Society. Series C (Applied
  Statistics)}, 28(1):20--28.

\bibitem[{Ding(2026)}]{ding2026advantage}
Fei Ding. 2026.
\newblock Advantage scale calibration imbalance in group-relative optimization
  under low-variance rewards: Diagnosis and bounded recovery.
\newblock \emph{arXiv preprint arXiv:2609.19164}.

\bibitem[{Elmansouri et~al.(2026)Elmansouri, Izzati, Seddik, and
  Lahlou}]{elmansouri2025noise}
Omar Elmansouri, Fathinah Izzati, Mohamed El~Amine Seddik, and Salem Lahlou.
  2026.
\newblock \href {https://icml.cc/virtual/2026/poster/61804} {Noise-corrected
  {GRPO}: From noisy rewards to unbiased gradients}.
\newblock In \emph{Forty-third International Conference on Machine Learning}.

\bibitem[{Feng et~al.(2025)Feng, Zhang, Tang, Li, Cai, Chen, Dai, Zhu, and
  Dong}]{feng-etal-2025-expectation}
Xueyang Feng, Jingsen Zhang, Jiakai Tang, Wei Li, Guohao Cai, Xu~Chen, Quanyu
  Dai, Yue Zhu, and Zhenhua Dong. 2025.
\newblock \href {https://doi.org/10.18653/v1/2025.findings-acl.307}
  {Expectation confirmation preference optimization for multi-turn
  conversational recommendation agent}.
\newblock In \emph{Findings of the Association for Computational Linguistics:
  ACL 2025}, pages 5896--5914.

\bibitem[{Grattafiori et~al.(2024)Grattafiori, Dubey, Jauhri, Pandey, Kadian
  et~al.}]{grattafiori2024llama}
Aaron Grattafiori, Abhimanyu Dubey, Abhinav Jauhri, Abhinav Pandey, Abhishek
  Kadian, et~al. 2024.
\newblock The {Llama} 3 herd of models.
\newblock \emph{arXiv preprint arXiv:2407.21783}.

\bibitem[{Greensmith et~al.(2004)Greensmith, Bartlett, and
  Baxter}]{greensmith2004variance}
Evan Greensmith, Peter~L. Bartlett, and Jonathan Baxter. 2004.
\newblock Variance reduction techniques for gradient estimates in reinforcement
  learning.
\newblock \emph{Journal of Machine Learning Research}, 5:1471--1530.

\bibitem[{Hayati et~al.(2020)Hayati, Kang, Zhu, Shi, and
  Yu}]{hayati-etal-2020-inspired}
Shirley~Anugrah Hayati, Dongyeop Kang, Qingxiaoyang Zhu, Weiyan Shi, and Zhou
  Yu. 2020.
\newblock \href {https://doi.org/10.18653/v1/2020.emnlp-main.654} {{INSPIRED}:
  Toward sociable recommendation dialog systems}.
\newblock In \emph{Proceedings of the 2020 Conference on Empirical Methods in
  Natural Language Processing (EMNLP)}, pages 8142--8152.

\bibitem[{He et~al.(2023)He, Xie, Jha, Steck, Liang, Feng, Majumder, Kallus,
  and McAuley}]{he2023large}
Zhankui He, Zhouhang Xie, Rahul Jha, Harald Steck, Dawen Liang, Yesu Feng,
  Bodhisattwa~Prasad Majumder, Nathan Kallus, and Julian McAuley. 2023.
\newblock \href {https://doi.org/10.1145/3583780.3614949} {Large language
  models as zero-shot conversational recommenders}.
\newblock In \emph{Proceedings of the 32nd ACM International Conference on
  Information and Knowledge Management (CIKM)}, pages 720--730.

\bibitem[{Hu et~al.(2025)Hu, Liu, Xu, and Shen}]{hu2025reinforcepp}
Jian Hu, Jason~Klein Liu, Haotian Xu, and Wei Shen. 2025.
\newblock {REINFORCE++}: Stabilizing critic-free policy optimization with
  global normalization.
\newblock \emph{arXiv preprint arXiv:2501.03262}.

\bibitem[{Ichihara et~al.(2026)Ichihara, Jinnai, Ariu, and
  Uchibe}]{ichihara2026consensus}
Yuki Ichihara, Yuu Jinnai, Kaito Ariu, and Eiji Uchibe. 2026.
\newblock Consensus group relative policy optimization for text generation.
\newblock \emph{arXiv preprint arXiv:2602.03102}.

\bibitem[{Khatri et~al.(2026)Khatri, Madaan, Tiwari, Bansal, Duvvuri, Zaheer,
  Dhillon, Brandfonbrener, and Agarwal}]{devvrit2026the}
Devvrit Khatri, Lovish Madaan, Rishabh Tiwari, Rachit Bansal, Venkata Sai
  Surya~Subramanyam Duvvuri, Manzil Zaheer, Inderjit Dhillon, David
  Brandfonbrener, and Rishabh Agarwal. 2026.
\newblock \href {https://openreview.net/forum?id=FMjeC9Msws} {The art of
  scaling reinforcement learning compute for {LLM}s}.
\newblock In \emph{The Fourteenth International Conference on Learning
  Representations}.

\bibitem[{Koehn(2004)}]{koehn-2004-statistical}
Philipp Koehn. 2004.
\newblock \href {https://aclanthology.org/W04-3250/} {Statistical significance
  tests for machine translation evaluation}.
\newblock In \emph{Proceedings of the 2004 Conference on Empirical Methods in
  Natural Language Processing}, pages 388--395, Barcelona, Spain. Association
  for Computational Linguistics.

\bibitem[{Kool et~al.(2019)Kool, van Hoof, and Welling}]{kool2019buy}
Wouter Kool, Herke van Hoof, and Max Welling. 2019.
\newblock \href {https://openreview.net/forum?id=r1lgTGL5DE} {Buy 4 {REINFORCE}
  samples, get a baseline for free!}
\newblock In \emph{Deep Reinforcement Learning Meets Structured Prediction,
  ICLR 2019 Workshop}.

\bibitem[{Kool et~al.(2020)Kool, van Hoof, and Welling}]{kool2020estimating}
Wouter Kool, Herke van Hoof, and Max Welling. 2020.
\newblock Estimating gradients for discrete random variables by sampling
  without replacement.
\newblock In \emph{International Conference on Learning Representations
  (ICLR)}.

\bibitem[{Koolen et~al.(2012)Koolen, Kamps, and Kazai}]{koolen2012socialbook}
Marijn Koolen, Jaap Kamps, and Gabriella Kazai. 2012.
\newblock \href {https://doi.org/10.1145/2396761.2396788} {Social book search:
  Comparing topical relevance judgements and book suggestions for evaluation}.
\newblock In \emph{Proceedings of the 21st ACM International Conference on
  Information and Knowledge Management (CIKM)}, pages 185--194.

\bibitem[{Kwon et~al.(2023)Kwon, Li, Zhuang, Sheng, Zheng, Yu, Gonzalez, Zhang,
  and Stoica}]{kwon2023vllm}
Woosuk Kwon, Zhuohan Li, Siyuan Zhuang, Ying Sheng, Lianmin Zheng, Cody~Hao Yu,
  Joseph~E. Gonzalez, Hao Zhang, and Ion Stoica. 2023.
\newblock \href {https://doi.org/10.1145/3600006.3613165} {Efficient memory
  management for large language model serving with {PagedAttention}}.
\newblock In \emph{Proceedings of the 29th Symposium on Operating Systems
  Principles (SOSP)}, pages 611--626.

\bibitem[{Li et~al.(2024)Li, Zhang, Dong, Deik, Tang, and Liu}]{li2024aligning}
Dexun Li, Cong Zhang, Kuicai Dong, Derrick Goh~Xin Deik, Ruiming Tang, and Yong
  Liu. 2024.
\newblock Aligning crowd feedback via distributional preference reward
  modeling.
\newblock \emph{arXiv preprint arXiv:2402.09764}.

\bibitem[{Li et~al.(2018)Li, Ebrahimi~Kahou, Schulz, Michalski, Charlin, and
  Pal}]{li2018redial}
Raymond Li, Samira Ebrahimi~Kahou, Hannes Schulz, Vincent Michalski, Laurent
  Charlin, and Chris Pal. 2018.
\newblock Towards deep conversational recommendations.
\newblock In \emph{Advances in Neural Information Processing Systems},
  volume~31.

\bibitem[{Lin et~al.(2025)Lin, Wang, and Qian}]{lin2025recr}
Jiacheng Lin, Tian Wang, and Kun Qian. 2025.
\newblock \href {https://openreview.net/forum?id=YBRU9MV2vE} {Rec-r1: Bridging
  generative large language models and user-centric recommendation systems via
  reinforcement learning}.
\newblock \emph{Transactions on Machine Learning Research}.

\bibitem[{Liu et~al.(2025)Liu, Chen, Li, Qi, Pang, Du, Lee, and
  Lin}]{liu2025understanding}
Zichen Liu, Changyu Chen, Wenjun Li, Penghui Qi, Tianyu Pang, Chao Du, Wee~Sun
  Lee, and Min Lin. 2025.
\newblock \href {https://openreview.net/forum?id=5PAF7PAY2Y} {Understanding
  r1-zero-like training: A critical perspective}.
\newblock In \emph{Second Conference on Language Modeling}.

\bibitem[{Liu et~al.(2026)Liu, Liu, He, Wang, Liu, Pan, Hu, Xiong, Huang, Hu,
  Huang, Yang, Wang, Su, and Zheng}]{liu2026tricks}
Zihe Liu, Jiashun Liu, Yancheng He, Weixun Wang, Jiaheng Liu, Ling Pan, Xinyu
  Hu, Shaopan Xiong, Ju~Huang, Jian Hu, Shengyi Huang, Siran Yang, Jiamang
  Wang, Wenbo Su, and Bo~Zheng. 2026.
\newblock \href {https://openreview.net/forum?id=R0JM3BWP7W} {Tricks or traps?
  a deep dive into {RL} for {LLM} reasoning}.
\newblock In \emph{The Fourteenth International Conference on Learning
  Representations}.

\bibitem[{Manning et~al.(2008)Manning, Raghavan, and
  Sch{\"u}tze}]{manning2008introduction}
Christopher~D. Manning, Prabhakar Raghavan, and Hinrich Sch{\"u}tze. 2008.
\newblock \href {https://doi.org/10.1017/CBO9780511809071} {\emph{Introduction
  to Information Retrieval}}.
\newblock Cambridge University Press.

\bibitem[{Menon et~al.(2021)Menon, Rawat, Reddi, Kim, and
  Kumar}]{pmlr-v139-menon21a}
Aditya~K Menon, Ankit~Singh Rawat, Sashank Reddi, Seungyeon Kim, and Sanjiv
  Kumar. 2021.
\newblock A statistical perspective on distillation.
\newblock In \emph{Proceedings of the 38th International Conference on Machine
  Learning}, volume 139 of \emph{Proceedings of Machine Learning Research},
  pages 7632--7642. PMLR.

\bibitem[{{Mistral-AI} et~al.(2025){Mistral-AI}, Rastogi, Jiang, Lo, Berrada,
  Lample, Rute, Barmentlo, Yadav, Khandelwal, Chandu
  et~al.}]{rastogi2025magistral}
{Mistral-AI}, Abhinav Rastogi, Albert~Q. Jiang, Andy Lo, Gabrielle Berrada,
  Guillaume Lample, Jason Rute, Joep Barmentlo, Karmesh Yadav, Kartik
  Khandelwal, Khyathi~Raghavi Chandu, et~al. 2025.
\newblock Magistral.
\newblock \emph{arXiv preprint arXiv:2506.10910}.

\bibitem[{Mostafazadeh~Davani et~al.(2022)Mostafazadeh~Davani, D{\'i}az, and
  Prabhakaran}]{davani2022dealing}
Aida Mostafazadeh~Davani, Mark D{\'i}az, and Vinodkumar Prabhakaran. 2022.
\newblock Dealing with disagreements: Looking beyond the majority vote in
  subjective annotations.
\newblock \emph{Transactions of the Association for Computational Linguistics},
  10:92--110.

\bibitem[{Peterson et~al.(2019)Peterson, Battleday, Griffiths, and
  Russakovsky}]{peterson2019human}
Joshua~C. Peterson, Ruairidh~M. Battleday, Thomas~L. Griffiths, and Olga
  Russakovsky. 2019.
\newblock \href {https://doi.org/10.1109/ICCV.2019.00971} {Human uncertainty
  makes classification more robust}.
\newblock In \emph{Proceedings of the IEEE/CVF International Conference on
  Computer Vision (ICCV)}, pages 9616--9625.

\bibitem[{Plank(2022)}]{plank-2022-problem}
Barbara Plank. 2022.
\newblock \href {https://doi.org/10.18653/v1/2022.emnlp-main.731} {The
  ``problem'' of human label variation: On ground truth in data, modeling and
  evaluation}.
\newblock In \emph{Proceedings of the 2022 Conference on Empirical Methods in
  Natural Language Processing}, pages 10671--10682, Abu Dhabi, United Arab
  Emirates. Association for Computational Linguistics.

\bibitem[{{Qwen Team}(2026)}]{qwen35}
{Qwen Team}. 2026.
\newblock {Qwen3.5}: Towards native multimodal agents.
\newblock \url{https://qwen.ai/blog?id=qwen3.5}.

\bibitem[{Raykar et~al.(2010)Raykar, Yu, Zhao, Valadez, Florin, Bogoni, and
  Moy}]{raykar2010learning}
Vikas~C. Raykar, Shipeng Yu, Linda~H. Zhao, Gerardo~Hermosillo Valadez, Charles
  Florin, Luca Bogoni, and Linda Moy. 2010.
\newblock Learning from crowds.
\newblock \emph{Journal of Machine Learning Research}, 11:1297--1322.

\bibitem[{Razin et~al.(2025)Razin, Wang, Strauss, Wei, Lee, and
  Arora}]{razin2025makes}
Noam Razin, Zixuan Wang, Hubert Strauss, Stanley Wei, Jason~D. Lee, and Sanjeev
  Arora. 2025.
\newblock What makes a reward model a good teacher? an optimization
  perspective.
\newblock In \emph{Advances in Neural Information Processing Systems
  (NeurIPS)}.

\bibitem[{Rennie et~al.(2017)Rennie, Marcheret, Mroueh, Ross, and
  Goel}]{rennie2017self}
Steven~J. Rennie, Etienne Marcheret, Youssef Mroueh, Jerret Ross, and Vaibhava
  Goel. 2017.
\newblock \href {https://doi.org/10.1109/CVPR.2017.131} {Self-critical sequence
  training for image captioning}.
\newblock In \emph{Proceedings of the IEEE Conference on Computer Vision and
  Pattern Recognition (CVPR)}.

\bibitem[{Salmani-Zarchi et~al.(2026)Salmani-Zarchi, Rahimi, Faili, and
  Dousti}]{salmani-zarchi-etal-2026-mdp}
Mohammad~Mahdi Salmani-Zarchi, Zahra Rahimi, Heshaam Faili, and Mohammad~Javad
  Dousti. 2026.
\newblock \href {https://doi.org/10.18653/v1/2026.acl-long.1982} {{MDP}-{GRPO}:
  Stabilized group relative policy optimization for multi-constraint
  instruction following}.
\newblock In \emph{Proceedings of the 64th Annual Meeting of the {A}ssociation
  for {C}omputational {L}inguistics (Volume 1: Long Papers)}, pages
  42784--42797. Association for Computational Linguistics.

\bibitem[{Schulman et~al.(2017)Schulman, Wolski, Dhariwal, Radford, and
  Klimov}]{schulman2017ppo}
John Schulman, Filip Wolski, Prafulla Dhariwal, Alec Radford, and Oleg Klimov.
  2017.
\newblock Proximal policy optimization algorithms.
\newblock \emph{arXiv preprint arXiv:1707.06347}.

\bibitem[{Shao et~al.(2024)Shao, Wang, Zhu, Xu, Song, Bi, Zhang, Zhang, Li, Wu,
  and Guo}]{shao2024deepseekmath}
Zhihong Shao, Peiyi Wang, Qihao Zhu, Runxin Xu, Junxiao Song, Xiao Bi, Haowei
  Zhang, Mingchuan Zhang, Y.~K. Li, Y.~Wu, and Daya Guo. 2024.
\newblock Deepseekmath: Pushing the limits of mathematical reasoning in open
  language models.
\newblock \emph{arXiv preprint arXiv:2402.03300}.

\bibitem[{Siththaranjan et~al.(2024)Siththaranjan, Laidlaw, and
  Hadfield-Menell}]{siththaranjan2024distributional}
Anand Siththaranjan, Cassidy Laidlaw, and Dylan Hadfield-Menell. 2024.
\newblock \href {https://openreview.net/forum?id=0tWTxYYPnW} {Distributional
  preference learning: Understanding and accounting for hidden context in
  {RLHF}}.
\newblock In \emph{International Conference on Learning Representations
  (ICLR)}.

\bibitem[{Surana et~al.(2026)Surana, Namburi, Mundada, Lal, Novack, McAuley,
  and Wu}]{surana2026musicrs}
Rohan Surana, Amit Namburi, Gagan Mundada, Abhay Lal, Zachary Novack, Julian
  McAuley, and Junda Wu. 2026.
\newblock \href {https://doi.org/10.1109/ICASSP55912.2026.11462411} {{MusiCRS}:
  Benchmarking audio-centric conversational recommendation}.
\newblock In \emph{ICASSP 2026 - IEEE International Conference on Acoustics,
  Speech and Signal Processing}, pages 12412--12416.

\bibitem[{Tang et~al.(2025)Tang, Wang, Madaan, and Munos}]{tang2025beyond}
Yunhao Tang, Sid Wang, Lovish Madaan, and R{\'e}mi Munos. 2025.
\newblock \href {https://openreview.net/forum?id=pc6M9h3T9m} {Beyond verifiable
  rewards: Scaling reinforcement learning in language models to unverifiable
  data}.
\newblock In \emph{Advances in Neural Information Processing Systems
  (NeurIPS)}.

\bibitem[{Uma et~al.(2021)Uma, Fornaciari, Hovy, Paun, Plank, and
  Poesio}]{uma2021learning}
Alexandra~N. Uma, Tommaso Fornaciari, Dirk Hovy, Silviu Paun, Barbara Plank,
  and Massimo Poesio. 2021.
\newblock \href {https://doi.org/10.1613/jair.1.12752} {Learning from
  disagreement: A survey}.
\newblock \emph{Journal of Artificial Intelligence Research}, 72:1385--1470.

\bibitem[{Vedantam et~al.(2015)Vedantam, Zitnick, and
  Parikh}]{vedantam2015cider}
Ramakrishna Vedantam, C.~Lawrence Zitnick, and Devi Parikh. 2015.
\newblock \href {https://doi.org/10.1109/CVPR.2015.7299087} {{CIDEr}:
  Consensus-based image description evaluation}.
\newblock In \emph{Proceedings of the IEEE Conference on Computer Vision and
  Pattern Recognition (CVPR)}, pages 4566--4575.

\bibitem[{von Werra et~al.(2020)von Werra, Belkada, Tunstall, Beeching, Thrush,
  Lambert, Huang, Rasul, and Gallou{\'e}dec}]{vonwerra2020trl}
Leandro von Werra, Younes Belkada, Lewis Tunstall, Edward Beeching, Tristan
  Thrush, Nathan Lambert, Shengyi Huang, Kashif Rasul, and Quentin
  Gallou{\'e}dec. 2020.
\newblock {TRL}: Transformers reinforcement learning.
\newblock \url{https://github.com/huggingface/trl}.

\bibitem[{Wan and McAuley(2018)}]{wan2018item}
Mengting Wan and Julian McAuley. 2018.
\newblock \href {https://doi.org/10.1145/3240323.3240369} {Item recommendation
  on monotonic behavior chains}.
\newblock In \emph{Proceedings of the 12th ACM Conference on Recommender
  Systems (RecSys)}, pages 86--94.

\bibitem[{Wan et~al.(2019)Wan, Misra, Nakashole, and
  McAuley}]{wan-etal-2019-fine}
Mengting Wan, Rishabh Misra, Ndapa Nakashole, and Julian McAuley. 2019.
\newblock \href {https://doi.org/10.18653/v1/P19-1248} {Fine-grained spoiler
  detection from large-scale review corpora}.
\newblock In \emph{Proceedings of the 57th Annual Meeting of the Association
  for Computational Linguistics}, pages 2605--2610, Florence, Italy.
  Association for Computational Linguistics.

\bibitem[{Wang et~al.(2026)Wang, Huang, Wang, Ren, and Dong}]{wang2026grpoma}
Hongcheng Wang, Yinuo Huang, Sukai Wang, Guanghui Ren, and Hao Dong. 2026.
\newblock Why tree-style branching matters for thought advantage estimation in
  {GRPO}.
\newblock In \emph{International Conference on Machine Learning (ICML)}.

\bibitem[{Wang et~al.(2020)Wang, Liu, and Li}]{wang2020perturbed}
Jingkang Wang, Yang Liu, and Bo~Li. 2020.
\newblock Reinforcement learning with perturbed rewards.
\newblock In \emph{Proceedings of the AAAI Conference on Artificial
  Intelligence}, volume~34, pages 6202--6209.

\bibitem[{Wang et~al.(2025)Wang, Li, Zang, Zhou, Bu, Wang, Lu, Jin, and
  Wang}]{wang2025prefgrpo}
Yibin Wang, Zhimin Li, Yuhang Zang, Yujie Zhou, Jiazi Bu, Chunyu Wang, Qinglin
  Lu, Cheng Jin, and Jiaqi Wang. 2025.
\newblock Pref-{GRPO}: Pairwise preference reward-based {GRPO} for stable
  text-to-image reinforcement learning.
\newblock \emph{arXiv preprint arXiv:2508.20751}.

\bibitem[{Wen et~al.(2025)Wen, Wang, Yan, Cai, Jiang, and
  Zhang}]{wen2025usbrec}
Jianyu Wen, Jingyun Wang, Cilin Yan, Jiayin Cai, Xiaolong Jiang, and Ying
  Zhang. 2025.
\newblock \href {https://doi.org/10.1145/3705328.3748089} {{USB-Rec}: An
  effective framework for improving conversational recommendation capability of
  large language model}.
\newblock In \emph{Proceedings of the Nineteenth ACM Conference on Recommender
  Systems}, pages 472--481.

\bibitem[{Williams(1992)}]{williams1992simple}
Ronald~J. Williams. 1992.
\newblock \href {https://doi.org/10.1007/BF00992696} {Simple statistical
  gradient-following algorithms for connectionist reinforcement learning}.
\newblock \emph{Machine Learning}, 8(3--4):229--256.

\bibitem[{Wu et~al.(2026)Wu, Wang, Ong, Liang, and Picard}]{wu2026shala}
Jingyao Wu, Ashley Wang, Keane Ong, Paul~Pu Liang, and Rosalind Picard. 2026.
\newblock {SHALA-LLM}: Smartly handling ambiguous labels in aligning {LLMs}.
\newblock \emph{arXiv preprint arXiv:2606.05376}.

\bibitem[{Xiang et~al.(2026)Xiang, Pan, and Wei}]{xiang2026smtpo}
Xingyuan Xiang, Xiangchen Pan, and Wei Wei. 2026.
\newblock User simulator-guided multi-turn preference optimization for
  reasoning {LLM}-based conversational recommendation.
\newblock \emph{arXiv preprint arXiv:2604.03671}.

\bibitem[{Xie et~al.(2025)Xie, Ren, Qi, Hu, and Shan}]{xie2025recllmr1}
Yu~Xie, Xingkai Ren, Ying Qi, Yao Hu, and Lianlei Shan. 2025.
\newblock {RecLLM-R1}: A two-stage training paradigm with reinforcement
  learning and chain-of-thought v1.
\newblock \emph{arXiv preprint arXiv:2506.19235}.

\bibitem[{Xie et~al.(2024)Xie, Wu, Jeon, He, Steck, Jha, Liang, Kallus, and
  McAuley}]{xie2024nbcrs}
Zhouhang Xie, Junda Wu, Hyunsik Jeon, Zhankui He, Harald Steck, Rahul Jha,
  Dawen Liang, Nathan Kallus, and Julian McAuley. 2024.
\newblock \href {https://doi.org/10.1145/3640457.3688191} {Neighborhood-based
  collaborative filtering for conversational recommendation}.
\newblock In \emph{Proceedings of the 18th ACM Conference on Recommender
  Systems (RecSys)}, pages 1045--1050.

\bibitem[{Xu et~al.(2025)Xu, Wang, Lv, and Zheng}]{xu-etal-2025-beyond}
Haozhe Xu, Xiaohua Wang, Changze Lv, and Xiaoqing Zheng. 2025.
\newblock \href {https://doi.org/10.18653/v1/2025.acl-long.758} {Beyond single
  labels: Improving conversational recommendation through {LLM}-powered data
  augmentation}.
\newblock In \emph{Proceedings of the 63rd Annual Meeting of the Association
  for Computational Linguistics (Volume 1: Long Papers)}, pages 15573--15590.

\bibitem[{Yu et~al.(2025)Yu, Zhang, Zhu, Yuan, Zuo, Yue, Dai, Fan, Liu, Liu,
  Liu, Liu, Lin, Lin, Ma, Sheng, Tong, Zhang, Zhang, Zhang, Zhang, Zhu, Zhu,
  Chen, Chen, Wang, Yu, Song, Wei, Zhou, Liu, Ma, Zhang, Yan, Wu, and
  Wang}]{yu2025dapo}
Qiying Yu, Zheng Zhang, Ruofei Zhu, Yufeng Yuan, Xiaochen Zuo, Yu~Yue, Weinan
  Dai, Tiantian Fan, Gaohong Liu, Juncai Liu, LingJun Liu, Xin Liu, Haibin Lin,
  Zhiqi Lin, Bole Ma, Guangming Sheng, Yuxuan Tong, Chi Zhang, Mofan Zhang, and
  17 others. 2025.
\newblock \href {https://openreview.net/forum?id=2a36EMSSTp} {{DAPO}: An
  open-source {LLM} reinforcement learning system at scale}.
\newblock In \emph{The Thirty-ninth Annual Conference on Neural Information
  Processing Systems}.

\bibitem[{Yue et~al.(2026)Yue, Zhuang, Qin, He, Zeng, McAuley, and
  Wang}]{yue2026rar}
Zhenrui Yue, Honglei Zhuang, Zhen Qin, Zhankui He, Huimin Zeng, Julian McAuley,
  and Dong Wang. 2026.
\newblock Retrieval augmented conversational recommendation with reinforcement
  learning.
\newblock \emph{arXiv preprint arXiv:2604.04457}.

\bibitem[{Zhang et~al.(2026)Zhang, Niu, Ying, Lin, Hao, Fan, Huang, Keung,
  Chen, and Lin}]{zhang2026a2search}
Fengji Zhang, Xinyao Niu, Chengyang Ying, Guancheng Lin, Zhongkai Hao, Zhou
  Fan, Chengen Huang, Jacky Keung, Bei Chen, and Junyang Lin. 2026.
\newblock {A$^2$Search}: Ambiguity-aware question answering with reinforcement
  learning.
\newblock In \emph{International Conference on Learning Representations
  (ICLR)}.

\bibitem[{Zheng et~al.(2025{\natexlab{a}})Zheng, Wang, Liu, Wang, Hongyu, Chen,
  Zhao, and Wen}]{zheng2025deeprec}
Bowen Zheng, Xiaolei Wang, Enze Liu, Xi~Wang, Lu~Hongyu, Yu~Chen, Wayne~Xin
  Zhao, and Ji-Rong Wen. 2025{\natexlab{a}}.
\newblock {DeepRec}: Towards a deep dive into the item space with large
  language model based recommendation.
\newblock \emph{arXiv preprint arXiv:2505.16810}.

\bibitem[{Zheng et~al.(2025{\natexlab{b}})Zheng, Liu, Li, Chen, Yu, Gao, Dang,
  Liu, Men, Yang, Zhou, and Lin}]{zheng2025gspo}
Chujie Zheng, Shixuan Liu, Mingze Li, Xiong-Hui Chen, Bowen Yu, Chang Gao, Kai
  Dang, Yuqiong Liu, Rui Men, An~Yang, Jingren Zhou, and Junyang Lin.
  2025{\natexlab{b}}.
\newblock Group sequence policy optimization.
\newblock \emph{arXiv preprint arXiv:2507.18071}.

\bibitem[{Zheng et~al.(2018)Zheng, Ma, and Huang}]{zheng2018multi}
Renjie Zheng, Mingbo Ma, and Liang Huang. 2018.
\newblock Multi-reference training with pseudo-references for neural
  translation and text generation.
\newblock In \emph{Proceedings of the 2018 Conference on Empirical Methods in
  Natural Language Processing}, pages 3188--3197.

\bibitem[{Zhou et~al.(2026)Zhou, Liu, Sims, Wang, Pang, Li, Wang, Lin, and
  Du}]{zhou2026verifree}
Xiangxin Zhou, Zichen Liu, Anya Sims, Haonan Wang, Tianyu Pang, Chongxuan Li,
  Liang Wang, Min Lin, and Chao Du. 2026.
\newblock \href {https://openreview.net/forum?id=nnwvwge40d} {Reinforcing
  general reasoning without verifiers}.
\newblock In \emph{International Conference on Learning Representations
  (ICLR)}.

\bibitem[{Zhu et~al.(2026{\natexlab{a}})Zhu, Steck, Liang, He, Ostuni, Li, and
  Kallus}]{zhu2026rankgrpo}
Yaochen Zhu, Harald Steck, Dawen Liang, Yinhan He, Vito~Claudio Ostuni, Jundong
  Li, and Nathan Kallus. 2026{\natexlab{a}}.
\newblock \href {https://openreview.net/forum?id=Xgw2D9cALS} {Rank-{GRPO}:
  Training {LLM}-based conversational recommender systems with reinforcement
  learning}.
\newblock In \emph{The Fourteenth International Conference on Learning
  Representations}.

\bibitem[{Zhu et~al.(2026{\natexlab{b}})Zhu, Steck, McInerney, Sinha, He,
  Kallus, and Li}]{zhu2026multdpo}
Yaochen Zhu, Harald Steck, James McInerney, Aditya Sinha, Yinhan He, Nathan
  Kallus, and Jundong Li. 2026{\natexlab{b}}.
\newblock Mult-{DPO}: Multinomial direct preference optimization for
  recommender systems.
\newblock \emph{arXiv preprint arXiv:2606.10078}.

\bibitem[{Zhu et~al.(2025)Zhu, Wan, Steck, Liang, Feng, Kallus, and
  Li}]{zhu2025crag}
Yaochen Zhu, Chao Wan, Harald Steck, Dawen Liang, Yesu Feng, Nathan Kallus, and
  Jundong Li. 2025.
\newblock \href {https://doi.org/10.1145/3696410.3714908} {Collaborative
  retrieval for large language model-based conversational recommender systems}.
\newblock In \emph{Proceedings of the ACM Web Conference 2025}, pages
  3323--3334.

\bibitem[{Zuo et~al.(2025)Zuo, Zhang, Sheng, Qu, Cui, Zhu, Li, Zhang, Long,
  Hua, Qi, Sun, Ma, Yuan, Ding, and Zhou}]{zuo2025ttrl}
Yuxin Zuo, Kaiyan Zhang, Li~Sheng, Shang Qu, Ganqu Cui, Xuekai Zhu, Haozhan Li,
  Yuchen Zhang, Xinwei Long, Ermo Hua, Biqing Qi, Youbang Sun, Zhiyuan Ma,
  Lifan Yuan, Ning Ding, and Bowen Zhou. 2025.
\newblock {TTRL}: Test-time reinforcement learning.
\newblock In \emph{Advances in Neural Information Processing Systems}.

\end{thebibliography}

\appendix
\renewcommand{\topfraction}{0.9}\renewcommand{\dbltopfraction}{0.9}\renewcommand{\bottomfraction}{0.8}
\renewcommand{\textfraction}{0.07}\renewcommand{\floatpagefraction}{0.85}\renewcommand{\dblfloatpagefraction}{0.85}
\setcounter{topnumber}{4}\setcounter{bottomnumber}{3}\setcounter{totalnumber}{6}\setcounter{dbltopnumber}{3}
\newcommand{\apptocsec}[2]{\par\noindent\makebox[1.4em][l]{\ref{#1}}#2\dotfill\pageref{#1}}
\newcommand{\apptocsub}[2]{\par\noindent\hspace*{1.4em}\makebox[2.2em][l]{\ref{#1}}#2\dotfill\pageref{#1}}

\section*{About this appendix}
Appendix~\ref{app:proofs} proves the results of \S\ref{sec:method} (Lemma~\ref{lem:linear} and
Propositions~\ref{prop:rb}--\ref{prop:agg}), and Appendix~\ref{app:dpo} describes the offline baselines that train
on pooled positives. Appendix~\ref{app:extra} holds the algorithm and the additional results behind \S\ref{sec:exp}:
robustness to seeds, resampling, cutoffs and learning rates; compute and longer training; stronger baselines; where the
gain comes from; judges and popularity; and the full result tables. Appendices~\ref{app:p0}--\ref{app:music} specify
the length control, the hyperparameters and the construction of Reddit-Books and Reddit-Music, and
Appendices~\ref{app:repro}--\ref{app:protoqa} give reproducibility notes, the memorization probe and the ProtoQA
experiment that carries the reward beyond recommendation.

\medskip
{\small
\apptocsec{app:proofs}{Proofs}
\apptocsec{app:dpo}{Offline baselines on pooled positives}
\apptocsec{app:extra}{Algorithm and additional results}
\apptocsub{app:alg}{Algorithm}
\apptocsub{app:robust}{Robustness}
\apptocsub{app:compute}{Compute and longer training}
\apptocsub{app:baselines}{Stronger baselines}
\apptocsub{app:mech}{Where the gain comes from}
\apptocsub{app:judge}{Judges and popularity}
\apptocsub{app:full}{Full results}
\apptocsec{app:p0}{Length control}
\apptocsec{app:hparams}{Hyperparameters}
\apptocsec{app:books}{Reddit-Books}
\apptocsec{app:music}{Reddit-Music}
\apptocsec{app:repro}{Reproducibility notes}
\apptocsec{app:memo}{Memorization}
\apptocsec{app:protoqa}{ProtoQA}
\par}

\FloatBarrier
\section{Proofs}
\label{app:proofs}

\begin{table}[htbp]
\centering\small
\setlength{\tabcolsep}{4.5pt}
\begin{adjustbox}{max width=\columnwidth}
\begin{tabular}{@{}l ccccc l@{}}
\toprule
list $L$ (top 5) & A & B & C & D & E & items named \\
\midrule
row 1 (responder 1) & 1 & 0 & 0 & 0 & 0 & \{A\} \\
row 2 (responder 2) & 0 & 1 & 0 & 0 & 0 & \{B, F\} \\
row 3 (responder 3) & 0 & 0 & 0 & 0 & 0 & \{G\} \\
row 4 (responder 4) & 1 & 0 & 1 & 0 & 0 & \{A, C\} \\
row 5 (responder 5) & 0 & 0 & 0 & 0 & 0 & \{H\} \\
\midrule
sum & 2 & 1 & 1 & 0 & 0 & per-row objective \\
bounded sum (union) & 1 & 1 & 1 & 0 & 0 & $\approx$ per-row \\
mean & 0.4 & 0.2 & 0.2 & 0 & 0 & per-request \\
\bottomrule
\end{tabular}
\end{adjustbox}
\tabnote{One row per responder under Rank-GRPO, one reward per request after marginalizing; targets follow Proposition~\ref{prop:agg}.}
\caption{Per-position rewards for a toy request with five responders.}
\label{tab:example}
\end{table}

\paragraph{Setup.} For a request $x$ the policy draws $L^1,\dots,L^G$ and position $k$ of list $g$ receives $R^g_k$.
The estimator is $A^g_k=(R^g_k-\bar R_k)/\sigma_k$ with $\bar R_k=\frac1G\sum_hR^h_k$ and
$\hat g=\frac1G\sum_{g,k}A^g_k\,s^g_k$, where $s^g_k$ is the (reward-independent, loss-normalization-weighted) sum of
token score functions $\nabla_\theta\log\pi_\theta$ over the tokens of line $k$ (Table~\ref{tab:example} shows the
rewards of one list under each aggregation). This is the on-policy gradient of the
Rank-GRPO surrogate (importance ratio 1, before clipping); the KL term does not depend on the reward and is omitted.
The length-shaping terms do not depend on the item rewards under Rank-GRPO's fixed constants; under the length control of
Appendix~\ref{app:p0} their scale is a batch statistic of the rewards, which, like $\sigma_k$ under batch scaling,
moves by $O(1/B)$ with one group. $\sigma_k$ is $1$ (no scaling), $\mathrm{std}_g(R^g_k)+\epsilon$ (group), or the
standard deviation of position $k$'s rewards over the batch plus $\epsilon$ (batch); $\epsilon=10^{-4}$, Bessel-corrected
standard deviations.

\begin{lemma}[linearity]
\label{lem:linear}
If $\sigma_k$ does not depend on this group's rewards, then $\hat g=\sum_{g,k}R^g_k\,u^g_k$ with
$u^g_k=(s^g_k-\bar s_k)/(G\sigma_k)$ and $\bar s_k=\frac1G\sum_hs^h_k$. If moreover $R^g_k=e^g_k\,Y(L^g_k)$ for an
item-level reward $Y$, then $\hat g(Y)=\sum_iY(i)\,U_i$ with $U_i=\sum_{(g,k):L^g_k=i}e^g_k\,u^g_k$.
\end{lemma}
\begin{proof}
$\sum_g(R_g-\bar R)s_g=\sum_gR_gs_g-\bar R\sum_gs_g=\sum_gR_g(s_g-\bar s)$ at each position; then group the terms by
item.
\end{proof}

With $\sigma_k=1$, $\E[\hat g]=\frac{G-1}{G}h(\theta)$ with $h(\theta)=\E\sum_kR_k\nabla\log\pi_\theta(L_k\mid
x,L_{<k})$; the factor arises because the group mean includes the sample itself. $h$ is Rank-GRPO's immediate-credit
surrogate: it drops the terms $\E[R_j\nabla\log\pi_\theta(L_k\mid\cdot)]$, $j>k$, of the gradient of the expected
number of hits. The statements below compare rewards inside this fixed estimator and do not depend on this choice.

\paragraph{Proposition~\ref{prop:rb}, full statement.} Fix a request $x$ with $n$ responders whose rows share the
prompt and the mentioned-title set, and let $r\sim\mathrm{Unif}\{1,\dots,n\}$ independently of the lists. Let $\sigma$
not depend on this group's rewards. Then
\begin{enumerate}
  \item $\E_r[\gS\mid x,L^{1:G}]=\gM$; hence $\E[\gS]=\E[\gM]$ whenever both draw requests from the same distribution,
  for any policy.
  \item $\Cov(\gS)=\Cov(\gM)+\E\big[\Cov(\gS\mid x,L^{1:G})\big]$ with
  $\Cov(\gS\mid x,L^{1:G})=\sum_{i,i'}\hat C_{ii'}\,U_iU_{i'}^{\top}$ and
  $\hat C_{ii'}=\frac1n\sum_r\1\{i\in S_r\}\1\{i'\in S_r\}-\phat_i\phat_{i'}$. In trace form,
  $\tr\Var(\gS)-\tr\Var(\gM)=\E\|\sum_i(\1\{i\in S_r\}-\phat_i)U_i\|^2\ge0$. Ignoring cross-item covariances this is
  $\E\sum_i\phat_i(1-\phat_i)\|U_i\|^2$, the within-request reward variance weighted by the squared score
  coefficients; it vanishes if every generated item is named by all or none of the responders, in particular when
  $n=1$.
  \item For a batch of independently drawn requests both statements hold for the summed update, and the variance gaps
  add. Multiplying the reward by the constant $n$ (the sum) multiplies item~1 by $n$ and item~2 by $n^2$.
\end{enumerate}
\begin{proof}
By Lemma~\ref{lem:linear}, $\gS=\sum_i\1\{i\in S_r\}U_i$ with $U_i$ a function of $(x,L^{1:G})$ only, and
$\E_r\1\{i\in S_r\}=\phat_i$, which gives item~1. Item~2 is the law of total covariance applied to a linear map of the
indicator vector $(\1\{i\in S_r\})_i$, whose conditional covariance is $\hat C$. Item~3 follows from independence
across requests and linearity.
\end{proof}

\begin{proposition}[aggregation and objective]
\label{prop:agg}
With a shared fixed $\sigma$, the mean with requests drawn in proportion to $n_x$ has Rank-GRPO's expected update
(that of $\Jrow$) at no larger variance, the sum with requests drawn uniformly has a constant multiple of it, and the
mean with requests drawn uniformly has the expected update of $\Jreq$.
\end{proposition}

\paragraph{Proposition~\ref{prop:agg}, full statement.} The per-row and per-request objectives are
\begin{align*}
\Jrow&=\tfrac{1}{N_R}\textstyle\sum_j\sum_{r}\E_L\sum_k q_k(L;S_{jr}),\\
\Jreq&=\tfrac{1}{N_T}\textstyle\sum_j\tfrac{1}{n_j}\sum_{r}\E_L\sum_k q_k(L;S_{jr}),
\end{align*}
with $N_T$ requests and $N_R=\sum_jn_j$ rows (per-row recall divides hits by $|S_r|$; a weight $1/|S_r|$ per responder
carries through every statement, and like Rank-GRPO we use weight 1). Write $\bar g_j(Y)$ for the expected update on
request $j$ with item reward $Y$. With $\sigma$ fixed: (i) Rank-GRPO's estimator (rows drawn uniformly) has expectation
$\sum_j\frac{n_j}{N_R}\bar g_j(\phat_j)$, the update for $\Jrow$. (ii) The mean with requests drawn in proportion to
$n_j$ has the same expectation and is Rank-GRPO's estimator conditioned on the request, so at a shared per-position
scale its variance is never larger. (iii) The sum with requests drawn uniformly has expectation $\frac{N_R}{N_T}$ times
(i), the same direction; the request weight $n_j$ moves from the sampling distribution into the reward, which can raise
the variance across requests. (iv) The mean with requests drawn uniformly has expectation $\frac1{N_T}\sum_j\bar
g_j(\phat_j)$, the update for $\Jreq$, which gives single-responder requests 41\% of the weight instead of their 7.7\% in
$\Jrow$. The bounded sum equals the sum on items named by one responder (89\% of union items, 76\% of the sum's mass);
it removes the responder draw but is not itself a Rao--Blackwellization, since it undercounts items that several
responders name.

\paragraph{Proof of Proposition~\ref{prop:agg}.} Let $\bar g_j(Y)=\E[\hat g(Y)\mid x_j]$, linear in $Y$ by
Lemma~\ref{lem:linear}. (i) Rank-GRPO draws request $j$ with probability $n_j/N_R$ and then a responder uniformly; by
Proposition~\ref{prop:rb}.1 the conditional expectation given $j$ is $\bar g_j(\phat_j)$, so the expectation is
$\sum_j\frac{n_j}{N_R}\bar g_j(\phat_j)$. Since $\Jrow=\frac1{N_R}\sum_jn_j\,\E_L\sum_ke_k\phat_j(L_k)$, this is the
update for $\Jrow$ within the surrogate. (ii) Drawing $j$ with probability $n_j/N_R$ and using $\phat_j$ is the
conditional expectation of (i) given $j$; at a shared scale the law of total variance gives the variance ordering. (iii) With $j$ uniform
and reward $c_j=n_j\phat_j$, the expectation is $\frac1{N_T}\sum_jn_j\bar g_j(\phat_j)=\frac{N_R}{N_T}$ times (i). (iv)
With $j$ uniform and reward $\phat_j$ the expectation is $\frac1{N_T}\sum_j\bar g_j(\phat_j)$, and
$\Jreq=\frac1{N_T}\sum_j\E_L\sum_ke_k\phat_j(L_k)$. For the bounded sum $b_j(i)=\min(c_{ji},1)$,
$c_j-b_j=\sum_i(c_{ji}-1)_+\,\1_i$ is supported on items named by at least two responders, so the expected update
differs from (iii) only through those items.

\paragraph{Zero-signal groups.} A group contributes nothing at position $k$ iff its $G$ rewards are equal. Under any of
the three aggregations all rewards are zero iff no list hits any item of the union of responders' sets, so
$P(\text{all zero}\mid\text{aggregated})=P(\cap_r\{\text{all zero for }S_r\})\le\min_rP(\text{all zero for }S_r)$.

\paragraph{What marginalizing does not remove.} Marginalizing removes the per-step responder noise, not the sampling
error of $\phat$ as an estimate of what other potential responders would say; with a fixed dataset every reward shares
that error.

\paragraph{Batch scaling.} With batch scaling $\sigma_k$ is a statistic of $B=384$ rewards, one group moves it by
$O(1/B)$, and treating it as fixed is accurate, so Propositions~\ref{prop:rb} and \ref{prop:agg} hold per position. On
the same lists, $\Var_{\mathrm{batch}}(R^{\mathrm S}_k)=\Var_{\mathrm{batch}}(R^{\mathrm M}_k)+\E\,\Var(R^{\mathrm
S}_k\mid x,L)\ge\Var_{\mathrm{batch}}(R^{\mathrm M}_k)$. Each run divides by the batch standard deviation of its own
raw rewards, so the per-responder advantages are divided by the larger scale, and a fraction
$\E\Var(R^{\mathrm S}_k\mid x,L)/\Var(R^{\mathrm S}_k)$ of the per-responder reward's batch variance is responder noise.
The two runs are therefore compared at different advantage scales; Adam absorbs a global rescaling, but the balance
between the policy-gradient and KL terms differs. At the SFT policy the per-position batch standard deviation is 0.130
for the per-responder hit and 0.059 for the mean under row sampling (0.124 and 0.088 under request sampling; 0.370 and
0.249 for the bounded sum), so the batch-scaled mean takes larger effective steps; Proposition~\ref{prop:rb} compares
estimators at a common scale, not these runs. With the exact sum the batch variance is dominated by the few requests
with large $n$ (1.251 under row sampling), which is why we cap it.

\paragraph{Group scaling.} With group scaling $\sigma_k$ depends on the group's rewards, the map is nonlinear, and
item~1 of Proposition~\ref{prop:rb} fails as an identity. For 0/1 rewards the per-group factor
$\rho=(\sigma_B+\epsilon)/(\sigma_g+\epsilon)$ of every non-degenerate group lies within a factor 1.51 of a constant at
each position (Proposition~\ref{prop:norm}.3), so the group-scaled update is a reweighting of the batch-scaled one by
group weights in a fixed range, but for the per-responder hit the weight depends on the responder drawn, so $\E_r$
does not pass through it. This is
the sense in which \method{}, which uses per-group scaling, is motivated by Propositions~\ref{prop:rb}
and~\ref{prop:agg} rather than covered by them; the arms that test the propositions use batch scaling. The bounded sum
and the mean differ in more than scale only where some list hits an item named by two or more responders: 5.8\% of
group-positions at the SFT policy (23.7\% of those with a hit), rising to 26.3\% for requests with ten or more
responders.

\paragraph{Proposition~\ref{prop:norm}, full statement.} Fix a position, a group with rewards $a_1,\dots,a_G$, mean
$\bar a$ and Bessel standard deviation $\sigma_g$, and let $\sigma_B$ be the batch standard deviation at that position. Let
$\hat A^{\mathrm{grp}}_g=(a_g-\bar a)/(\sigma_g+\epsilon)$ and $\hat A^{\mathrm{bat}}_g=(a_g-\bar a)/(\sigma_B+\epsilon)$.
\begin{enumerate}
  \item For $\epsilon=0$ and $\sigma_g>0$, $\hat A^{\mathrm{grp}}$ is invariant to $a\mapsto\lambda a+\beta$, $\lambda>0$,
  and $\sum_g(\hat A^{\mathrm{grp}}_g)^2=G-1$ for every non-degenerate group; with $\epsilon>0$ the sum is
  $(G-1)\sigma_g^2/(\sigma_g+\epsilon)^2$.
  \item $\hat A^{\mathrm{grp}}=\rho\,\hat A^{\mathrm{bat}}$ with $\rho=(\sigma_B+\epsilon)/(\sigma_g+\epsilon)$. Under batch
  scaling a group's advantage mass is $(G-1)\sigma_g^2/(\sigma_B+\epsilon)^2$, proportional to its share of the reward
  variance; under group scaling it is $G-1$ regardless.
  \item If $a_g\in\{0,1\}$ with $m\in\{1,\dots,G-1\}$ hits, $\sigma_g^2=m(G-m)/(G(G-1))$, so
  $\sigma_g\in[1/\sqrt G,\sqrt{G/(4(G-1))}]$ for even $G$. For $G=8$, $\sigma_g\in[0.354,0.535]$ and $\max\rho/\min\rho\le
  G/(2\sqrt{G-1})=1.51$ within a position.
  \item If rewards lie on a grid of spacing $\delta$, the group with $G-1$ equal rewards and one reward $\delta$ higher
  has $\sigma_g=\delta/\sqrt G$ and receives the advantages $\big(\frac{G-1}{\sqrt G},-\frac1{\sqrt G},\dots\big)=(2.47,-0.35,
  \dots)$ of a full hit against misses. Relative to the smallest binary spread the amplification is $1/\delta$, which is
  $n$ for the mean.
  \item If the reward differences inside a group are of the order of the estimation noise of the weights
  ($\sqrt{p(1-p)/n}$ for $\phat$), group-normalized advantages there have magnitude $O(1)$ and a sign set by which
  responders happened to be observed; batch-normalized ones have magnitude $O(\sigma_g/\sigma_B)$.
\end{enumerate}
\begin{proof}
Items 1--2 are immediate. For 3, $\sum_g(a_g-\bar a)^2=m(G-m)/G$; the extremes are $m=1$ and $m=G/2$, and
$\sqrt{G/(4(G-1))}\cdot\sqrt G=G/(2\sqrt{G-1})$. For 4, $\sum_g(a_g-\bar a)^2=\delta^2(G-1)/G$, so $\sigma_g=\delta/\sqrt G$
and the advantages equal those of $m=1$ in item~3. Item~5 follows from item~2.
\end{proof}

\FloatBarrier
\section{Offline baselines on pooled positives}
\label{app:dpo}
\paragraph{SFT on pooled positives.} One example per training request instead of one per row: the target lists the
union of the responders' items, ordered by how many responders named each (ties by training popularity, then title),
truncated or filled to 20 from the request's own row targets, with mentioned titles excluded (21,023 requests; 12\% have
more than 20 pooled items). Recipe and selection are those of our SFT (Appendix~\ref{app:hparams}); at 4B, where validation
still rose at step 1,500, it runs 3,000 steps and is selected among the six checkpoints. On test it scores 15.94 at 2B and 17.25 at 4B, $+0.53$ $[-0.19,+1.27]$ and $+0.73$ $[+0.04,+1.37]$ above
our SFT and far below both RL arms (Table~\ref{tab:checks}).

\paragraph{DPO on pooled positives.} The DPO baseline prefers the best of eight sampled SFT lists to the worst, both
scored against all responders, so both lists are the policy's own (42,204 pairs at 2B; requests whose lists all tie give
none); its learning rate and epochs were chosen among three configurations on 2B validation, all peaking at 19.2 to
19.3; the 4B run follows the same recipe, selected on 4B validation. It improves on SFT
($+1.64$ $[+1.11,+2.15]$ at 4B) but stays 5.5 and 6.9 points below \method{}. Taking the union of the responders' items
as the preferred list instead scores 13.69 on test, $-1.71$ $[-2.27,-1.15]$ below SFT: the preferred lists are text the
policy would rarely write.

\FloatBarrier
\section{Algorithm and additional results}
\label{app:extra}

\subsection{Algorithm}\label{app:alg}

\begin{algorithm}[h]
\small
\hrule\vspace{3pt}
\textbf{Algorithm 1:} one training step. \method{} uses $a_x(i)=\min(c_i,1)$, uniform request sampling and
per-group $\sigma_k$; the ablation arms vary these three choices.\vspace{2pt}
\hrule\vspace{3pt}
\textbf{Input:} policy $\pi_\theta$; requests $x$ with responder counts $c_i(x)$ and $n_x$; aggregation
$a_x(i)\in\{\min(c_i,1),\ c_i/n_x\}$; $B'=48$ requests per step; $G=8$.
\begin{enumerate}[leftmargin=1.4em]
  \item Draw $B'$ requests, uniformly or in proportion to $n_x$.
  \item For each request sample $G$ lists at temperature 1; number the non-blank lines (a line that does not parse is a miss).
  \item Reward every position: $q^g_k=e^g_k\,a_x(L^g_k)$.
  \item Advantage: $A^g_k=(q^g_k-\bar q_k(x))/(\sigma_k+10^{-4})$, with $\sigma_k$ the standard deviation of position
  $k$'s rewards within the group of $G$ lists (\method{}, Rank-GRPO), or over all $B'G$ lists (batch scaling: the batch-scaled
  arms of Tables~\ref{tab:main} and~\ref{tab:ablation} and Figure~\ref{fig:noise}b).
  \item Put $A^g_k$ on the tokens of the $k$-th item line; add length shaping in reward units (App.~\ref{app:p0}).
  \item Update $\theta$ with the clipped surrogate and the KL penalty.
\end{enumerate}
\vspace{2pt}\hrule
\end{algorithm}

\FloatBarrier

\subsection{Robustness}\label{app:robust}

Figure~\ref{fig:main} shows the gain of every setting and seed. Table~\ref{tab:robust} recomputes it with whole
submissions as the resampling unit, without test requests whose submission also appears in training, with other
metrics and seeds, and with batch instead of per-group scaling, and Table~\ref{tab:topk} at other cutoffs.

\begin{table*}[t]
\centering\small
\setlength{\tabcolsep}{8pt}
\renewcommand{\arraystretch}{1.08}
\begin{tabular}{@{}l cccc@{}}
\toprule
 & \multicolumn{3}{c}{\textbf{Reddit-v2}} & \textbf{Reddit-Books} \\
\cmidrule(lr){2-4}\cmidrule(l){5-5}
 & 2B & 4B & 9B & 2B \\
\midrule
\multicolumn{5}{@{}l}{\grp{Resampling unit}} \\
\quad requests (primary) & \dstack{\up{+1.57}}{[1.10,2.06]} & \dstack{\up{+1.54}}{[1.10,2.00]} & \dstack{\up{+1.46}}{[1.02,1.93]} & \dstack{\up{+0.81}}{[0.19,1.46]} \\
\quad Reddit submissions & \dstack{\up{+1.57}}{[1.09,2.07]} & \dstack{\up{+1.54}}{[1.08,2.00]} & \dstack{\up{+1.46}}{[1.01,1.92]} & -- \\
\quad without the 42 train-overlap requests & \dstack{\up{+1.56}}{[1.10,2.03]} & \dstack{\up{+1.55}}{[1.09,2.05]} & \dstack{\up{+1.47}}{[1.04,1.92]} & -- \\
\addlinespace[3pt]
\multicolumn{5}{@{}l}{\grp{Metric}} \\
\quad R@20 per request & \dstack{\up{+2.29}}{[1.43,3.17]} & \dstack{\up{+2.58}}{[1.82,3.41]} & \dstack{\up{+2.40}}{[1.64,3.22]} & \dstack{\ns{+0.70}}{[$-$0.15,1.52]} \\
\quad N@20 per row & \dstack{\up{+0.76}}{[0.53,0.99]} & \dstack{\up{+0.82}}{[0.61,1.04]} & \dstack{\up{+0.84}}{[0.62,1.07]} & \dstack{\up{+0.65}}{[0.36,0.96]} \\
\addlinespace[3pt]
\multicolumn{5}{@{}l}{\grp{Seeds}} \\
\quad seed 42 pair & \dstack{\up{+1.39}}{[0.94,1.83]} & \dstack{\up{+1.91}}{[1.47,2.32]} & \dstack{\up{+1.71}}{[1.29,2.14]} & \dstack{\up{+1.31}}{[0.70,1.95]} \\
\quad seed 1234 pair & \dstack{\up{+1.40}}{[0.94,1.86]} & \dstack{\up{+1.50}}{[1.04,1.97]} & \dstack{\up{+2.13}}{[1.65,2.62]} & \dstack{\ns{+0.54}}{[$-$0.09,1.19]} \\
\quad three seeds each (Welch) & \dstack{\up{+1.45}}{[1.32,1.58]} & \dstack{\up{+1.65}}{[1.12,2.19]} & \dstack{\up{+1.77}}{[0.75,2.78]} & \dstack{\up{+0.89}}{[0.39,1.38]} \\
\addlinespace[3pt]
\multicolumn{5}{@{}l}{\grp{Advantage scale}} \\
\quad group $-$ batch, bounded sum & \dstack{\up{+0.69}}{[0.32,1.06]} & \dstack{\ns{+0.09}}{[$-$0.27,0.42]} & \dstack{\ns{$-$0.01}}{[$-$0.37,0.34]} & \dstack{\ns{$-$0.43}}{[$-$0.95,0.09]} \\
\bottomrule
\end{tabular}

\tabnote{Points; per-row R@20 unless noted. Train-overlap requests: test requests whose submission also appears in training. --: not run or not applicable.}
\caption{Robustness of the gain over the better per-responder baseline.}
\label{tab:robust}
\end{table*}

\begin{table}[htbp]
\centering\small
\setlength{\tabcolsep}{4pt}
\renewcommand{\arraystretch}{1.08}
\begin{tabular}{@{}l cccc@{}}
\toprule
 & \multicolumn{3}{c}{\textbf{Reddit-v2}} & \textbf{Reddit-Books} \\
\cmidrule(lr){2-4}\cmidrule(l){5-5}
 & 2B & 4B & 9B & 2B \\
\midrule
R@5 & \dstack{\up{+0.73}}{[0.36,1.10]} & \dstack{\up{+0.65}}{[0.24,1.06]} & \dstack{\up{+0.75}}{[0.36,1.14]} & \dstack{\up{+1.13}}{[0.62,1.67]} \\
R@10 & \dstack{\up{+1.22}}{[0.80,1.63]} & \dstack{\up{+1.11}}{[0.69,1.51]} & \dstack{\up{+1.13}}{[0.71,1.56]} & \dstack{\up{+0.78}}{[0.25,1.35]} \\
R@20 & \dstack{\up{+1.57}}{[1.10,2.06]} & \dstack{\up{+1.54}}{[1.10,2.00]} & \dstack{\up{+1.46}}{[1.02,1.93]} & \dstack{\up{+0.81}}{[0.19,1.46]} \\
N@20 & \dstack{\up{+0.76}}{[0.53,0.99]} & \dstack{\up{+0.82}}{[0.61,1.04]} & \dstack{\up{+0.84}}{[0.62,1.07]} & \dstack{\up{+0.65}}{[0.36,0.96]} \\
Hit@20 & \dstack{\up{+1.96}}{[1.37,2.56]} & \dstack{\up{+1.58}}{[1.06,2.07]} & \dstack{\up{+1.47}}{[0.96,2.03]} & \dstack{\up{+0.85}}{[0.11,1.58]} \\
\bottomrule
\end{tabular}

\tabnote{Hit@20: share of rows with at least one responder item in the top 20.}
\caption{\method{} minus the better per-responder baseline at other cutoffs (test, per row).}
\label{tab:topk}
\end{table}

\paragraph{Other seeds, KL and single-responder requests.} With the other seeds, \method{} needs 0.50 and 0.52 of
the baseline's steps at 2B (seeds 42 and 1234), 0.31 and 0.33 at 4B, 0.29 at 9B (seed 1234), 0.45 and 0.41 on
Reddit-Books 2B, 0.39 on Reddit-Books 4B and 0.45 on Llama-3.2-3B (seed 42, against Rank-GRPO): every seed lies within
the 25 to 53\% of Table~\ref{tab:main}, and \method{} is ahead (CI above zero) at 71 of these 72 validated checkpoints.
When \method{} reaches Rank-GRPO's step-2,000 validation R@20, its KL to the reference policy is 0.34, 0.31, 0.29,
0.47 and 0.42 (2B, 4B, 9B, Llama-3.2-3B, Llama-3.1-8B), against Rank-GRPO's 0.29, 0.26, 0.27, 0.39 and 0.47 at step
2,000: slightly higher in four settings (by 0.02 to 0.08), lower in one. Under per-group scaling, requests with one
responder make up 41\% of \method{}'s groups but carry 9.6\% $[9.4,9.9]$ of its squared advantage at the start of
training (Rank-GRPO: 7.7\% of groups, 5.2\% of the squared advantage), because their groups have fewer hits; the
per-group scale does not shift the update toward these requests.

\paragraph{Later test set.}\label{app:late} After model selection we built
a new test set from r/suggestmeabook threads created between 2017-10 and 2017-12 with the unchanged Books pipeline
(3,795 requests, 21,369 rows; no submission shared with the Books splits; grounding 96.6\%) and
evaluated each selected checkpoint once. For the 4B seed-3407 pair, \method{} minus Rank-GRPO is
$+0.67$ $[+0.34,+1.02]$ per row ($+0.46$ $[-0.04,+0.95]$ per request; 29.41 against 28.74; popularity 6.21). Against the per-responder batch baseline, \method{} gains $+0.60$
$[+0.27,+0.93]$ (4B, seed 3407), $+1.15$ $[+0.80,+1.50]$ (4B, seed 42), $+0.36$ $[+0.04,+0.70]$ and $+0.69$
$[+0.37,+1.02]$ (2B, seeds 42 and 1234). The 2B pair trained to a plateau is level ($+0.25$ $[-0.08,+0.57]$). Re-voting
the catalog's title forms with the new threads changed 557 of 74,474 titles; restricted to items whose form did not
change, no sign or CI changes (primary $+0.71$ $[+0.34,+1.08]$).

\paragraph{Cosine learning-rate schedule.} With the learning rate decayed from $10^{-6}$ to 0 over 4,000 steps (2B,
seeds 3407 and 42, otherwise as Table~\ref{tab:main}), Rank-GRPO's validation R@20 peaks at 24.26 (step 3,000) and
24.12 (step 4,000); \method{} reaches these scores after about 895 and 963 steps (30\% and 24\%) and Rank-GRPO's
step-2,000 scores after about 676 and 736. Tested at their best steps, \method{} leads by $+1.08$ $[+0.66,+1.48]$ and
$+1.12$ $[+0.74,+1.49]$ per row ($+2.30$ and $+1.32$ per request); over the two seeds $+1.10$ $[+0.66,+1.54]$. The
schedule helps neither arm: within 4,000 steps each arm's best validation score is 0.3 to 0.8 points below its
constant-rate run, and on test each cosine arm is within 0.6 points of its constant-rate runs.

\paragraph{Coverage and step size.} In 2,000 steps the row-sampled arms see 55\% of the training requests (25\% of those with one responder), so
part of \method{}'s speed could come from coverage. Row 8 of Table~\ref{tab:ablation} keeps the per-responder reward
but samples requests uniformly, as \method{} does. Its validation R@20 every 250 steps (2B) is 20.02, 21.78, 22.44,
22.64, 23.42, 23.77, 23.56 and 23.95, against 20.16 to 24.19 for Rank-GRPO and 21.18 to 25.57 for \method{}. Row 8
never reaches Rank-GRPO's 24.19 (78.5\% of bootstrap resamples cross nowhere); \method{} is ahead of row 8 at all eight
checkpoints (CIs above zero; $+1.62$ $[+1.08,+2.14]$ at step 2,000) and reaches row 8's step-2,000 score after 921
steps (0.46 $[0.35,0.59]$). Rank-GRPO with learning rate $2\cdot10^{-6}$ scores 21.52, 22.25, 22.97, 23.59, 23.45,
23.66, 24.12 and 24.03, and with $3\cdot10^{-6}$ 21.68, 22.23, 23.20, 23.04, 23.11, 23.10, 23.18 and 22.96: faster than
$10^{-6}$ for 750 steps, and neither reaches 24.19 within 2,000 steps (no crossing in 54\% and 99.95\% of resamples).
For \method{}, half the learning rate ($5\cdot10^{-7}$) still leads Rank-GRPO on test by $+0.69$ $[+0.29,+1.10]$, and
twice the rate scores $+0.34$ $[-0.05,+0.77]$ above $10^{-6}$ (validation 24.73, 25.57 and 26.09); selecting the rate
on validation for both methods would only widen the gap, so both keep $10^{-6}$.

\FloatBarrier

\subsection{Compute and longer training}\label{app:compute}

\paragraph{Longer training.} In the 4,000-step pairs (validation, per row), \method{} is ahead best against best
($+1.05$ $[+0.54,+1.56]$ at 2B, $+1.14$ $[+0.65,+1.64]$ at 4B), and the gap at step 4,000 itself is $+0.81$ at 2B and
$+0.72$ at 4B, down from $+1.26$ and $+1.61$ at step 2,000; Rank-GRPO is still improving at step 4,000. The plateau
pairs train both arms until 1,500 steps pass without a new validation best, up to 8,000 steps, and test each arm's best
checkpoint once (Table~\ref{tab:plateau}, Figures~\ref{fig:curves} and~\ref{fig:curves-music}). Interpolated between checkpoints as in
\S\ref{sec:exp}, \method{} reaches Rank-GRPO's best validation score with 52\% and 63\% of its GPU-hours at 2B (seeds
3407 and 42), 45\% at 4B and 18\% on Reddit-Music. At 9B and on Reddit-Books Rank-GRPO's peak stays slightly higher
(best against best $-0.33$ $[-0.83,+0.16]$ and $-0.08$ $[-0.58,+0.38]$ at 9B, $-0.36$ $[-1.02,+0.25]$ on Reddit-Books);
at 9B Rank-GRPO first passes \method{}'s best at steps 7,250 and 7,500, after 117.2 and 119.5 GPU-hours against
\method{}'s 48.0 and 47.0 to its best. On test \method{} is never significantly behind. With seed 42 at 2B Rank-GRPO
collapses after its peak (validation 3.68 at step 7,500, lists of 1.4 items on average); at 4B \method{} dips to 25.15
at step 2,500 after a KL spike, with 14\% of items written with a wrong year, and recovers by step 2,750
(Figure~\ref{fig:gpuh-raw}); on Reddit-Music \method{}'s validation falls to 12.6 to 12.8 in its last checkpoints while
its share of exact-20 lists drops to about 0.95. The Reddit-Books and Reddit-Music pairs are separate runs of the
configurations of Table~\ref{tab:main}, so their test scores differ slightly from Tables~\ref{tab:main-full-b}
and~\ref{tab:music}.

\begin{table*}[t]
\centering\footnotesize
\setlength{\tabcolsep}{3.5pt}
\begin{tabular}{@{}l cc cc ll c@{}}
\toprule
 & \multicolumn{2}{c}{best validation R@20 (step)} & \multicolumn{2}{c}{test R@20} & \multicolumn{2}{c}{$\Delta$ test R@20} & GPU-h \\
\cmidrule(lr){2-3}\cmidrule(lr){4-5}\cmidrule(lr){6-7}
pair & \method{} & Rank-GRPO & \method{} & Rank-GRPO & per row & per request & share \\
\midrule
2B & 26.15 (4,000) & 25.82 (5,250) & 22.72 & 22.58 & \ns{+0.14} [$-$0.30,0.59] & \ns{$-$0.03} [$-$0.84,0.78] & 52\% \\
2B, seed 42 & 26.23 (4,500) & 26.10 (7,250) & 23.29 & 22.87 & \ns{+0.42} [$-$0.05,0.90] & \ns{+0.43} [$-$0.45,1.29] & 63\% \\
4B & 28.33 (3,000) & 27.89 (5,250) & 25.34 & 24.51 & \up{+0.83} [0.39,1.29] & \up{+1.10} [0.33,1.90] & 45\% \\
9B & 28.89 (2,750) & 29.22 (7,250) & 25.85 & 25.77 & \ns{+0.08} [$-$0.34,0.50] & \ns{+0.50} [$-$0.27,1.25] & -- \\
9B, seed 42 & 29.00 (2,750) & 29.09 (8,000) & 25.93 & 25.88 & \ns{+0.05} [$-$0.38,0.46] & \ns{+0.20} [$-$0.62,0.98] & -- \\
Reddit-Books 2B & 31.23 (2,000) & 31.59 (3,250) & 30.77 & 30.00 & \up{+0.76} [0.12,1.47] & \ns{+0.60} [$-$0.18,1.44] & -- \\
Reddit-Music 2B & 14.41 (1,250) & 13.97 (4,750) & 14.59 & 14.45 & \ns{+0.15} [$-$0.25,0.54] & \ns{$-$0.01} [$-$0.51,0.51] & 18\% \\
\bottomrule
\end{tabular}
\caption{Plateau pairs: both arms trained until 1,500 steps pass without a new validation best (at most 8,000 steps),
each tested once at its best checkpoint; Reddit-v2 unless noted. GPU-h share: \method{}'s GPU-hours to reach Rank-GRPO's best validation score as a share of
Rank-GRPO's (interpolated; --: not reached). Seed 3407 unless noted.}
\label{tab:plateau}
\end{table*}

\begin{figure*}[tp]
\centering
\includegraphics[width=\textwidth]{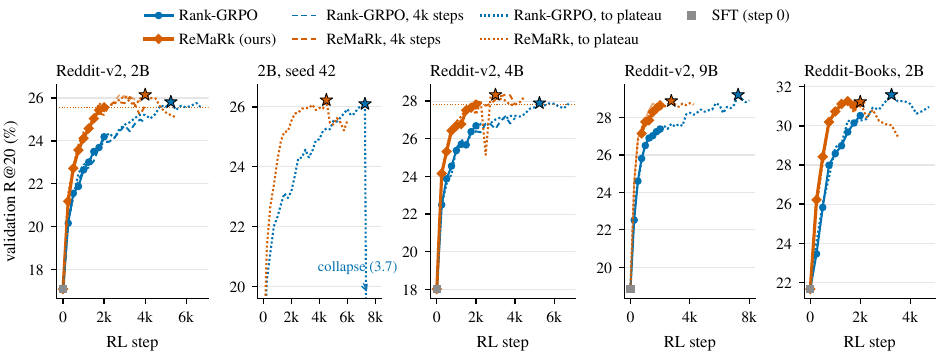}
\caption{Validation R@20 every 250 RL steps (step 0: SFT). Dashed: 4,000-step runs. Dotted curves (Reddit-v2 2B seeds 3407 and 42, 4B, 9B; Reddit-Books 2B): runs
trained until 1,500 steps pass without a new validation best or to 8,000 steps, with stars at the selected steps
(2B seed 3407: \method{} 4,000, Rank-GRPO 5,250; seed 42: 4,500 and 7,250, after which Rank-GRPO collapses to 3.68;
4B: 3,000 and 5,250, with a transient dip of \method{} at 2,500; 9B: 2,750 and 7,250; Reddit-Books: 2,000 and 3,250). Horizontal dotted line: \method{}'s best within 2,000 steps. Faint: runs of the same configuration with every checkpoint kept (Appendix~\ref{app:compute}, test at the
matching step).}
\label{fig:curves}
\end{figure*}

\begin{figure}[htbp]
\centering
\includegraphics[width=\columnwidth]{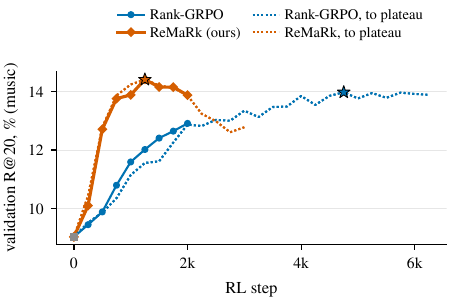}
\caption{Reddit-Music (2B): validation R@20 every 250 RL steps. Solid: 2,000-step runs; dotted: runs trained until
1,500 steps pass without a new validation best (\method{} stopped at 2,750, Rank-GRPO at 6,250), with stars at the
selected steps (1,250 and 4,750).}
\label{fig:curves-music}
\end{figure}

\begin{figure}[htbp]
\centering
\includegraphics[width=\columnwidth]{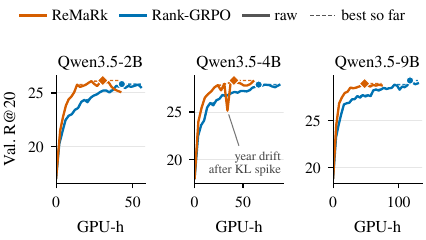}
\caption{Raw validation curves behind Figure~\ref{fig:gpuh} (seed 3407). Solid: validation R@20 of each checkpoint; dashed: best so far.}
\label{fig:gpuh-raw}
\end{figure}

\paragraph{Test at the matching step.} Validation picks both Rank-GRPO's target and the step at which \method{}
reaches it. To check the speed-up on test, we trained \method{} again at 2B, 4B and 9B (seed 3407) with every
checkpoint kept and tested the first checkpoint whose validation R@20 reaches Rank-GRPO's step-2,000 score once
against Rank-GRPO at step 2,000. The
targets are 24.19, 26.69 and 27.39; each run first reaches its target at step 1,000 (24.71, 27.24, 27.88;
interpolated crossings at steps 878, 810 and 764). On test these checkpoints score 21.31, 24.43 and 24.65 against 20.97,
23.56 and 24.05: $+0.35$ $[-0.07,+0.77]$, $+0.87$ $[+0.42,+1.33]$ and $+0.60$ $[+0.16,+1.10]$ per row; $+0.84$
$[+0.12,+1.54]$, $+1.54$ $[+0.78,+2.34]$ and $+1.01$ $[+0.33,+1.79]$ per request; N@20 $+0.16$ $[-0.03,+0.34]$,
$+0.43$ $[+0.22,+0.64]$ and $+0.26$ $[+0.06,+0.46]$ per row. After half the steps (at 2B, 7 against 13 GPU-hours,
Table~\ref{tab:compute}), \method{} is level with Rank-GRPO per row at 2B and ahead at 4B and 9B.

\subsection{Stronger baselines}\label{app:baselines}

\paragraph{Twice the requests per step.} \method{} turns groups without signal into signal, so each of its steps may
simply carry more learning signal. Rank-GRPO with 96 instead of 48 requests per step (2B, seeds 3407 and 42, otherwise as
Table~\ref{tab:main}) keeps 55\% of its groups without signal during training (training logs; \method{}: 30\%) and so sees
43 to 44 groups with signal per step, more than \method{}'s 34; a step takes 2.1 to 2.2 times as long as Rank-GRPO's. At the
GPU-hours of \method{}'s 2,000 steps (14.3 and 17.0) it reaches 23.01 and 23.34 validation R@20 (interpolated), 2.56 and
2.36 below \method{} (25.57 and 25.70), and it needs 23.4 and 20.4 GPU-hours to reach Rank-GRPO's step-2,000 score (24.19
and 23.86), against 7.5 and 8.9 for \method{} and 14.0 and 13.4 for Rank-GRPO. Its selected checkpoints (steps 2,000 and
1,750) score 21.59 and 21.35 on test: $+0.63$ $[+0.21,+1.07]$ and $+0.29$ $[-0.13,+0.69]$ relative to Rank-GRPO, and
$-0.94$ $[-1.34,-0.54]$ and $-1.10$ $[-1.51,-0.68]$ relative to \method{}, per row (per request: $+1.03$ and $-0.15$;
$-1.26$ and $-1.50$). GPU-hours here come from each run's logged
training wall-clock (seed 42: 2,000 steps at its measured 4-GPU step time, 13.1\,s), so they can differ slightly from
Table~\ref{tab:compute}.

\paragraph{Dynamic sampling.} Following DAPO \citep{yu2025dapo}, this Rank-GRPO variant (2B, seed 3407, with the
length control) generates eight lists for three times as many requests as it updates on and keeps 48 groups, groups
with signal first: 45\% of generated groups have a reward that varies at some rank position, and 99\% of the kept
ones do. Among kept groups, 0.7\% of positions have equal rewards for all eight lists, against 56\% in Rank-GRPO.
Validation R@20 every 250 steps is 20.90, 22.20, 22.71, 23.32, 23.69, 23.89, 23.76 and 24.49 (Rank-GRPO 24.19 and
\method{} 25.57 at step 2,000; \method{} passes 24.49 at step 1,250). On test it scores 21.19 against 20.97 for
Rank-GRPO ($+0.23$ $[-0.21,+0.65]$ per row, $+0.76$ $[+0.05,+1.49]$ per request); \method{} leads it by $+1.34$
$[+0.92,+1.74]$ per row and $+1.54$ $[+0.79,+2.27]$ per request. A step takes 16.3\,s against 5.65\,s for Rank-GRPO,
about 36 GPU-hours for 2,000 steps. At 4B (seed 3407, same recipe) \method{} leads it by $+1.31$ $[+0.91,+1.68]$ per row.

\paragraph{More lists per request.} Rank-GRPO with $G=16$ lists per request and 24 requests per step keeps the 384 lists
per step of $G=8$ with 48 requests, so a step costs the same (about 14 GPU-hours for 2,000 steps at 2B; 2B and 4B, seed
3407, otherwise as Table~\ref{tab:main}, selected on validation within 2,000 steps). It scores 20.72 at 2B and 23.04 at 4B
on test: $-0.25$ $[-0.68,+0.20]$ and $-0.52$ $[-0.97,-0.07]$ against Rank-GRPO with $G=8$, and $1.81$ and $2.06$ points
below \method{} (Table~\ref{tab:checks}). With $G=16$ for both arms, \method{} leads by $+1.43$ $[+0.98,+1.89]$ and
$+1.65$ $[+1.21,+2.11]$.

\paragraph{Exact Rao--Blackwellization.} Averaging Rank-GRPO's own group-scaled advantages over the responders (row
sampling, 2B) gains $+0.94$, $+0.91$ and $+0.95$ points over Rank-GRPO with seeds 3407, 42 and 1234 (each CI above
zero; $+0.93$ $[+0.83,+1.04]$ over the three), and \method{} is $+0.52$ $[+0.39,+0.64]$ above it over the three
seeds ($+0.85$ $[+0.59,+1.11]$ per request; seed 3407: $+0.63$, Table~\ref{tab:checks}). At 4B (seed 3407) the same
estimator is $+0.31$ $[-0.07,+0.70]$ below \method{} per row. With seed 3407 the remaining gap splits by metric: drawing requests uniformly instead of
by row leaves the exact average's per-row R@20 unchanged ($-0.18$ $[-0.66,+0.32]$) and raises its per-request R@20
($+1.18$ $[+0.34,+2.04]$), while the union reward adds $+0.80$ $[+0.33,+1.27]$ per row and nothing per request
($-0.09$ $[-0.90,+0.73]$); each metric favours the sampling that matches its weighting.

\paragraph{Update size and KL.} Before clipping, \method{}'s gradient norm averages 0.69, 0.64 and 0.52 at 2B, 4B and
9B against Rank-GRPO's 0.35, 0.29 and 0.24 (seed 3407, every 10th step), consistent with its larger number of groups
with signal; its KL to the SFT policy averages 1.3 to 1.5 times Rank-GRPO's over training, but where it reaches
Rank-GRPO's step-2,000 validation score its KL is only slightly higher than Rank-GRPO's at step 2,000
(Appendix~\ref{app:robust}). Larger steps alone do not reproduce the speed-up: Rank-GRPO with twice or three times the
learning rate never reaches its own step-2,000 score (Appendix~\ref{app:robust}). At 4B five logged windows of \method{} around step 1,800 have KL
above 1 (at most 263), after the step at which it reaches that score.

\FloatBarrier

\subsection{Where the gain comes from}\label{app:mech}

\begin{table}[htbp]
\centering\small
\setlength{\tabcolsep}{3pt}
\begin{adjustbox}{max width=\columnwidth}
\begin{tabular}{@{}l rrrr r@{}}
\toprule
responders & 1 & 2--4 & 5--19 & $\ge$20 & all \\
\midrule
\method{} $-$ Rank-GRPO, 2B & \up{+3.70} & \up{+1.45} & \up{+1.08} & \up{+1.88} & \up{+1.57} \\
\quad 4B & \up{+4.06} & \up{+1.82} & \up{+0.86} & \up{+1.92} & \up{+1.54} \\
\quad 9B & \up{+3.46} & \up{+2.26} & \up{+1.04} & \up{+1.31} & \up{+1.46} \\
\midrule
sampling step (2B) & \up{+2.74} & \ns{+0.45} & \ns{$-$0.23} & \ns{+0.02} & \ns{+0.16} \\
reward step (2B) & \ns{+0.96} & \ns{+1.00} & \up{+1.31} & \up{+1.86} & \up{+1.40} \\
\bottomrule
\end{tabular}
\end{adjustbox}
\tabnote{Bottom: the 2B gain split into the sampling step (row 8 $-$ row 1 of Table~\ref{tab:ablation}) and the reward step (row 7 $-$ row 8).}
\caption{Per-row test R@20 gain (points) by number of responders.}
\label{tab:buckets}
\end{table}

\paragraph{Dose-response at a fixed scale and under per-group scaling.} Under batch scaling every arm divides by its own
batch standard deviation, which is smaller for averaged rewards (Appendix~\ref{app:proofs}). With one fixed advantage
scale for every arm instead (0.130, the per-position batch standard deviation of the per-responder hit at the SFT
policy; 2B, seeds 3407 and 42, otherwise as Figure~\ref{fig:noise}b), mean validation R@20 for $m=1$, 2, 4 and all
responders is 23.52, 24.00, 24.67 and 24.86: the slope in the share of noise removed is $+1.38$ $[+0.97,+1.78]$, and
$+3.34$ $[+1.76,+4.84]$ on single-responder requests, as under batch scaling. On test the four arms score $-0.45$,
$+0.38$, $+0.66$ and $+0.94$ (seed 3407) and $-0.36$, $+0.40$, $+0.70$ and $+0.96$ (seed 42) against Rank-GRPO per row:
the fixed scale itself costs the per-responder reward about 0.4, and at the same scale the average over all responders
leads it by $+1.39$ $[+0.95,+1.84]$ and $+1.32$ $[+0.84,+1.78]$ per row ($+1.44$ and $+1.21$ per request, CIs above zero).
With per-group scaling (seeds 3407, 42 and 1234), mean validation R@20 is 24.00, 24.26, 24.65 and 23.68. The slope is
$+0.03$ $[-0.36,+0.45]$ ($-0.21$ $[-0.62,+0.22]$ at step 2,000): validation rises up to four responders and falls with
all of them, which lose 4.22 points on single-responder requests (slope there $-3.18$ $[-4.96,-1.28]$). On test, the
average over four responders is ahead of Rank-GRPO per row with every seed ($+0.29$ $[-0.12,+0.70]$, $+0.63$
$[+0.18,+1.09]$, $+0.68$ $[+0.25,+1.11]$); the average over all responders is level per row ($-0.22$ to $-0.04$) and
lower per request ($-1.80$ to $-2.45$, every CI below zero). Per-group scaling, not the averaging, removes the trend.
It does so through the item advantages, not the length shaping, whose scale grows with $m$ (19.4 on average with all
responders, 2.47 with the per-responder reward): holding it at 2.47 for the all-responder arm leaves its test R@20 unchanged ($+0.05$
$[-0.35,+0.48]$ per row; seed 3407). A likely reason is that averaging puts the reward on a finer grid, which per-group scaling inflates
(Proposition~\ref{prop:norm}), and sampling requests in proportion to their responders draws the many-responder
requests where this is largest; with uniform request sampling the averaged reward under per-group scaling is level
with \method{} (Table~\ref{tab:ablation}, row 9).

\paragraph{Sequence-level GRPO.} To test the reward outside Rank-GRPO's rank-level credit, we train ordinary GRPO in
which each list gets one reward, its recall@20 against one responder (row sampling) or against the union (request
sampling), with the group-scaled advantage on every token and the same length control (2B, seed 3407, 2,000 steps).
The union gains $+1.24$ $[+0.75,+1.70]$ per row and $+1.92$ $[+0.98,+2.85]$ per request over the per-responder
reward (test 19.74 against 18.50); 70\% of the per-responder groups and 40\% of the union groups have no signal.
Both stay below rank-level training (Rank-GRPO 20.97, \method{} 22.53).

\paragraph{Rewarding consensus.} Weighting each matched item by $1+0.25\cdot\1\{c_i\ge2\}$ instead of the union's
$\1\{c_i\ge1\}$ (2B, seed 3407, otherwise \method{}) scores 22.35 on test: $-0.18$ $[-0.49,+0.13]$ against \method{}
per row ($-0.19$ $[-0.72,+0.31]$ per request) and $+1.39$ $[+0.92,+1.85]$ against Rank-GRPO. Recall of items named by
two or more responders does not rise (45.48 against 46.71, $-1.23$ $[-3.09,+0.54]$), and tail items fall slightly
($-0.36$ $[-0.70,-0.02]$).

\paragraph{Single answers.} We split each test request's union items into those named by exactly one responder
($U_1$; 92.5\% of test union items on Reddit-v2, 95.1\% on Reddit-Books and 95.4\% on Reddit-Music)
in the top 20. \method{} minus Rank-GRPO on $U_1$: $+2.13$ $[+1.30,+2.95]$, $+2.36$ $[+1.59,+3.18]$ and $+2.27$
$[+1.49,+3.06]$ at 2B, 4B and 9B; $+0.81$ $[+0.01,+1.65]$ on Reddit-Books 2B ($+0.85$ $[+0.11,+1.60]$ against the
batch-scaled per-responder arm); $+1.31$ $[+0.74,+1.92]$ on Reddit-Music. On $U_2$: $+3.13$, $+2.97$ and $+3.61$ at 2B
to 9B (CIs above zero), $+1.18$ $[-2.31,+4.50]$ on Reddit-Books and $-0.62$ $[-3.23,+1.93]$ on Reddit-Music. Most test
requests have no item named twice (1,607 of 2,050 on Reddit-v2); on them the per-row R@20 gain is $+1.27$, $+1.81$ and
$+1.75$ at 2B, 4B and 9B (CIs above zero), $+0.62$ $[-0.19,+1.39]$ on Reddit-Books and $+1.20$ $[+0.71,+1.68]$ on
Reddit-Music, against $+1.76$, $+1.37$, $+1.27$, $+1.12$ and $+0.94$ on the other requests.

\FloatBarrier

\subsection{Judges and popularity}\label{app:judge}

\paragraph{Zero-shot, judge and popularity.} GPT-5.6 zero-shot, the strongest system without RL, is level with 2B
\method{} on Reddit-v2 and behind it at 4B (Table~\ref{tab:checks}), at 9B ($+2.88$ $[+1.89,+3.85]$) and on Reddit-Books
($+6.50$ $[+5.30,+7.80]$), although it may have seen these threads. The blind GPT-5.6 judge never sees the responders'
items; at 9B it scored 299 requests in both orders (order consistency 0.84), and on the 4B plateau pair 298 (0.866).
The second judge, Gemini 3.1 Pro (gemini-3.1-pro-preview), uses the same prompt and requests: order consistency 0.78
(9B) and 0.85 (4B); agreement with GPT-5.6, Cohen's $\kappa=0.53$ and 0.51 per call. For precision@20 each call shows
one list and the request and asks for a fit / no-fit verdict per title (500 test requests drawn with seed 42; at most
40 of 20,000 titles per judge and pair unparsed, counted as no-fit). Gemini is stricter (57.6\% of titles fit against
67.0\% for GPT-5.6; item-level $\kappa=0.71$). By popularity (cut-offs of Table~\ref{tab:popstrata}), \method{} is more
precise on head items at 4B (Gemini $+3.8$ $[+1.5,+6.0]$, GPT-5.6 $+2.2$ $[+0.8,+3.7]$) and, for GPT-5.6, at 9B
($+2.0$ $[+0.3,+3.9]$); on tail and torso items every interval includes zero.
On test, \method{}'s gain is largest on torso and head items but stays positive on the tail at every Reddit-v2 size;
on Reddit-Books it sits on the torso (Tables~\ref{tab:popstrata} and~\ref{tab:pop}).

\begin{table}[htbp]
\centering\small
\setlength{\tabcolsep}{3.5pt}
\begin{adjustbox}{max width=\columnwidth}
\begin{tabular}{@{}l cc cc@{}}
\toprule
 & tail & head & coverage & mean pop. \\
\midrule
SFT & 9.60 & 29.96 & 5,158 & 131 \\
Rank-GRPO (1) & 11.76 & 44.24 & 3,566 & 185 \\
per-responder, row (2) & 10.93 & 44.93 & 3,249 & 189 \\
per-responder, request (3) & 11.92 & 41.57 & 4,073 & 166 \\
mean, row (4) & 11.37 & 46.77 & 3,450 & 190 \\
mean, request (5) & 12.17 & 45.40 & 3,756 & 180 \\
bounded sum, batch (6) & 12.02 & 46.70 & 3,282 & 189 \\
\method{} (7) & 12.04 & 46.65 & 3,516 & 182 \\
\bottomrule
\end{tabular}
\end{adjustbox}
\tabnote{Rows of Table~\ref{tab:ablation}. Tail / head: recall@20 (\%) of items at or below the median / in the top 20\% of training popularity; coverage: distinct items recommended; mean pop.: mean training frequency of recommended items.}
\caption{Popularity and exposure (2B, Reddit-v2 validation).}
\label{tab:pop}
\end{table}

\begin{table}[htbp]
\centering\small
\setlength{\tabcolsep}{4pt}
\begin{tabular}{@{}l rrr@{}}
\toprule
 & tail & torso & head \\
\midrule
2B & \up{+0.46} & \up{+2.09} & \up{+3.36} \\
4B & \up{+0.89} & \up{+1.24} & \up{+2.66} \\
9B & \up{+0.61} & \up{+2.33} & \up{+1.40} \\
Reddit-Books 2B & \ns{+0.24} & \up{+2.02} & \ns{$-$0.09} \\
Reddit-Books 4B & \ns{+0.62} & \up{+2.07} & \ns{$-$0.11} \\
\bottomrule
\end{tabular}
\tabnote{Over Rank-GRPO (Reddit-v2) or the batch-scaled per-responder arm (Reddit-Books). Tail, torso, head: training frequency at or below the median, up to the 80th percentile, above it (quantiles over the test ground truth).}
\caption{Per-row test R@20 gain (points) of \method{} by item popularity.}
\label{tab:popstrata}
\end{table}

\FloatBarrier

\subsection{Full results}\label{app:full}

Tables~\ref{tab:main-full}, \ref{tab:main-full-llama} and~\ref{tab:main-full-b} give every metric behind Table~\ref{tab:main};
Tables~\ref{tab:ablation-val}, \ref{tab:ablation-full} and~\ref{tab:ablation-diffs} expand Table~\ref{tab:ablation}.

\begin{table*}[tp]
\centering\small
\setlength{\tabcolsep}{5pt}
\renewcommand{\arraystretch}{1.08}
\begin{tabular}{@{}ll cc cc c@{}}
\toprule
 & & \multicolumn{2}{c}{per row} & \multicolumn{2}{c}{per request} & \\
\cmidrule(lr){3-4}\cmidrule(lr){5-6}
\textbf{backbone} & \textbf{method} & R@20 & N@20 & R@20 & N@20 & exact-20 \\
\midrule
\multicolumn{7}{@{}l}{\grp{Reddit-v2, Qwen3.5} (10,972 rows, 2,050 requests)} \\
-- & Popularity (train) & \rf{2.75} & \rf{1.29} & \rf{2.68} & \rf{1.23} & -- \\
-- & kNN prior ($k{=}200$) & \rf{15.94} & \rf{7.76} & \rf{15.99} & \rf{7.90} & -- \\
-- & GPT-5.6 zero-shot & \rf{22.63} & \rf{11.61} & \rf{23.24} & \rf{12.37} & -- \\
-- & Qwen3.5-27B zero-shot & \rf{15.22} & \rf{7.58} & \rf{15.07} & \rf{7.72} & -- \\
-- & GPT-4o zero-shot$^\dagger$ & \rf{21.47} & \rf{11.97} & -- & -- & -- \\
\midrule
Qwen3.5-2B & SFT & 15.40 & 8.01 & 15.32 & 8.15 & 99.9 \\
 & SFT on pooled positives & 15.94 & 8.09 & 15.35 & 8.68 & 100.0 \\
 & DPO, pooled positives & 17.08 & 8.82 & 16.06 & 8.72 & 99.9 \\
\baserow  & Rank-GRPO & 20.97 & 10.87 & 20.94 & 11.25 & 100.0 \\
 & Rank-GRPO, $G{=}16$ & 20.72 & 10.64 & 20.80 & 11.04 & 99.7 \\
 & per-responder, batch & 20.43 & 10.63 & 19.92 & 10.89 & 99.9 \\
 & bounded sum, batch & \second{21.84} & \second{11.32} & \second{22.38} & \second{12.15} & 99.9 \\
\oursrow  & \method{} (ours) & \best{22.53} & \best{11.62} & \best{23.24} & \best{12.78} & 99.9 \\
 & \quad$\Delta$ vs.\ Rank-GRPO & \dstack{\up{+1.57}}{[1.10,2.06]} & \dstack{\up{+0.76}}{[0.53,0.99]} & \dstack{\up{+2.29}}{[1.43,3.17]} & \dstack{\up{+1.53}}{[1.04,2.04]} & \\
\midrule
Qwen3.5-4B & SFT & 16.53 & 8.92 & 16.19 & 8.86 & 52.3 \\
 & SFT on pooled positives & 17.25 & 8.83 & 17.59 & 9.91 & 98.4 \\
 & DPO, pooled positives & 18.17 & 9.55 & 18.19 & 9.90 & 96.2 \\
\baserow  & Rank-GRPO & 23.56 & 12.30 & 24.82 & 13.51 & 100.0 \\
 & Rank-GRPO, $G{=}16$ & 23.04 & 12.03 & 24.51 & 13.21 & 99.8 \\
 & per-responder, batch & 23.16 & 12.13 & 24.56 & 13.63 & 99.9 \\
 & bounded sum, batch & \second{25.02} & \second{13.08} & \second{27.07} & \second{15.09} & 99.6 \\
\oursrow  & \method{} (ours) & \best{25.10} & \best{13.12} & \best{27.40} & \best{15.14} & 99.8 \\
 & \quad$\Delta$ vs.\ Rank-GRPO & \dstack{\up{+1.54}}{[1.10,2.00]} & \dstack{\up{+0.82}}{[0.61,1.04]} & \dstack{\up{+2.58}}{[1.82,3.41]} & \dstack{\up{+1.64}}{[1.17,2.12]} & \\
\midrule
Qwen3.5-9B & SFT & 17.23 & 9.31 & 16.89 & 9.44 & 99.6 \\
\baserow  & Rank-GRPO & 24.05 & 12.72 & 25.95 & 14.66 & 99.8 \\
 & per-responder, batch & 23.95 & 12.62 & 26.06 & 14.61 & 100.0 \\
 & bounded sum, batch & \best{25.52} & \second{13.51} & \best{28.53} & \second{16.07} & 99.9 \\
\oursrow  & \method{} (ours) & \second{25.51} & \best{13.56} & \second{28.35} & \best{16.08} & 99.9 \\
 & \quad$\Delta$ vs.\ Rank-GRPO & \dstack{\up{+1.46}}{[1.02,1.93]} & \dstack{\up{+0.84}}{[0.62,1.07]} & \dstack{\up{+2.40}}{[1.64,3.22]} & \dstack{\up{+1.42}}{[0.94,1.90]} & \\
\midrule
Qwen3.5-27B & SFT$^\ddagger$ & 17.90 & 9.77 & 17.11 & 9.94 & 7.3 \\
\baserow  & Rank-GRPO & 25.17 & 13.24 & 28.38 & 15.62 & 99.9 \\
 & per-responder, batch & 23.96 & 12.51 & 25.56 & 13.61 & 99.9 \\
 & bounded sum, batch & \second{26.32} & \second{13.81} & \best{30.04} & \second{16.67} & 99.9 \\
\oursrow  & \method{} (ours) & \best{26.62} & \best{14.10} & \second{29.84} & \best{16.83} & 99.6 \\
 & \quad$\Delta$ vs.\ Rank-GRPO & \dstack{\up{+1.45}}{[0.98,1.94]} & \dstack{\up{+0.86}}{[0.64,1.09]} & \dstack{\up{+1.46}}{[0.68,2.25]} & \dstack{\up{+1.21}}{[0.78,1.66]} & \\
\midrule
Qwen3.5-35B-A3B & SFT & 17.34 & 9.47 & 17.23 & 9.78 & 99.5 \\
\baserow  & Rank-GRPO & 24.74 & 13.10 & \second{27.81} & 15.78 & 99.9 \\
 & per-responder, batch & 24.58 & 13.01 & 27.44 & 15.47 & 99.9 \\
 & bounded sum, batch & \second{26.16} & \best{14.02} & \best{29.76} & \best{17.24} & 100.0 \\
\oursrow  & \method{} (ours) & \best{26.20} & \second{13.94} & \best{29.76} & \second{17.13} & 99.9 \\
 & \quad$\Delta$ vs.\ Rank-GRPO & \dstack{\up{+1.47}}{[1.03,1.90]} & \dstack{\up{+0.85}}{[0.64,1.04]} & \dstack{\up{+1.95}}{[1.14,2.78]} & \dstack{\up{+1.35}}{[0.91,1.80]} & \\
\bottomrule
\end{tabular}

\tabnote{R@20 and N@20 (\%) per row and per request; exact-20 lists (\%). $\Delta$: \method{} minus the named arm; bounded sum, batch: \method{}'s reward with batch scaling. $^\dagger$Reported by \citet{zhu2026rankgrpo} on the same split. $^\ddagger$The 27B SFT checkpoint writes 19-line lists (10,141 of 10,972). The batch-scaled 27B and 35B-A3B rows start from a retrained SFT checkpoint (Appendix~\ref{app:repro}).}
\caption{Full test results on Reddit-v2, Qwen3.5.}
\label{tab:main-full}
\end{table*}

\begin{table*}[tp]
\centering\small
\setlength{\tabcolsep}{5pt}
\renewcommand{\arraystretch}{1.08}
\begin{tabular}{@{}ll cc cc c@{}}
\toprule
 & & \multicolumn{2}{c}{per row} & \multicolumn{2}{c}{per request} & \\
\cmidrule(lr){3-4}\cmidrule(lr){5-6}
\textbf{backbone} & \textbf{method} & R@20 & N@20 & R@20 & N@20 & exact-20 \\
\midrule
\multicolumn{7}{@{}l}{\grp{Reddit-v2, Llama} (10,972 rows, 2,050 requests)} \\
Llama-3.2-3B & Rank-GRPO, released (15.8k steps) & \rf{23.65} & \rf{12.78} & \rf{25.33} & \rf{14.50} & 47.7 \\
\midrule
Llama-3.2-3B & SFT & 15.38 & 8.37 & 14.66 & 8.28 & 99.8 \\
\baserow  & Rank-GRPO & 21.93 & 11.51 & 22.97 & 12.59 & 99.0 \\
 & per-responder, batch & 22.28 & \second{11.69} & 23.89 & 12.89 & 99.0 \\
 & bounded sum, batch & \second{23.43} & \best{12.45} & \second{25.41} & \best{14.38} & 100.0 \\
\oursrow  & \method{} (ours) & \best{23.56} & \best{12.45} & \best{25.62} & \second{14.26} & 85.2 \\
 & \quad$\Delta$ vs.\ per-responder, batch & \dstack{\up{+1.28}}{[0.86,1.71]} & \dstack{\up{+0.77}}{[0.57,0.97]} & \dstack{\up{+1.73}}{[0.92,2.52]} & \dstack{\up{+1.38}}{[0.92,1.85]} & \\
 & \quad$\Delta$ vs.\ Rank-GRPO & \dstack{\up{+1.62}}{[1.18,2.08]} & \dstack{\up{+0.95}}{[0.75,1.16]} & \dstack{\up{+2.65}}{[1.86,3.50]} & \dstack{\up{+1.67}}{[1.18,2.18]} & \\
\midrule
Llama-3.1-8B & SFT & 16.84 & 8.99 & 16.38 & 8.95 & 99.9 \\
\baserow  & Rank-GRPO & 24.22 & 12.73 & 26.51 & 14.82 & 98.8 \\
 & per-responder, batch & 23.89 & 12.57 & 26.12 & 14.53 & 97.0 \\
 & bounded sum, batch & \best{25.67} & \best{13.50} & \best{28.11} & \best{15.74} & 92.8 \\
\oursrow  & \method{} (ours) & \second{25.22} & \second{13.25} & \second{27.87} & \second{15.55} & 94.9 \\
 & \quad$\Delta$ vs.\ Rank-GRPO & \dstack{\up{+1.00}}{[0.55,1.44]} & \dstack{\up{+0.52}}{[0.32,0.74]} & \dstack{\up{+1.35}}{[0.55,2.16]} & \dstack{\up{+0.73}}{[0.35,1.14]} & \\
\bottomrule
\end{tabular}

\tabnote{As Table~\ref{tab:main-full}.}
\caption{Full test results on Reddit-v2, Llama.}
\label{tab:main-full-llama}
\end{table*}

\begin{table*}[tp]
\centering\small
\setlength{\tabcolsep}{3.8pt}
\renewcommand{\arraystretch}{1.08}
\begin{tabular}{@{}r lll cccc c@{}}
\toprule
 & \textbf{aggregation} & \textbf{scale} & \textbf{sampling} & R@20 & req.\ R@20 & R@20, $n{=}1$ & R@20, $n{\ge}5$ & exact-20 \\
\midrule
0 & \multicolumn{3}{@{}l}{SFT (no RL)} & 17.09 & 16.16 & 13.75 & 17.30 & 99.9 \\
\baserow 1 & per-responder & group & row (Rank-GRPO) & 24.19 & 24.55 & 23.75 & 24.16 & 100.0 \\
2 & per-responder & batch & row & 23.96 & 23.83 & 22.40 & 24.00 & 100.0 \\
3 & per-responder & batch & request & 23.95 & 25.83 & 27.31 & 23.54 & 99.9 \\
4 & mean & batch & row & 25.15 & 25.89 & 25.44 & 24.99 & 100.0 \\
5 & mean & batch & request & 24.94 & 26.90 & \second{28.46} & 24.49 & 99.8 \\
6 & bounded sum & batch & request & 25.31 & 26.84 & 27.34 & 24.93 & 99.7 \\
\oursrow 7 & bounded sum & group & request (\method) & \second{25.57} & \second{27.02} & 27.66 & \best{25.32} & 99.9 \\
8 & per-responder & group & request & 23.95 & 25.56 & 26.68 & 23.61 & 99.6 \\
9 & mean & group & request & \best{25.62} & \best{27.37} & \best{28.63} & \second{25.27} & 100.0 \\
\midrule
\multicolumn{4}{@{}l}{\method{} vs.\ Rank-GRPO (7$-$1)} & \dstack{\up{+1.38}}{[0.87,1.89]} & \dstack{\up{+2.46}}{[1.52,3.38]} & \dstack{\up{+3.90}}{[1.91,5.93]} & \dstack{\up{+1.16}}{[0.58,1.75]} &  \\
\multicolumn{4}{@{}l}{scale, bounded sum (7$-$6)} & \dstack{\ns{+0.26}}{[$-$0.18,0.64]} & \dstack{\ns{+0.18}}{[$-$0.62,0.95]} & \dstack{\ns{+0.32}}{[$-$1.44,2.03]} & \dstack{\ns{+0.39}}{[$-$0.08,0.88]} &  \\
\multicolumn{4}{@{}l}{bounded sum, batch scale vs.\ Rank-GRPO (6$-$1)} & \dstack{\up{+1.12}}{[0.64,1.62]} & \dstack{\up{+2.29}}{[1.42,3.18]} & \dstack{\up{+3.58}}{[1.68,5.71]} & \dstack{\up{+0.77}}{[0.21,1.32]} &  \\
\multicolumn{4}{@{}l}{reward, row sampling (4$-$2)} & \dstack{\up{+1.19}}{[0.69,1.75]} & \dstack{\up{+2.07}}{[1.00,3.13]} & \dstack{\up{+3.04}}{[0.67,5.31]} & \dstack{\up{+0.99}}{[0.42,1.56]} &  \\
\multicolumn{4}{@{}l}{reward, request sampling (5$-$3)} & \dstack{\up{+0.99}}{[0.49,1.50]} & \dstack{\up{+1.06}}{[0.24,1.93]} & \dstack{\ns{+1.15}}{[$-$0.79,2.94]} & \dstack{\up{+0.94}}{[0.38,1.50]} &  \\
\multicolumn{4}{@{}l}{reward, request sampling, bounded sum (6$-$3)} & \dstack{\up{+1.36}}{[0.82,1.93]} & \dstack{\up{+1.01}}{[0.14,1.90]} & \dstack{\ns{+0.03}}{[$-$1.95,2.01]} & \dstack{\up{+1.38}}{[0.73,2.01]} &  \\
\multicolumn{4}{@{}l}{sampling, mean (5$-$4)} & \dstack{\ns{$-$0.21}}{[$-$0.74,0.34]} & \dstack{\up{+1.00}}{[0.05,1.99]} & \dstack{\up{+3.02}}{[0.90,5.29]} & \dstack{\ns{$-$0.50}}{[$-$1.09,0.13]} &  \\
\multicolumn{4}{@{}l}{sampling, per-responder (3$-$2)} & \dstack{\ns{0.00}}{[$-$0.64,0.63]} & \dstack{\up{+2.01}}{[0.86,3.13]} & \dstack{\up{+4.91}}{[2.43,7.50]} & \dstack{\ns{$-$0.46}}{[$-$1.18,0.25]} &  \\
\multicolumn{4}{@{}l}{bounded sum vs.\ mean, request sampling (6$-$5)} & \dstack{\ns{+0.37}}{[$-$0.11,0.84]} & \dstack{\ns{$-$0.06}}{[$-$0.94,0.90]} & \dstack{\ns{$-$1.12}}{[$-$3.14,0.89]} & \dstack{\ns{+0.44}}{[$-$0.07,0.95]} &  \\
\multicolumn{4}{@{}l}{two estimators of the row objective (6$-$4)} & \dstack{\ns{+0.16}}{[$-$0.32,0.64]} & \dstack{\up{+0.94}}{[0.11,1.89]} & \dstack{\ns{+1.90}}{[$-$0.14,3.91]} & \dstack{\ns{$-$0.07}}{[$-$0.58,0.45]} &  \\
\multicolumn{4}{@{}l}{scale, per-responder (2$-$1)} & \dstack{\ns{$-$0.23}}{[$-$0.70,0.24]} & \dstack{\ns{$-$0.73}}{[$-$1.53,0.14]} & \dstack{\ns{$-$1.36}}{[$-$3.34,0.53]} & \dstack{\ns{$-$0.16}}{[$-$0.68,0.35]} &  \\
\multicolumn{4}{@{}l}{reward, group scale (7$-$8)} & \dstack{\up{+1.62}}{[1.08,2.14]} & \dstack{\up{+1.46}}{[0.62,2.23]} & \dstack{\ns{+0.98}}{[$-$0.65,2.69]} & \dstack{\up{+1.71}}{[1.13,2.34]} &  \\
\multicolumn{4}{@{}l}{sampling, group scale (8$-$1)} & \dstack{\ns{$-$0.24}}{[$-$0.83,0.36]} & \dstack{\up{+1.00}}{[0.02,2.00]} & \dstack{\up{+2.92}}{[0.86,5.07]} & \dstack{\ns{$-$0.55}}{[$-$1.22,0.09]} &  \\
\multicolumn{4}{@{}l}{bounded sum vs.\ mean, group scale (7$-$9)} & \dstack{\ns{$-$0.05}}{[$-$0.45,0.32]} & \dstack{\ns{$-$0.36}}{[$-$1.04,0.29]} & \dstack{\ns{$-$0.97}}{[$-$2.48,0.53]} & \dstack{\ns{+0.05}}{[$-$0.37,0.49]} &  \\
\multicolumn{4}{@{}l}{scale, mean (9$-$5)} & \dstack{\up{+0.68}}{[0.24,1.11]} & \dstack{\ns{+0.48}}{[$-$0.35,1.36]} & \dstack{\ns{+0.17}}{[$-$1.73,2.12]} & \dstack{\up{+0.78}}{[0.30,1.28]} &  \\
\bottomrule
\end{tabular}

\tabnote{2B, Reddit-v2 validation (1,752 requests); each arm at its best validated step; exact-20 in \%.}
\caption{Validation version of Table~\ref{tab:ablation}.}
\label{tab:ablation-val}
\end{table*}

\begin{table*}[tp]
\centering\small
\setlength{\tabcolsep}{5pt}
\renewcommand{\arraystretch}{1.08}
\begin{tabular}{@{}ll cc cc c@{}}
\toprule
 & & \multicolumn{2}{c}{per row} & \multicolumn{2}{c}{per request} & \\
\cmidrule(lr){3-4}\cmidrule(lr){5-6}
\textbf{backbone} & \textbf{method} & R@20 & N@20 & R@20 & N@20 & exact-20 \\
\midrule
\multicolumn{7}{@{}l}{\grp{Reddit-Books} (6,333 rows, 1,150 requests)} \\
-- & Popularity (train) & \rf{5.77} & \rf{2.60} & \rf{5.32} & \rf{2.39} & -- \\
-- & kNN prior ($k{=}200$) & \rf{21.60} & \rf{11.00} & \rf{21.58} & \rf{10.94} & -- \\
-- & GPT-5.6 zero-shot & \rf{23.66} & \rf{12.43} & \rf{25.41} & \rf{14.02} & -- \\
-- & Qwen3.5-27B zero-shot & \rf{13.31} & \rf{7.88} & \rf{13.36} & \rf{8.33} & -- \\
\midrule
Qwen3.5-2B & SFT & 20.67 & 10.76 & 21.34 & 11.29 & 90.5 \\
\baserow  & Rank-GRPO & 29.34 & 15.96 & 30.43 & 17.16 & 99.8 \\
 & per-responder, batch & 29.35 & 16.04 & 30.42 & 17.34 & 99.9 \\
 & bounded sum, batch & \best{30.59} & \best{16.78} & \best{31.34} & \best{17.84} & 97.8 \\
\oursrow  & \method{} (ours) & \second{30.16} & \second{16.69} & \second{31.12} & \second{17.79} & 99.4 \\
 & \quad$\Delta$ vs.\ per-responder, batch & \dstack{\up{+0.81}}{[0.19,1.46]} & \dstack{\up{+0.65}}{[0.36,0.96]} & \dstack{\ns{+0.70}}{[$-$0.15,1.52]} & \dstack{\up{+0.46}}{[0.02,0.89]} & \\
 & \quad$\Delta$ vs.\ Rank-GRPO & \dstack{\up{+0.82}}{[0.12,1.50]} & \dstack{\up{+0.73}}{[0.42,1.05]} & \dstack{\ns{+0.69}}{[$-$0.18,1.56]} & \dstack{\up{+0.64}}{[0.16,1.14]} & \\
\midrule
Qwen3.5-4B & SFT & 21.73 & 11.42 & 22.53 & 12.12 & 91.4 \\
\baserow  & Rank-GRPO & 30.40 & 16.62 & 31.72 & 17.96 & 99.8 \\
 & per-responder, batch & 30.51 & 16.76 & 31.92 & 18.44 & 99.8 \\
 & bounded sum, batch & \second{31.40} & \second{17.41} & \second{32.73} & \second{18.97} & 96.8 \\
\oursrow  & \method{} (ours) & \best{31.49} & \best{17.44} & \best{33.28} & \best{19.05} & 99.7 \\
 & \quad$\Delta$ vs.\ per-responder, batch & \dstack{\up{+0.98}}{[0.35,1.64]} & \dstack{\up{+0.68}}{[0.35,1.04]} & \dstack{\up{+1.36}}{[0.51,2.27]} & \dstack{\up{+0.61}}{[0.09,1.13]} & \\
 & \quad$\Delta$ vs.\ Rank-GRPO & \dstack{\up{+1.09}}{[0.42,1.78]} & \dstack{\up{+0.82}}{[0.48,1.14]} & \dstack{\up{+1.55}}{[0.67,2.49]} & \dstack{\up{+1.09}}{[0.53,1.68]} & \\
\midrule
\multicolumn{7}{@{}l}{\grp{Reddit-Music} (15,893 rows, 1,931 requests)} \\
-- & Popularity (train) & \rf{7.97} & \rf{3.64} & \rf{6.46} & \rf{2.90} & -- \\
-- & kNN prior ($k{=}200$) & \rf{9.17} & \rf{4.32} & \rf{8.19} & \rf{3.81} & -- \\
-- & GPT-5.6 zero-shot & \rf{7.40} & \rf{3.79} & \rf{7.16} & \rf{3.86} & -- \\
-- & GPT-5.6, year convention stated & \rf{12.13} & \rf{6.27} & \rf{12.43} & \rf{6.70} & -- \\
\midrule
Qwen3.5-2B & SFT & 9.31 & 4.67 & 8.97 & 4.68 & 60.1 \\
\baserow  & Rank-GRPO & 13.40 & 6.59 & 12.77 & 6.56 & 99.9 \\
 & per-responder, batch & 13.71 & 6.72 & 12.97 & 6.48 & 99.9 \\
 & bounded sum, batch & \best{14.50} & \best{7.16} & \second{13.63} & \second{6.87} & 99.1 \\
\oursrow  & \method{} (ours) & \second{14.48} & \second{7.10} & \best{14.05} & \best{6.97} & 98.5 \\
 & \quad$\Delta$ vs.\ per-responder, batch & \dstack{\up{+0.78}}{[0.38,1.17]} & \dstack{\up{+0.38}}{[0.22,0.54]} & \dstack{\up{+1.08}}{[0.49,1.69]} & \dstack{\up{+0.50}}{[0.25,0.75]} & \\
 & \quad$\Delta$ vs.\ Rank-GRPO & \dstack{\up{+1.08}}{[0.66,1.51]} & \dstack{\up{+0.51}}{[0.33,0.69]} & \dstack{\up{+1.28}}{[0.73,1.87]} & \dstack{\up{+0.42}}{[0.15,0.68]} & \\
\bottomrule
\end{tabular}

\tabnote{As Table~\ref{tab:main-full}. GPT-5.6 loses most music hits to the year convention.}
\caption{Full test results on Reddit-Books and Reddit-Music.}
\label{tab:main-full-b}
\end{table*}

\begin{table*}[tp]
\centering\small
\setlength{\tabcolsep}{3.8pt}
\renewcommand{\arraystretch}{1.1}
\begin{tabular}{@{}r lll cccc@{}}
\toprule
 & \textbf{aggregation} & \textbf{scale} & \textbf{sampling} & R@20 & req.\ R@20 & R@20, $n{=}1$ & R@20, $n{\ge}5$ \\
\midrule
0 & \multicolumn{3}{@{}l}{SFT (no RL)} & 15.40 & 15.32 & 14.75 & 15.57 \\
\baserow 1 & per-responder & group & row (Rank-GRPO) & 20.97 & 20.94 & 20.05 & 21.12 \\
2 & per-responder & batch & row & 20.43 & 19.92 & 18.35 & 20.71 \\
3 & per-responder & batch & request & 20.93 & 22.19 & 23.16 & 20.73 \\
4 & mean & batch & row & 21.79 & 22.25 & 22.63 & 21.90 \\
5 & mean & batch & request & \second{21.95} & \best{23.60} & \best{25.06} & 21.63 \\
6 & bounded sum & batch & request & 21.84 & 22.38 & 22.46 & 21.86 \\
\oursrow 7 & bounded sum & group & request (\method) & \best{22.53} & 23.24 & 23.76 & \best{22.49} \\
8 & per-responder & group & request & 21.13 & 22.13 & 22.79 & 20.98 \\
9 & mean & group & request & \best{22.53} & \second{23.41} & \second{24.11} & \second{22.44} \\
\bottomrule
\end{tabular}

\tabnote{Rows of Table~\ref{tab:ablation}.}
\caption{Ablations by number of responders $n$ (2B, Reddit-v2 test, R@20).}
\label{tab:ablation-full}
\end{table*}

\begin{table*}[tp]
\centering\small
\setlength{\tabcolsep}{4pt}
\renewcommand{\arraystretch}{1.08}
\begin{tabular}{@{}r lll cccc@{}}
\toprule
\multicolumn{4}{@{}l}{difference (rows of Table~\ref{tab:ablation})} & R@20 & req.\ R@20 & R@20, $n{=}1$ & R@20, $n{\ge}5$ \\
\midrule
\multicolumn{4}{@{}l}{\method{} vs.\ Rank-GRPO (7$-$1)} & \dstack{\up{+1.57}}{[1.10,2.06]} & \dstack{\up{+2.29}}{[1.43,3.17]} & \dstack{\up{+3.70}}{[1.95,5.53]} & \dstack{\up{+1.37}}{[0.86,1.89]} \\
\multicolumn{4}{@{}l}{scale, bounded sum (7$-$6)} & \dstack{\up{+0.69}}{[0.32,1.06]} & \dstack{\up{+0.86}}{[0.20,1.54]} & \dstack{\ns{+1.29}}{[$-$0.04,2.69]} & \dstack{\up{+0.63}}{[0.21,1.03]} \\
\multicolumn{4}{@{}l}{bounded sum, batch scale vs.\ Rank-GRPO (6$-$1)} & \dstack{\up{+0.87}}{[0.47,1.28]} & \dstack{\up{+1.43}}{[0.66,2.22]} & \dstack{\up{+2.41}}{[0.61,4.09]} & \dstack{\up{+0.74}}{[0.32,1.18]} \\
\multicolumn{4}{@{}l}{reward, row sampling (4$-$2)} & \dstack{\up{+1.35}}{[0.91,1.78]} & \dstack{\up{+2.32}}{[1.46,3.18]} & \dstack{\up{+4.28}}{[2.38,6.25]} & \dstack{\up{+1.19}}{[0.68,1.70]} \\
\multicolumn{4}{@{}l}{reward, request sampling (5$-$3)} & \dstack{\up{+1.02}}{[0.53,1.50]} & \dstack{\up{+1.42}}{[0.63,2.24]} & \dstack{\up{+1.90}}{[0.37,3.42]} & \dstack{\up{+0.90}}{[0.40,1.42]} \\
\multicolumn{4}{@{}l}{reward, request sampling, bounded sum (6$-$3)} & \dstack{\up{+0.91}}{[0.43,1.34]} & \dstack{\ns{+0.19}}{[$-$0.61,0.98]} & \dstack{\ns{$-$0.70}}{[$-$2.21,0.84]} & \dstack{\up{+1.13}}{[0.64,1.62]} \\
\multicolumn{4}{@{}l}{sampling, mean (5$-$4)} & \dstack{\ns{+0.16}}{[$-$0.30,0.59]} & \dstack{\up{+1.36}}{[0.51,2.21]} & \dstack{\up{+2.43}}{[0.56,4.15]} & \dstack{\ns{$-$0.27}}{[$-$0.78,0.24]} \\
\multicolumn{4}{@{}l}{sampling, per-responder (3$-$2)} & \dstack{\up{+0.50}}{[0.01,0.98]} & \dstack{\up{+2.26}}{[1.34,3.17]} & \dstack{\up{+4.81}}{[2.93,6.66]} & \dstack{\ns{+0.02}}{[$-$0.55,0.55]} \\
\multicolumn{4}{@{}l}{bounded sum vs.\ mean, request sampling (6$-$5)} & \dstack{\ns{$-$0.11}}{[$-$0.54,0.30]} & \dstack{\dn{$-$1.23}}{[$-$2.02,$-$0.46]} & \dstack{\dn{$-$2.59}}{[$-$4.20,$-$1.02]} & \dstack{\ns{+0.24}}{[$-$0.22,0.70]} \\
\multicolumn{4}{@{}l}{two estimators of the row objective (6$-$4)} & \dstack{\ns{+0.05}}{[$-$0.36,0.43]} & \dstack{\ns{+0.13}}{[$-$0.62,0.83]} & \dstack{\ns{$-$0.17}}{[$-$1.65,1.32]} & \dstack{\ns{$-$0.03}}{[$-$0.50,0.40]} \\
\multicolumn{4}{@{}l}{scale, per-responder (2$-$1)} & \dstack{\dn{$-$0.53}}{[$-$0.91,$-$0.14]} & \dstack{\dn{$-$1.02}}{[$-$1.71,$-$0.36]} & \dstack{\dn{$-$1.71}}{[$-$3.13,$-$0.40]} & \dstack{\ns{$-$0.41}}{[$-$0.86,0.01]} \\
\multicolumn{4}{@{}l}{\method{} vs.\ mean, request sampling (7$-$5)} & \dstack{\up{+0.58}}{[0.18,1.01]} & \dstack{\ns{$-$0.37}}{[$-$1.11,0.38]} & \dstack{\ns{$-$1.30}}{[$-$2.86,0.19]} & \dstack{\up{+0.86}}{[0.43,1.33]} \\
\multicolumn{4}{@{}l}{reward, group scale (7$-$8)} & \dstack{\up{+1.40}}{[0.93,1.86]} & \dstack{\up{+1.10}}{[0.33,1.88]} & \dstack{\ns{+0.96}}{[$-$0.55,2.53]} & \dstack{\up{+1.51}}{[1.01,2.03]} \\
\multicolumn{4}{@{}l}{sampling, group scale (8$-$1)} & \dstack{\ns{+0.16}}{[$-$0.33,0.66]} & \dstack{\up{+1.19}}{[0.25,2.12]} & \dstack{\up{+2.74}}{[0.94,4.51]} & \dstack{\ns{$-$0.14}}{[$-$0.66,0.39]} \\
\multicolumn{4}{@{}l}{scale, request sampling (8$-$3)} & \dstack{\ns{+0.20}}{[$-$0.22,0.60]} & \dstack{\ns{$-$0.05}}{[$-$0.68,0.59]} & \dstack{\ns{$-$0.37}}{[$-$1.71,0.90]} & \dstack{\ns{+0.25}}{[$-$0.21,0.71]} \\
\multicolumn{4}{@{}l}{bounded sum vs.\ mean, group scale (7$-$9)} & \dstack{\ns{0.00}}{[$-$0.32,0.32]} & \dstack{\ns{$-$0.18}}{[$-$0.75,0.41]} & \dstack{\ns{$-$0.35}}{[$-$1.58,0.87]} & \dstack{\ns{+0.05}}{[$-$0.31,0.41]} \\
\multicolumn{4}{@{}l}{scale, mean (9$-$5)} & \dstack{\up{+0.58}}{[0.19,0.98]} & \dstack{\ns{$-$0.19}}{[$-$0.88,0.50]} & \dstack{\ns{$-$0.95}}{[$-$2.40,0.49]} & \dstack{\up{+0.81}}{[0.38,1.25]} \\
\bottomrule
\end{tabular}

\tabnote{Rows of Table~\ref{tab:ablation}; unadjusted 95\% CIs; red: CI below zero.}
\caption{Paired differences between the 2B ablation arms (test, points).}
\label{tab:ablation-diffs}
\end{table*}

\section{Length control}
\label{app:p0}

\paragraph{Definition.} (a) The overflow ($-0.1$ per token after the 20th item line), exact-length end ($+0.1$) and
early-stop ($-0.1$) terms are divided by the mean per-position scale used for the item advantages (the mean batch
standard deviation over positions, or the mean non-zero group standard deviation under group scaling), instead of
being added to the normalized advantages as fixed constants. (b) Completions are capped at 320 tokens (well-formed
20-item movie lists take about 185--255 tokens; 384 on Reddit-Books); truncated completions stay in the loss, receive
no end bonus and pay the overflow penalty. (c) Every non-blank line takes a rank position, and a line that does not
parse as an item is a zero-reward miss at its position; the evaluation parser skips such a line, so it costs a slot in
both; the end bonus counts lines, so a malformed 20th line still earns it. We check that every arm keeps between 19.9 and 20.1 items per list and at
least 95\% exact-20 lists in every logging window through 2,000 steps, and that Rank-GRPO on Qwen3.5-2B with the
length control stays within $\pm0.3$ validation R@20 of the run without it (24.19 against 24.07, below). On
Reddit-Books and Reddit-Music the end bonus requires 20 parser-valid items, and a list with fewer pays the early-stop
penalty, so that a malformed 20th line cannot earn it.

\paragraph{Effective strength.} The shaping terms are multiplied by the run's mean advantage scale, so their effective
weight is 0.1 times that scale. Averaged over training it is 0.233 for \method{} and 0.247 for Rank-GRPO with Qwen3.5-2B
(each of three seeds), and 0.231 to 0.233 against 0.244 to 0.247 with the two Llama backbones: a ratio of 0.94 in every
pair, with a 10 to 90\% spread below 0.01 within each run.

\paragraph{Matched strength.} To rule out the 6\% gap, we retrained \method{} with Qwen3.5-2B (seed 3407) with the
shaping terms multiplied by 1.06. Their effective weight averaged 0.247 over training, as for Rank-GRPO. The checkpoint selected on validation (step 1,750) reaches 22.14
R@20 on test: $+1.18$ $[+0.73,+1.61]$ above Rank-GRPO ($+2.10$ $[+1.22,+2.94]$ per request; N@20 $+0.64$
$[+0.44,+0.85]$) and $-0.39$ $[-0.75,-0.01]$ below \method{} with the unmatched weight ($-0.19$ $[-0.83,+0.47]$ per
request).

\paragraph{Without the length control.} We also trained \method{} and Rank-GRPO with Qwen3.5-2B (seed 3407) without
the length control. \method{}'s validation R@20 peaks at 25.30 at step 1,750 and drops to 23.53 at step 2,000;
Rank-GRPO's reaches 24.07 at step 2,000. On test, the selected checkpoints score 22.27 and 20.83 R@20: \method{}
leads by $+1.45$ $[+1.01,+1.88]$ per row ($+2.20$ $[+1.39,+3.01]$ per request), against $+1.57$ with the length
control on both sides. Removing the length control changes neither method's test score detectably ($-0.26$
$[-0.62,+0.13]$ for \method{}, $-0.14$ $[-0.55,+0.27]$ for Rank-GRPO, per row).

\paragraph{Check.} On Reddit-v2, every 2B arm kept at least 98.8\% exact-20 lists in every training window and
99.8\% on validation; the 4B arms start from an SFT policy whose training samples (temperature 1) have 67\% exact-20
lists and exceed 95\% from step 30 (99\% from step 200). Every Reddit-Books arm reaches 95\% exact-20 lists by step 150,
and malformed lines stay at or below 0.046\% in every training window. Three Reddit-Books 4B arms (\method{} seed 42,
the per-responder batch arm seed 3407 and Rank-GRPO) have episodes in which most lists end one item early: a window's
exact-20 share falls to 0.34, 0.08 and 0.43, in 6 of 185, 3 of 185 and 11 of 29 logged windows after step 150. Every
arm recovers before its selected checkpoint; selected checkpoints keep 97.7--100\% exact-20 lists on validation and
97.8--99.9\% on test.
On Llama the check fails, and both arms drift. With Llama-3.2-3B, \method{}'s training exact-20 share fell from 0.99 to
0.27--0.67 from about step 450, first with 21-item and then with 19-item lists (at test, 85.2\% exact-20 and 14.8\%
with 19 items), and Rank-GRPO's ended at 0.93 (99.0\% at test). With Llama-3.1-8B both arms drifted during training
(exact-20 about 0.84 by step 600); at test 94.9\% of \method{}'s lists and 98.8\% of Rank-GRPO's have exactly 20 items. We report the Llama
pairs with this failure; a list shorter than 20 items can only lose recall.
The batch-scaled arms of Table~\ref{tab:main} drift further in training: Llama-3.1-8B lists settle at 21 items from
about step 130 (exact-20 near the selected step 0.10 per-responder, 0.41 \method{}), Llama-3.2-3B \method{} lists
grow to about 44 items by step 2,000, and Reddit-Books 4B \method{} lists hold 19 items from about step 490 (log kept
to step 590); the Reddit-Music batch arms stay at 20. At test, which keeps the first 20 parsed items, 92.8 to 100\% of
these arms' lists have 20 items.

\paragraph{Group scaling with the mean.}
On Llama-3.2-3B, group-scaled training with the mean diverged after about 850 steps: KL rose from 0.6 to 47, entropy
from 0.28 to 7.6, and completions grew from 150 to more than 400 tokens. The per-responder run with the same scaling
stayed stable (KL 0.3--0.6). With batch scaling the mean kept improving past step 850 (validation R@20 23.87, 24.24
and 25.37 at steps 750, 1,000 and 1,500), with KL spikes of 0.7--10.6 after step 550; without the length control
its lists later grew past 20 items.

\paragraph{With the length control.} The mean with request sampling and per-group scaling (Table~\ref{tab:ablation},
row 9) was trained on Qwen3.5-2B with three seeds. A run whose KL exceeds five times the same-seed \method{} run's
maximum is evaluated at its best validated checkpoint before that point. Seeds 3407 and 1234 stay stable (KL at most
0.41 and 0.84) and are evaluated at step 2,000; seed 42 has one 10-step window with KL 392.5 at step 1,800 (neighbours
0.40 and 0.375) and is evaluated at step 1,500. On test the three seeds reach 22.53, 22.20 and 22.42
R@20 against \method{}'s 22.53, 22.45 and 22.41: per seed $0.00$ $[-0.32,+0.32]$, $-0.25$ $[-0.63,+0.13]$ and $+0.02$
$[-0.37,+0.39]$ per row ($+0.18$, $-0.38$ and $+0.44$ per request, each CI covering 0), and $-0.08$ $[-0.45,+0.29]$
over the three seeds. With seed 42 at step 2,000 (22.41) the three-seed difference is $-0.01$ $[-0.15,+0.14]$. Against
the same reward with batch scaling (row 5) the seed-3407 run gains $+0.58$ $[+0.19,+0.98]$ per row ($-0.19$
$[-0.88,+0.50]$ per request).

\section{Hyperparameters}
\label{app:hparams}

\begin{table}[htbp]
\centering\small
\setlength{\tabcolsep}{4pt}
\begin{tabular}{@{}>{\raggedright\arraybackslash}p{0.43\columnwidth}>{\raggedright\arraybackslash}p{0.53\columnwidth}@{}}
\toprule
\multicolumn{2}{@{}l}{\emph{RL (all arms)}} \\
learning rate & $10^{-6}$, constant; AdamW $(0.9,0.99)$ \\
KL coefficient & $10^{-3}$ \\
clipping $\epsilon_{\mathrm{low}}/\epsilon_{\mathrm{high}}$ & 0.06 / 0.08 \\
group size $G$ & 8 \\
requests / lists per step & 48 / 384 \\
policy updates per batch & 1 (on-policy) \\
importance ratio & token level \\
loss normalization & over the batch's tokens (TRL bnpo) \\
steps & 2,000; checkpoint every 250 \\
rollout temperature & 1.0 \\
max.\ prompt / completion & 2,048 / 320 tokens (384 on Reddit-Books) \\
advantage scale & per group (\method, Rank-GRPO); batch (batch-scaled arms) \\
advantage $\epsilon$ & $10^{-4}$ \\
length shaping & end $+0.1$; early stop $-0.1$; overflow $-0.1$ per token \\
precision; seed & bf16; 3407 (second seeds 42, 1234) \\
\midrule
\multicolumn{2}{@{}l}{\emph{SFT (Qwen3.5)}} \\
learning rate & $5\cdot10^{-5}$, cosine, 5\% warm-up \\
effective batch & 384 \\
steps & 1,500 (select 500/1,000/1,500) \\
max.\ length & 2,048; loss on completion only \\
\midrule
\multicolumn{2}{@{}l}{\emph{Evaluation}} \\
decoding & greedy, $\le$1,024 tokens (vLLM) \\
\midrule
\multicolumn{2}{@{}l}{\emph{Software}} \\
Qwen3.5 and Llama pairs & TRL 1.14.0, vLLM 0.17.1, transformers 5.3.0, PyTorch 2.10 \\
bridge runs (App.~\ref{app:repro}) & released trainer; TRL 0.21, vLLM 0.10, transformers 4.55.4, PyTorch 2.7.1 \\
\bottomrule
\end{tabular}
\tabnote{RL follows Rank-GRPO's Llama-3.2-3B setting.}
\caption{Hyperparameters.}
\label{tab:hparams}
\end{table}

Table~\ref{tab:hparams} lists the hyperparameters. Rank positions are assigned to the non-blank lines, in order; a line that does not match Rank-GRPO's item pattern is
a zero-reward miss at its position, and every token after the 20th line counts as overflow. We use TRL's token-level
importance ratios with the loss normalized over the batch's tokens. The released trainer uses a rank-level ratio, the
geometric mean over a line's tokens, but copies it to every token of the line and sums the loss over tokens; on-policy
(ratio 1, no clipping) both trainers therefore give each line the sum of its tokens' score functions, as in GSPO's
analysis of sequence-level ratios \citep{zheng2025gspo}. The released trainer excludes newline tokens from the loss;
ours includes them. Rank positions, rewards and advantages are unit-tested against the released trainer. Every arm, the baselines included, uses the same implementation. Drawing requests in proportion to
their responders is implemented by repeating each request's example once per responder. The Qwen3.5 base models are
fine-tuned on the released Reddit-v2 demonstrations (95,753 lists); for Llama-3.2-3B we use the SFT checkpoint released
with Rank-GRPO, and Llama-3.1-8B-Instruct is fine-tuned with Rank-GRPO's SFT recipe (selected at step 1,000). SFT runs
for 1,500 steps and
keeps the step among 500/1,000/1,500 with the best validation R@20. Training prompts longer than 2,048 tokens are dropped (72,233 to 72,125 requests), because the TRL version we use no
longer truncates prompts and Rank-GRPO's maximum prompt length was 2,048. There is no correction for the mismatch
between vLLM and training-time token probabilities, as in the released trainer. Each Qwen3.5 run uses 4 B200 GPUs; total compute is in the
Ethics statement.

\section{Reddit-Books}
\label{app:books}

\paragraph{Construction.} Source: r/suggestmeabook submissions and top-level
comments from the Arctic Shift archive.\footnote{\url{https://github.com/ArthurHeitmann/arctic_shift}} Responders are
distinct top-level commenters, excluding the submitter, deleted content and bot accounts. An LLM extracts the books
each comment recommends; mentions are grounded to the UCSD Goodreads book graph and written as ``Title (Year)'' with the
first-publication year, so the trainer, parser and year-tolerant evaluation are unchanged. Books named in the
submission are the request's mentioned titles. The split is temporal: training, then one validation month, then one
test month. SFT targets are 20-item lists from an LLM grounded to the catalog, the role of Rank-GRPO's demonstrations.
\paragraph{Statistics.} 32,539 training requests (May 2014 to
July 2017; the Goodreads graph ends in 2017) with 4.60 responders on average and 18.1\% single-responder requests;
validation (August 2017) 6,259 rows from 1,254 requests; test (September 2017) 6,333 rows from 1,150 requests.
Extraction precision is 0.997 (388 of 389 extracted books; LLM judge on 200 comments, confirmed by a manual check of 30
comments), 96.9\% of 319,262 mentions are grounded, and year agreement with Open Library is 0.903 $[0.852,0.938]$ on
185 mentions (13 of the 18 disagreements are Open Library errors).
Extraction recall was audited on 100 (submission, responder) pairs drawn at random (seed 42) from all 186,511, including
pairs with no extracted book: an independent LLM (Gemini 2.5 Pro, temperature 0) listed the books each comment
recommends, and our extraction covers 0.877 $[0.818,0.919]$ of them under strict title matching and 0.969
$[0.929,0.986]$ after dropping non-specific items (``books by Ursula K.\ Le Guin''), mapping series to their first
book and allowing fuzzy titles; most strict misses are such matching artefacts. On 30 pairs read by hand, the extraction
covers all 45 specific books recommended. Books
first published before AD 1000 (0.55\% of ground truth) are dropped. The SFT targets are one LLM-written list per training
request (97.0\% of lines grounded; lists keep 15--20 books), and the catalog has 74,474 items. The popularity list (the
20 most-named training books) reaches R@20 7.0 on validation and 5.77 on test.
\paragraph{Prompts.} Mention extraction (GPT-5.6). The user message is the prompt below followed by the thread as
\texttt{REQUEST TITLE}, \texttt{REQUEST BODY} and \texttt{RESPONSES} (\texttt{[r1]}, \texttt{[r2]}, \dots, one per
responder, with that responder's comments joined); threads with many responders are split into chunks, and several
threads can share one call, in which case a note asks for one JSON object per thread.
\begin{quote}\footnotesize
You extract book mentions from a Reddit r/suggestmeabook thread. The thread has a REQUEST (the original post) and
RESPONSES (top-level replies by other users; each response id is one user, with that user's replies merged).
Return strict JSON of this shape:
\texttt{\{"request\_books": [\{"title": str, "author": str or null, "year": int or null, "read": true or false\}],
"responses": \{"r1": [\{"title": str, "author": str or null, "year": int or null, "series": str or null\}], "r2": [],
...\}\}}
Rules:
1. request\_books: every specific book the requester names (books they read, liked, disliked, want to avoid, or have
not read yet). read = true if the requester says or clearly implies they have read it (including disliked or
unfinished books); otherwise false.
2. responses: for EVERY response id, the books that response recommends to the requester, in order of appearance.
Include a book when the responder suggests it as something to read, including replies that are only a title, "+1 for
X", or a list. Exclude books mentioned only for comparison or context (e.g. the requester's own examples, "it's like X
but darker"), books the responder advises against, and books merely mentioned without being suggested. Use [] when a
response recommends no specific book.
3. Only books: novels, non-fiction, short-story collections, poetry, plays, graphic novels, comics, manga. Not films,
TV shows, games, podcasts, websites, magazines, genres, or authors without a book.
4. Series: if a response recommends a series without naming a specific book, give the first book of the series as
"title" and the series name as "series". If it names a specific book, give that book (series may still be filled).
5. A recommendation of an author in general without naming any book ("anything by Terry Pratchett") yields nothing.
6. title: the book's standard published English title, correctly spelled and capitalized, without series numbering,
quotes or markdown (fix obvious typos and abbreviations, e.g. "HHGTTG" -> "The Hitchhiker's Guide to the Galaxy").
author: the author's full name as usually credited, from the text or from your knowledge if you are confident, else
null. year: the year of first publication (original language) if you know it, else null.
7. Never add books that the thread does not mention. Output the JSON object only.
\end{quote}
SFT target lists (GPT-5.6). The user message is Rank-GRPO's Reddit-v2 instruction with ``movie'' replaced by ``book''
and the conversation appended; the system message adds the author and the date:
\begin{quote}\footnotesize
\emph{System:} You are a book recommender. Follow the user's instructions, but end every line with " by <author>"
after the year, e.g. "The Hobbit (1937) by J.R.R. Tolkien". Give exactly 20 lines, one book per line, each a real,
specific book (not a series), different from the books the user says they have read. The conversation took place in
\{month year\}: recommend only books published before \{month year\}.\\
\emph{User:} Pretend you are a book recommender system. I will give you a conversation between a user and you (a
recommender system). Based on the conversation, you need to reply with 20 recommendations. List the standardized
English title of each book in each line in the form of "book title" (first\_publication\_year) with NO extra words or
sentences. Here is the conversation: \{conversation\}
\end{quote}
The author suffix is used for grounding and removed from the training targets.

\paragraph{Validation.} Table~\ref{tab:books} gives the 2B validation scores of the ablation arms on Reddit-Books.
The batch-scaled per-responder arm scores above Rank-GRPO (30.65 against 30.53) and is therefore the $\Delta$ baseline
of Table~\ref{tab:main} there; \method{} (31.27) is level with the best batch-scaled request-level arm (31.50; row 7
minus row 6: $-0.16$ $[-0.70,+0.37]$).

\begin{table}[htbp]
\centering\small
\setlength{\tabcolsep}{2.5pt}
\renewcommand{\arraystretch}{1.08}
\begin{tabular}{@{}rl ccc@{}}
\toprule
 & & R@20 & req.\ R@20 & R@20, $n{\ge}5$ \\
\midrule
0 & SFT & 21.65 & 21.11 & 21.50 \\
\baserow 1 & Rank-GRPO & 30.53 & 29.86 & 30.45 \\
2 & per-resp., batch & 30.65 & 30.06 & 30.53 \\
4 & mean, $\propto n$ & \best{31.50} & \second{30.99} & \second{31.37} \\
5 & mean, request & 31.33 & \best{31.13} & 31.01 \\
6 & bounded sum, batch & \second{31.43} & 30.62 & \best{31.56} \\
\oursrow 7 & \method{} & 31.27 & 30.65 & 31.15 \\
\midrule
\multicolumn{2}{@{}l}{6 $-$ 2} & \dstack{\up{+0.78}}{[0.11,1.46]} & \dstack{\ns{+0.55}}{[$-$0.31,1.41]} & \dstack{\up{+1.03}}{[0.22,1.82]} \\
\multicolumn{2}{@{}l}{6 $-$ 1} & \dstack{\up{+0.90}}{[0.24,1.58]} & \dstack{\ns{+0.76}}{[$-$0.10,1.64]} & \dstack{\up{+1.11}}{[0.28,1.90]} \\
\multicolumn{2}{@{}l}{5 $-$ 4} & \dstack{\ns{$-$0.17}}{[$-$0.71,0.36]} & \dstack{\ns{+0.14}}{[$-$0.47,0.81]} & \dstack{\ns{$-$0.35}}{[$-$1.02,0.27]} \\
\multicolumn{2}{@{}l}{6 $-$ 5} & \dstack{\ns{+0.10}}{[$-$0.49,0.69]} & \dstack{\ns{$-$0.51}}{[$-$1.16,0.12]} & \dstack{\ns{+0.54}}{[$-$0.18,1.26]} \\
\multicolumn{2}{@{}l}{7 $-$ 2} & \dstack{\up{+0.62}}{[0.01,1.21]} & \dstack{\ns{+0.59}}{[$-$0.26,1.42]} & \dstack{\ns{+0.62}}{[$-$0.12,1.33]} \\
\multicolumn{2}{@{}l}{7 $-$ 1} & \dstack{\up{+0.74}}{[0.09,1.37]} & \dstack{\ns{+0.79}}{[$-$0.02,1.59]} & \dstack{\ns{+0.70}}{[$-$0.13,1.47]} \\
\multicolumn{2}{@{}l}{7 $-$ 6} & \dstack{\ns{$-$0.16}}{[$-$0.70,0.37]} & \dstack{\ns{+0.04}}{[$-$0.66,0.69]} & \dstack{\ns{$-$0.40}}{[$-$1.06,0.24]} \\
\bottomrule
\end{tabular}

\tabnote{384-token cap; rows numbered as in Table~\ref{tab:ablation}, row 3 not run. Bottom: paired differences (points).}
\caption{Reddit-Books validation (2B).}
\label{tab:books}
\end{table}

\FloatBarrier
\section{Reddit-Music}
\label{app:music}

\paragraph{GPT-5.6 and the year convention.} 43\% of GPT-5.6's music lines give a catalog artist a year other than its
birth or formation year (often the first release). Told this convention in a system message, GPT-5.6 reaches 12.13,
and \method{} leads by $+2.35$ $[+1.63,+3.08]$; matching artist names only, the two are level (14.24 or 14.12 against
14.49). Matching names only leaves \method{}'s gain over Rank-GRPO unchanged: $+1.09$ $[+0.67,+1.51]$ per row and
$+1.28$ $[+0.74,+1.87]$ per request (with the year: $+1.08$ and $+1.28$).

\paragraph{Construction.} Source: r/musicsuggestions submissions and top-level comments from the Arctic Shift archive,
with responders defined as for Reddit-Books and capped at the first 20 per thread (38 on average before the cap).
Training covers July 2023 to June 2024, validation July 2024 and test August 2024. The item is the artist: replies
usually name a song (78.6\% of mentions), which counts as its artist. GPT-5.6 extracts each responder's recommended
artists with the Reddit-Books prompt adapted to music, and a mention is grounded to a MusicBrainz artist (CC0) by its
normalized name or an alias: a unique match (65.8\% of mentions), else the candidate that another mention in the thread
matches uniquely, else the most active candidate if it is at least five times as active as the next (together 25.2\%).
Ambiguous names (6.5\%) and names without a match (2.6\%) are dropped. Items are written ``Artist (Year)'' with the
MusicBrainz begin year and matched within two years as on Reddit-v2; mentions of artists without a begin year are
dropped (6.3\%). Matching keeps the year because only 4.07\% of the teacher's lines give a year more than two off. Some accounts name the same artist in thread after thread regardless of the
request, so an (author, artist) pair seen in at least 20 threads counts as promotion and all its mentions are dropped
(796 pairs, 10.3\% of grounded mentions). Artists the requester
names are not positives for that request. An audit of 100 random (request, responder) pairs with a Gemini 2.5 Pro judge
gives precision 0.966 $[0.922,0.985]$ and recall 0.921 $[0.867,0.954]$.
\paragraph{Statistics.} 19,159 training requests (142,396 rows, 7.43 responders per request; 95.9\% of a
multi-responder request's union items are named by one responder), validation 1,966 requests (15,060 rows), test 1,951
threads (15,893 rows) with 1,931 distinct prompts, over which per-request metrics and the bootstrap run. The catalog holds the 38,881 artists of all splits and teacher lists plus the 20,000 most active
MusicBrainz artists with a begin year.
\paragraph{Training.} SFT targets are GPT-5.6 lists grounded to the catalog and kept with at least 15 items (18,645
requests). The 2B SFT runs for 1,500 steps and is selected at step 1,500 (validation R@20 9.04). Both RL arms use the
Reddit-Books length rule (Appendix~\ref{app:p0}) with a 320-token cap.
\paragraph{Results.} \method{} leads at every validation checkpoint (Table~\ref{tab:music}); the per-responder arm is
still rising at step 2,000 but stays below \method{}'s step 1,250. On test, per-row N@20 gains $+0.51$
$[+0.33,+0.69]$ and R@10 $+1.04$ $[+0.66,+1.39]$; no list fails to parse, and at the end of training 96.5\% (\method{})
and 99.6\% of lists have exactly 20 items.

\begin{table}[htbp]
\centering\footnotesize
\setlength{\tabcolsep}{2.5pt}
\begin{tabular}{@{}l rrl@{}}
\toprule
RL step & Rank-GRPO & \method{} & $\Delta$ \\
\midrule
250 & 9.45 & 10.11 & $+0.65$ \\
500 & 9.89 & 12.71 & $+2.82$ \\
750 & 10.80 & 13.76 & $+2.96$ \\
1,000 & 11.60 & 13.88 & $+2.29$ \\
1,250 & 12.02 & \best{14.41} & $+2.39$ \\
1,500 & 12.41 & 14.15 & $+1.74$ \\
1,750 & 12.65 & 14.15 & $+1.50$ \\
2,000 & \best{12.91} & 13.88 & $+0.96$ \\
\midrule
test, per row & 13.40 & 14.48 & \up{+1.08} [0.66,1.51] \\
test, per request & 12.77 & 14.05 & \up{+1.28} [0.73,1.87] \\
\bottomrule
\end{tabular}
\tabnote{Validation per-row R@20 (\%) every 250 RL steps (SFT 9.04; bold: selected step) and test R@20 of the selected checkpoints. $\Delta$: \method{} minus Rank-GRPO, paired 95\% CI on test. Table~\ref{tab:main} compares with the batch-scaled per-responder arm, better on validation (test 13.71; $+0.78$ $[+0.38,+1.17]$).}
\caption{Reddit-Music (2B): validation and test R@20.}
\label{tab:music}
\end{table}

\FloatBarrier
\section{Reproducibility notes}
\label{app:repro}

\paragraph{Protocol.} Every run is validated every 250 steps, the checkpoint is chosen on validation and each chosen
checkpoint is tested once. Some runs were interrupted and restarted from their latest
checkpoint $k$ (fresh optimizer state, data order seeded with seed${}+k$); checkpoints not validated before an
interruption were not selection candidates.
On the released Rank-GRPO checkpoint our per-row test R@20 is 23.65 (reported 23.68).

\paragraph{Step time and validation cost.} Validation runs every 250 steps on one GPU and takes 4, 5 and 7 minutes at 2B, 4B and 9B,
which adds 2 to 4\% to the training GPU-hours of either arm. A training step of \method{} takes 4 to 11\% longer
than Rank-GRPO's (5.89 against 5.65\,s at 2B, 12.24 against 11.50\,s at 4B, 16.18 against 14.61\,s at 9B), because request sampling draws requests with
longer threads (86k against 77k prompt tokens per step; completions are equally long). Including it, reaching the selected checkpoint costs
\method{} 31.5, 41.6 and 49.3 GPU-hours on the plateau pairs at 2B (seed 3407), 4B and 9B, against 44.2, 67.9 and 120.6
for Rank-GRPO.

\paragraph{Table and figure conventions.} In tables, bold and underline mark the best and second-best trained arm,
and green differences have CIs above zero. \emph{Exact-20} is the share of lists with exactly 20 items. Table~\ref{tab:data}: union items named once are the share of union items
named by one responder, over requests with at least two responders; the reward-noise block is defined below (95\% CIs of
the responder share: $[84,87]$, $[61,70]$, $[60,91]$). Table~\ref{tab:main}: $\Delta$ is \method{} minus the better of the
two per-responder arms, chosen on validation; steps to match is the share of the baseline's 2,000 steps \method{} needs
to reach the baseline's step-2,000 validation R@20 (interpolated between checkpoints; CI from a bootstrap over
requests); ``$\Delta$, batch scaling'' is \method{} with batch scaling minus the same baseline (on Llama-3.1-8B it is
$+0.45$ $[+0.08,+0.84]$ points above \method{}); at 27B and 35B-A3B the batch-scaled arms start from a retrained SFT checkpoint (test 17.62 and 17.19, against 17.90
and 17.34 for the original: $-0.28$ $[-0.79,+0.25]$ and $-0.16$ $[-0.61,+0.32]$), so the per-responder batch arm is not
used as the $\Delta$ baseline there; within that checkpoint the union leads it by $+2.36$ $[+1.91,+2.88]$ and $+1.59$
$[+1.11,+2.08]$;
 GPT-5.6's music score
uses the year convention discussed in Appendix~\ref{app:music}; all metrics are in Tables~\ref{tab:main-full},
\ref{tab:main-full-llama} and~\ref{tab:main-full-b}.
Figure~\ref{fig:gpuh}: each line is the best validation score among the checkpoints so far;
2B is the mean of two seeds (band: range); GPU-hours are 4 B200 times the logged training step time, validation
excluded; raw curves are in Figure~\ref{fig:gpuh-raw}. Table~\ref{tab:compute}: test R@20 per row with 95\% CIs; steps
are \method{}'s against Rank-GRPO's; the 9B 4,000-step Rank-GRPO run was not trained, and the 9B plateau pair covers
that size; ``best'' compares the checkpoints selected on validation in the runs of
Figure~\ref{fig:gpuh}, with the GPU-hours each arm took to reach its own best; the 18 to 63\% of the abstract are
the GPU-hours \method{} takes to reach Rank-GRPO's best (Appendix~\ref{app:compute}). Table~\ref{tab:ablation}: $\Delta$ is the per-row difference to row 1, req.\ the per-request
score. Table~\ref{tab:judges}: preference is the share of about 300 test requests on which the judge prefers
\method{}'s list, judged in both orders; precision is the share of listed titles the judge marks as fitting the
request, over 500 requests.

\paragraph{Table~\ref{tab:data} statistics.} For 500 training requests with at least two responders we sample eight
lists each at temperature 1 from the 2B SFT policy. Spread is the standard deviation of a fixed list's NDCG@20 across
responders and across lists; a group has no signal when all eight lists score zero against the row's responder (or
against all responders). The responder share is the share of the variance of Rank-GRPO's group-scaled advantage due to
the responder draw, with rows weighted as Rank-GRPO samples them.

\paragraph{Reference systems.} The popularity list holds the 20 items most often named in training, excluding those
mentioned in the request. The kNN prior lists the 20 items with the largest $\mu_i(x)$, the similarity-weighted frequency of item $i$
among the $k=200$ nearest training requests by bge-base-en-v1.5 embeddings of the request text, also excluding mentioned
items. Neighbours exclude every request from the same submission: Reddit-v2 splits one submission into several requests
(72,233 from 48,618 submissions), and without the exclusion a held-out responder's exact item set reappears among the
200 neighbours for 12.2\% of sampled requests. GPT-5.6 and Qwen3.5-27B (post-trained, thinking off) answer zero-shot with the
instruction and parser of the trained policies.

\paragraph{Aggregated rewards versus the per-responder expectation.} For 2,953 of 72,233 training requests (4.1\%) the
rows carry different mentioned-title sets; the aggregated rewards use their union, so for these requests the mean is
not exactly the expectation of the per-row reward. Responder counts are keyed by the exact (title, year) strings, while
matching uses normalized titles and a two-year tolerance; the two differ for 93 requests (0.13\%), holding 0.01\% of the
reward mass.

\paragraph{Bridge runs.} Some runs use the released trainer (rank-level importance ratios, newline-based positions):
the continuations of the released checkpoint and a Llama-3.2-3B pair without our length control. The Llama pairs of
Table~\ref{tab:main} and all Qwen3.5 runs use our port to a newer TRL, whose rank positions, rewards and advantages match the
released trainer on well-formed lists (unit test) and whose importance ratio is token level (Appendix~\ref{app:hparams}).
The continuation of the released 15.8k-step checkpoint ran in the released trainer, without our
length control, from that checkpoint, with the learning rate at its final value ($10^{-7}$), 1,000 steps, seed 3407; validation
R@20 declined with both rewards (bounded sum 26.72 to 25.66, per-responder 26.94 to 26.01 from step 250 to 1,000),
both were selected at step 250 and tested once (23.52 and 23.58; $-0.06$ $[-0.26,+0.12]$): the reward changes the
trajectory at a fixed budget, not a converged policy. \method{} in the released trainer on Llama-3.2-3B,
from the released SFT checkpoint without our length control (2,000 steps, seed 3407), is stable to about step 1,000 and then
drifts to over-long lists (mean 457 tokens and 63 items at step 1,080, clipping 0.22), the failure the length control of
Appendix~\ref{app:p0} prevents. Its best validation checkpoint (step 1,000, before the drift) scores 21.84 on test
against 21.70 for Rank-GRPO retrained in the same trainer: $+0.13$ $[-0.32,+0.54]$ per row, with a CI that
includes zero, and $+0.90$ $[+0.13,+1.66]$ per request. Without the length control the
bounded sum gives no per-row gain on this backbone.

\FloatBarrier
\section{Memorization}
\label{app:memo}
Each Qwen3.5 base model (before SFT) sees three training requests with their responders' items as a fixed few-shot
prompt (seed 42), then a Reddit-v2 test request, and continues greedily for at most 400 tokens; the evaluation parser
reads its list. A test request counts as recalled if the list contains an item of any of its responders.
Table~\ref{tab:memo} gives the share of recalled requests and, within each stratum, \method{} minus Rank-GRPO (seed
3407, per-row R@20, bootstrap over requests). Recall grows with size and concentrates on popular items (9B: item
recall 6.0 for the 20\% most-named training items, 1.5 for items never named in training). On Reddit-v2 the
gain of \method{} has its CI above zero in both strata and is smaller on recalled requests. On Reddit-Books and
Reddit-Music the base models recall fewer requests (2 to 13\%); the gain on the others has its CI above zero, and on the
few recalled ones it is not significant (Reddit-Books) or negative (Reddit-Music, 37 requests). The gain therefore does
not come from threads the model can reproduce. The probe covers the Qwen3.5 sizes we train at on each dataset; it does
not cover the Llama backbones, the 27B and 35B-A3B models, or GPT-5.6, which extracted the Reddit-Books and Reddit-Music labels.

\begin{table}[htbp]
\centering\small
\begin{adjustbox}{max width=\columnwidth}
\begin{tabular}{@{}lcccc@{}}
\toprule
 & & Probe & \multicolumn{2}{c}{$\Delta$ R@20 per row (points)} \\
\cmidrule(l){4-5}
Size & Recalled & R@20 (\%) & recalled & not recalled \\
\midrule
\multicolumn{5}{@{}l}{\grp{Reddit-v2}} \\
2B & 5.8\% & 0.8 & \dstack{\up{+1.35}}{[0.17,2.43]} & \dstack{\up{+1.60}}{[1.10,2.11]} \\
4B & 12.6\% & 1.8 & \dstack{\up{+1.07}}{[0.20,1.93]} & \dstack{\up{+1.73}}{[1.19,2.29]} \\
9B & 18.6\% & 3.2 & \dstack{\up{+1.04}}{[0.29,1.76]} & \dstack{\up{+1.72}}{[1.16,2.29]} \\
\multicolumn{5}{@{}l}{\grp{Reddit-Books}} \\
2B & 6.7\% & 1.4 & \dstack{\ns{+0.31}}{[$-$1.43,2.02]} & \dstack{\up{+0.89}}{[0.15,1.61]} \\
4B & 12.8\% & 2.4 & \dstack{\ns{+0.82}}{[$-$0.37,2.23]} & \dstack{\up{+1.16}}{[0.39,1.98]} \\
\multicolumn{5}{@{}l}{\grp{Reddit-Music}} \\
2B & 1.9\% & 0.2 & \dstack{\dn{$-$2.89}}{[$-$5.71,$-$0.31]} & \dstack{\up{+1.19}}{[0.74,1.62]} \\
\bottomrule
\end{tabular}
\end{adjustbox}
\tabnote{Recalled: share of test requests for which the base model names a responder item from the request alone (Reddit-v2: 14\%, 28\% and 38\% of rows; Reddit-Books: 12\% and 21\%; Reddit-Music: 3\%). Probe R@20: recall of the base model's few-shot list.}
\caption{Memorization probe (test sets; $\Delta$: \method{} minus Rank-GRPO).}
\label{tab:memo}
\end{table}

\FloatBarrier
\section{ProtoQA}
\label{app:protoqa}

\begin{table}[htbp]
\centering\footnotesize
\setlength{\tabcolsep}{2.2pt}
\begin{tabular}{@{}l@{\hspace{3pt}}rrr cc@{}}
\toprule
 & \multicolumn{3}{c}{R@10 by seed} & \multicolumn{2}{c}{$\Delta$} \\
\cmidrule(lr){2-4}\cmidrule(l){5-6}
arm (2B) & 3407 & 42 & 1234 & seed 3407 & 3 seeds \\
\midrule
(a) Rank-GRPO & 42.73 & 41.81 & 42.54 & & \\
(b) exact RB & 44.39 & 44.33 & 44.17 & \dstack{\up{+1.65}}{[0.24,3.08]} & \dstack{\up{+1.94}}{[0.80,3.07]} \\
(c) mean, batch & 42.37 & 43.86 & 44.37 & \dstack{\ns{$-$0.36}}{[$-$2.02,1.15]} & \dstack{\ns{+1.17}}{[$-$1.01,3.35]} \\
(d) \method{} & \best{45.77} & \best{46.15} & \best{46.17} & \dstack{\up{+3.04}}{[1.46,4.64]} & \dstack{\up{+3.67}}{[2.64,4.69]} \\
\bottomrule
\end{tabular}
\tabnote{R@10: share of a question's respondents whose answer is among the ten listed. (a): Rank-GRPO's reward and sampling; RB: Rao--Blackwellization of (a); (d): the union reward with request sampling. $\Delta$: minus (a); over 3 seeds, difference of seed means with a Welch CI.}
\caption{ProtoQA test results (558 questions, per-respondent recall@10).}
\label{tab:protoqa}
\end{table}

\paragraph{Data.} ProtoQA \citep{boratko-etal-2020-protoqa}, released under CC BY 4.0, gives each training question the answer
clusters of a survey of about 100 people with their counts (median 94 respondents per question). Each respondent becomes
one row holding a single cluster: 8,581 training questions give 765,701 rows. Validation (400 questions) and
test (558) are drawn from the scraped development set (seed 42); the 52 crowdsourced development questions, with about
100 raw answers each, form a second test set. We drop 71 training questions that nearly duplicate a held-out question
(word Jaccard $\geq 0.8$). Answers are matched to clusters by a re-implementation of the official WordNet matcher that
agrees with it on all 184,420 answer--cluster pairs we checked.
\paragraph{Training.} SFT targets are lists of ten answers from GPT-5.6 given only the question; the 683 training
questions whose list has fewer than ten distinct answers are left out of SFT. Qwen3.5-2B base is trained for 540
steps, and the checkpoint at step 360 has the best validation recall@10 (41.72). All four
arms start from it with the GRPO settings of the main runs and lists of ten, for 250 steps (about 1.4 passes over the
questions, close to Reddit-v2's 1.33), and the best of ten validation checkpoints is tested once. A position scores a hit when
it matches a cluster that no earlier position matched. (a) gives each (question, respondent) row the respondent's
hits, with group scaling and rows drawn uniformly. (b) averages (a)'s group-scaled per-position advantages over the
question's respondents. (c) rewards a hit by its cluster's share of respondents, with batch scaling and rows drawn as in (a). (d) gives a hit
for any cluster of the question, with group scaling and questions drawn uniformly.
\paragraph{Results.} Validation peaks are 44.95, 45.87, 45.50 and 47.45 for (a)--(d), and (b) and (d) are still
rising at step 250. Lists have exactly ten lines in 98.4, 99.5, 92.3 and 96.6\% of test outputs.
The union minus exact Rao--Blackwellization is $+1.39$ $[-0.04,+2.76]$. On the crowdsourced set, scored with the
official metric, Max Answers@10 moves by $-0.03$ $[-5.57,+4.93]$ for (b), $-1.27$ $[-6.54,+3.74]$ for (c) and $+4.92$
$[+0.13,+9.90]$ for (d), each relative to (a). Max Answers@1 moves by $+4.70$ $[+0.33,+11.07]$ for (b).
\paragraph{Training signal and speed.} Groups whose eight lists all receive the same reward make up 81\% of the
logged groups for (a), 82\% for (b), 41\% for (c) and 45\% for (d) (67, 67, 28 and 28\% of positions). Interpolating the
validation curves, (d) reaches (a)'s step-250 recall (44.49) after 81 steps (32\% of the steps; 95\% CI from a
bootstrap over the 400 questions $[6,70]\%$), and (b) after 98 (39\%, $[13,79]\%$).

\paragraph{Longer training.} With 500 steps (seed 3407, 20 validated checkpoints each), the best validation recall
is 46.79 for (a) at step 475 (43.73 at step 250), 47.75 for (b) at step 450 and 48.58 for (d) at step 350; (d) first
exceeds 46.79 at step 200. On test, (a) scores 45.49, (b) 46.76 ($+1.27$ $[-0.45,+3.04]$) and (d) 46.46 ($+0.96$
$[-0.89,+2.72]$); on the crowdsourced set, Max Answers@10 is 49.36, 54.21 and 54.15.

\paragraph{Memorization.} The questions come from a public game show. On validation, the zero-shot Qwen3.5-2B base
model reaches recall@10 of 27.87 and GPT-5.6, which wrote the SFT targets, 67.05; our 2B models stay below 48. All four
arms start from the same SFT checkpoint, so what the teacher memorized enters every arm alike.

\end{document}